\documentclass[twoside,11pt]{article}

\usepackage[nohyperref, abbrvbib, preprint]{jmlr2e}

\usepackage{amsmath,amssymb,latexsym}

\ShortHeadings{Differential Similarity in Higher Dimensional Spaces}{L.\ Thorne McCarty}
\firstpageno{1}

\begin{document}

\title{
Differential Similarity in Higher Dimensional Spaces:\\
Theory and Applications}

\author{\name L.\ Thorne McCarty \email mccarty@cs.rutgers.edu \\
       \addr Department of Computer Science\\
       Rutgers, The State University of New Jersey\\
       New Brunswick, NJ 08901-8554, USA}



\maketitle


\begin{abstract}%
This paper presents an extension and an elaboration of the \emph{theory of differential similarity}, which was originally proposed in \citep{CCCS_arXiv, CCCS_AMAI}.   The goal is to develop an algorithm for \emph{clustering} and \emph{coding} that combines a geometric model with a probabilistic model in a principled way.  For simplicity, the geometric model in the earlier paper was restricted to the three-dimensional case. The present paper removes this restriction, and considers the full $n$-dimensional case.  Although the mathematical model is the same, the strategies for computing solutions in the $n$-dimensional case are different, and one of the main purposes of this paper is to develop and analyze these strategies.  Another main purpose is to devise techniques for estimating the parameters of the model from sample data, again in $n$ dimensions.  We evaluate the solution strategies and the estimation techniques by applying them to two familiar real-world examples:  the classical MNIST dataset and the CIFAR-10 dataset. 
\end{abstract}


\section{Introduction.}
\label{Intro}
 
This paper presents an extension and an elaboration of the \emph{theory of differential similarity}, which was originally proposed in \citep{CCCS_arXiv, CCCS_AMAI}.   The goal is to develop an algorithm for \emph{clustering} and \emph{coding} that combines a geometric model with a probabilistic model in a principled way. The geometric model is a Riemannian manifold with a Riemannian metric, ${g}_{ij}({\bf x})$, which is interpreted as a measure of \emph{dissimilarity}.  The probabilistic model consists of a stochastic process with an invariant probability measure that matches the density of the sample input data.  The link between the two models is a potential function, $U({\bf x})$, and its gradient, $\nabla U({\bf x})$.  Since the components of the gradient appear in the definition of the dissimilarity metric, the measure of dissimilarity will depend on the probability measure.  Roughly speaking, the dissimilarity will be small in a region in which the probability density is high, and vice versa.  Finally, the dissimilarity metric is used to define a coordinate system on the embedded Riemannian manifold, which leads to an ``optimal'' lower dimensional encoding of the original data.

For simplicity, the geometric model in \citep{CCCS_AMAI} was restricted to the three-dimensional case, and this was the main deficiency of the theory.  The present paper removes this restriction, and considers the full $n$-dimensional case.  Although the mathematical model is the same, the strategies for computing solutions in the $n$-dimensional case are different, and one of the main purposes of this paper is to develop and analyze these strategies.  Another main purpose is to devise techniques for estimating the parameters of the model from sample data, again in $n$ dimensions.  We evaluate the solution strategies and the estimation techniques by applying them to two familiar real-world examples:  the classical MNIST dataset \citep{LeCun_etal_1998} and the CIFAR-10 dataset  \citep{Krizhevsky:2009}. 

In the broadest terms, this work is an exploration of the \emph{manifold hypothesis} in deep learning \citep{nipsRifaiDVBM11, Bengio_etal:2013, Fefferman_etal:2016}.  It is a common observation that real-world data in high-dimensional spaces tends to be concentrated on low-dimensional nonlinear subspaces, and this phenomenon seems to contribute to the success of deep neural networks in image recognition, speech recognition, and other similar tasks.  

There are now quite a few algorithms for manifold learning: \citep{Scholkopf_etal:1998, TippingBishop1999, Roweis2000, TenenbaumEtAl2000, belkin2003laplacian, Brand2002, DonohoGrimes2003, HintonRoweis2002, Zhang&Zha:2004, Coifman&Lafon:2006, Weinberger&Saul:2006, Chen_etal.MFA.2010, Chen&Zhang&Fleischer:2010, Yu&Zhang&Gong:NIPS2009}.  These algorithms fall into several categories.  Some use global methods: \citep{TenenbaumEtAl2000, Weinberger&Saul:2006}, while others are primarily local.  Some use probabilistic models: \citep{TippingBishop1999, HintonRoweis2002, Chen_etal.MFA.2010}, while others are primarily geometric.  Among the local geometric algorithms, some are based on the Laplacian \citep{belkin2003laplacian} or the Hessian \citep{DonohoGrimes2003}, and some are based explicitly on a diffusion process \citep{Coifman&Lafon:2006}.  For an analysis of this latter category, see \citep{LeeWasserman2010}.  Often, a discrete stochastic process is defined initially on a finite graph (e.g., as a random walk) and the limiting case is shown to be a diffusion on a manifold. See, e.g., \citep{Belkin&Nyogi:COLT2005, Hein&Audibert&vonLuxburg:2007, Ting&Huang&Jordan:ICML2010}.  In such algorithms, nonlinear dimensionality reduction is usually achieved by a spectral decomposition of the Laplacian on the graph.  

In a recent preprint on nonlinear dimensionality reduction, \citet{Ting&Jordan:2018} develop a general theory for the class of \emph{local spectral methods}: ``These methods construct a matrix using only information in local neighborhoods and take a spectral decomposition to find a
nonlinear embedding.''  The class includes: \citep{Roweis2000, belkin2003laplacian, DonohoGrimes2003, Zhang&Zha:2004, Coifman&Lafon:2006}.  Ting and Jordan's  general framework specifies a differential operator on a compact manifold, with a variety of boundary conditions, and they then analyze how each method converges from a set of conditions on a local neighborhood graph to an eigenproblem for the differential operator.  Distinct methods correspond to distinct operators and distinct boundary conditions.  The theory also leads to a ranking of the various algorithms, and the authors conclude that Local Tangent Space Alignment (LTSA) by \citet{Zhang&Zha:2004} is the best.  

However, if we are investigating the manifold hypothesis in the context of  \emph{deep learning}, \citet{Bengio_etal:2013} argue that we need a method with very different properties.  One problem is the neighborhood graph, which has quadratic complexity.  More significantly, any manifold learning algorithm based solely on local neighborhoods is not likely to generalize very well, in a deep network, beyond the initial training data. A better algorithm for deep learning would construct a \emph{parametric} coordinate mapping that takes into account remote data, as well as local neighborhood data.  Among the existing algorithms, \citep{Bengio_etal:2013} singles out Local Coordinate Coding (LCC) by \citet{Yu&Zhang&Gong:NIPS2009}, which has some similarities to \emph{sparse coding} \citep{Olshausen&Field:1996}.  But most of the discussion in \citep{Bengio_etal:2013} focuses on network architectures that can learn embedded manifolds directly from the structure of the data density itself, known as Regularized Auto-Encoders.  Two types are considered:  Denoising Auto-Encoders (DAEs) \citep{Vincent:2011} and a specialized form of Contractive Auto-Encoders (CAEs) \citep{Alain&Bengio:2012}, both of which can be shown to compute the gradient of the log of the input probability density.
 
The theory of differential similarity \citep{CCCS_AMAI} matches the desiderata advocated by \citet{Bengio_etal:2013} more closely than do the algorithms analyzed by \citet{Ting&Jordan:2018}.  First, it is not based on neighborhood graphs:  It is defined from the start on Euclidean ${\bf R}^{n}$.  Second, the lower dimensional encodings in the theory are constrained globally as well as locally:  In the geometric model, the relationship between the local tangent bundle and the global integral manifold is strictly determined by a classical theorem in differential geometry.  In the probabilistic model, the relationship between the local diffusion process and the global probability density is strictly determined by a classical theorem on stochastic processes.  The diffusion equation in the theory of differential similarity has a drift term as well as a Laplacian term, which means that the diffusion has an invariant probability measure, or a stationary probability density, unlike the diffusion in \citep{LeeWasserman2010}.  And the gradient of the log of the stationary probability density, which is computed by a DAE \citep{Vincent:2011} or a CAE \citep{Alain&Bengio:2012} in the deep learning framework, is precisely the vector field, $\nabla U({\bf x})$, in the theory of differential similarity.  We will return to this point in our discussion of ``Future Work'' in Section \ref{FutureWork}.

The balance of the paper is organized as follows:  Section \ref{ProbM} is a review of ``The Probabilistic Model'' and Section \ref{GeomM} is a review of ``The Geometric Model'' from \citep{CCCS_AMAI}. If the reader is willing to accept on authority (with citations) a few basic results on stochastic processes and differential geometry, this material should be accessible to anyone with a knowledge of  linear algebra and advanced calculus.  Section \ref{ProbM} is short, and it includes a simple proof of the basic theorem that we will need, without the extended discussion of stochastic processes in \citep{CCCS_AMAI}. Section \ref{GeomM} is longer, because it is necessary to extend the geometric model to $n$ dimensions.  We use a form of \emph{prototype coding} for the coordinate system, measuring the \emph{distance} from the \emph{origin} (i.e., the ``prototype'') in $n-1$ specified \emph{directions}.  Thus we define a radial coordinate, $\rho$, and the directional coordinates $\theta^{1}, \theta^{2}, \ldots, \theta^{n - 1}$, collectively denoted by $\Theta$.  We refer to this as the $\rho$, $\Theta$, coordinate system.  

The paper then turns from theory to applications:  ``How to Estimate $\nabla U({\bf x})$ from Sample Data'' in Section \ref{Estimator} and ``Computing the Geodesic Coordinate Curves'' in Section \ref{GeoCoords}.  To work with sample data, we borrow a technique from the literature on the \emph{mean shift algorithm} \citep{Fukunaga&Hostetler:1975, Cheng:1995, Comaniciu&Meer:2002}. Once we have an estimate of $\nabla U({\bf x})$, everything else in the model can be calculated from its components.  In particular, since the coordinate curves are defined by geodesics on the embedded Riemannian manifold, they are determined by the Euler-Lagrange equations for the minimization of the energy functional over the Riemannian metric, ${g}_{ij}({\bf x})$. This is a large system of differential and algebraic equations, in a high-dimensional space, but it can be solved numerically in \emph{Mathematica}.

Finally, Sections \ref{AppMNIST} and \ref{AppCIFAR10} demonstrate how the theory works when applied to real-world examples.  Section \ref{AppMNIST}, on the MNIST dataset \citep{LeCun_etal_1998}, shows that our calculations lead to intuitively reasonable results, and Section \ref{AppCIFAR10}, on the CIFAR-10 dataset \citep{Krizhevsky:2009}, adds another wrinkle:  We show how to use \emph{quotient manifolds} to build invariance into the geometric model, and we show how to use \emph{product manifolds} to combine low-dimensional solutions into a higher dimensional problem space, so that our dimensionality reduction techniques can be applied recursively. 

\section{The Probabilistic Model.}
\label{ProbM}
 
The probabilistic model is known in the literature as \emph{Brownian motion with a drift term}.  More precisely, it is a \emph{diffusion process} generated by the following differential operator:
\begin{align} 
\label{BMwithDrift}
\mathcal{L} \; &= \; \frac{1}{2}\Delta \; + \; \nabla U(\mathbf{x}) \cdot \nabla 
\end{align}
where $\Delta$ is the standard Laplacian expressed in Cartesian coordinates and $U(\mathbf{x})$ is a scalar potential function.
Brownian motion, by itself, is generated by the differential operator $\frac{1}{2}\Delta$.  But Brownian motion ``dissipates,'' that is, it has no invariant probability measure except \emph{zero}.  When we add a drift term, which is given here by $ \nabla U(\mathbf{x}) \cdot \nabla $, the invariant probability measure turns out to be finite and proportional to $e^{\,2 \,U(\mathbf{x})}$.  This means that $ \nabla U(\mathbf{x}) $ is proportional to the \emph{gradient} of the \emph{log} of the stationary probability density. 
  
There are several ways to analyze this diffusion process, and establish this result.  One classical approach is to use Kolmogorov's \emph{backward} and \emph{forward equations}.  See \citep{Kolmogorov1931}. Kolmogorov's backward equation is:
\begin{align}
\label{backwardEqn}
\frac{\partial}{\partial t} \; w(t,\mathbf{x}) &= \frac{1}{2} \sum_{i,j} a^{ij} ({\bf x}) \frac{\partial^{2} }{\partial{x}^{i} \partial{x}^{j}} \; w(t,\mathbf{x})
 \; + \;
 \sum_{i} b^{i}({\bf x}) \frac{\partial }{\partial{x}^{i} } \; w(t,\mathbf{x}) \\
  &= \mathcal{L} \; w(t,\mathbf{x}) ,  \; \mbox{\rm with initial condition} \; w(0,\mathbf{x}) = f(\mathbf{x}), \notag
\end{align}
in which $\mathbf{a}(\mathbf{x})$ is a matrix of diffusion coefficients and  $\mathbf{b}(\mathbf{x})$ is a vector of drift coefficients.  Notice that the operator $\mathcal{L}$ in \eqref{BMwithDrift} is a special case of the operator $\mathcal{L}$ in \eqref{backwardEqn}.  
Kolmogorov's forward equation is:
\begin{align}
\label{forwardEqn}
\frac{\partial}{\partial t} \; p(t,\mathbf{x}) &= \frac{1}{2} \sum_{i,j}  \frac{\partial^{2} }{\partial{x}^{i} \partial{x}^{j}} \; a^{ij} ({\bf x}) \; p(t,\mathbf{x})
 \; - \;
 \sum_{i} \frac{\partial }{\partial{x}^{i} } \; b^{i}({\bf x}) \; p(t,\mathbf{x}) \\
  &= \mathcal{L}^{\mathbf{*}} \; p(t,\mathbf{x}) ,   \notag
\end{align}
in which $\mathcal{L}^{\mathbf{*}}$ is the formal adjoint of $\mathcal{L}$, and $p(t,\mathbf{x})$ is a probability density equal, in the limit, as $t \rightarrow 0$, to the unit probability mass at \textbf{x}.  To find the stationary probability density of our diffusion process, we need to specialize the operator 
$\mathcal{L}^{\mathbf{*}}$ in \eqref{forwardEqn} to the formal adjoint of the operator $\mathcal{L}$ in \eqref{BMwithDrift}, and then set ${\partial p} / {\partial t} = \mathcal{L}^{\mathbf{*}} p = 0$.  But if $\mathbf{a}(\mathbf{x})$ is a constant matrix, then the right-hand side of \eqref{forwardEqn} can be expanded and simplified to
\begin{equation*}
 \frac{1}{2} \sum_{i,j} a^{ij} ({\bf x}) \frac{\partial^{2} }{\partial{x}^{i} \partial{x}^{j}} \; p(\mathbf{x}) 
- \sum_{i}  \left( \frac{\partial }{\partial{x}^{i} } \; b^{i}({\bf x}) \right) p(\mathbf{x})
- \sum_{i}b^{i}( {\bf x}) \left(\frac{\partial }{\partial{x}^{i} } \; p(\mathbf{x}) \right)
\end{equation*}
If we now set $\mathbf{a}(\mathbf{x})$ equal to the identity matrix and $\mathbf{b}(\mathbf{x}) = \nabla U(\mathbf{x}) $, it is a straightforward computation to verify that $\mathcal{L}^{\mathbf{*}}( e^{\,2 \,U(\mathbf{x})}) = 0$.
 
\begin{figure}[tb]
\begin{center}
\includegraphics[width=4in]{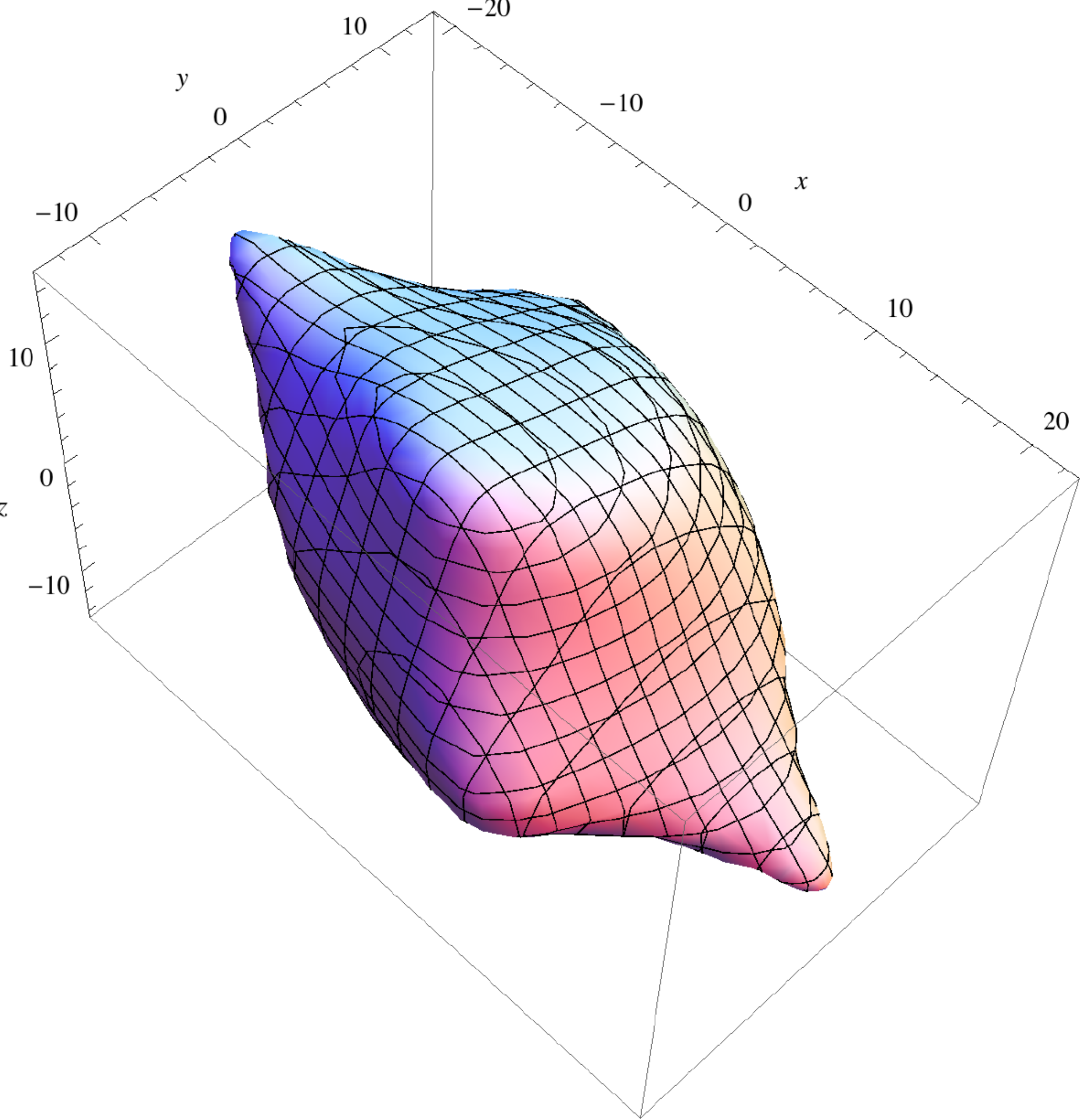}
\caption{Contour plot for the surface of a curvilinear Gaussian potential.}
\label{CurvilinearGaussian}
\end{center}
\end{figure}

Figure \ref{CurvilinearGaussian} is a three-dimensional example borrowed from \citep{CCCS_AMAI}.  The potential function $U({\bf x})$ is defined here as a quadratic polynomial over the variables $u$, $v$, and $w$, where
\begin{equation*}
u \;=\; u(x,y,z), \;\; v \;=\; v(x,y,z), \;\; w \;=\; w(x,y,z), 
\end{equation*}
is defined as a cubic polynomial coordinate transformation from $(x,y,z)$ to $(u,v,w)$.  Specifically, we start with a cubic polynomial: $C(t) = t^3 - t^2 - t$, and define the coordinate transformation as follows:
\begin{align} 
u \;=\; u(x,y,z) =  & \; C(1.4 \; y) + 2 x (y^2 + z^2), \notag \\
v \;=\; v(x,y,z) = & \; C(1.2 \; z) + 2 y (z^2 + x^2), \notag \\
w \;=\; w(x,y,z) = & \; C(1.0 \; x) + 2 z (x^2 + y^2). \notag 
\end{align}
We then define $U({\bf x})$ as: :
\begin{equation}
U(x,y,z) = - \frac{1}{2} ( a \,u(x,y,z)^2 + b \,v(x,y,z)^2 + c \,w(x,y,z)^2) * 10^{-6} \notag
\end{equation}
Thus $U({\bf x})$ is a sixth-degree polynomial in $x$, $y$, and $z$, and the gradient, $\nabla U({\bf x})$, is a fifth-degree polynomial.  We call this the \emph{curvilinear Gaussian potential}.  One important property of this potential function is the fact that $\nabla U(x,y,z) = ( 0, 0, 0 )$ at the origin, which means that $U(0,0,0)= 0$ is an extremal point, a maximal point, in fact.  Figure \ref{CurvilinearGaussian} is a contour plot for the surface at $U(x,y,z) = -10$. Figure \ref{StreamPlotStack}(a) shows a {\tt StreamPlot} of the gradient vector field generated by $\nabla U(x,y,z)$ at $z=-10$, and Figure \ref{StreamPlotStack}(b) shows a stack of such stream plots, at the values $z = 10$, $z = 0$ and $z = -10$.  Notice how the drift vector twists and turns to counteract the dissipative effects of the diffusion term, and maintain an invariant probability measure.

Looking at Figure \ref{StreamPlotStack}, an interesting idea comes to mind:  Could we use this gradient vector field to define a three-dimensional, $\rho$, $\Theta$, coordinate system?  The radial coordinate, $\rho$, would follow the gradient vector, $\nabla U$, and the directional coordinates, $\theta^{1}$, $\theta^{2}$, would be orthogonal to $\rho$.  We will see how to do this in Section \ref{GeomM}.

\begin{figure}[tb]
\begin{center}
\includegraphics[width=5in]{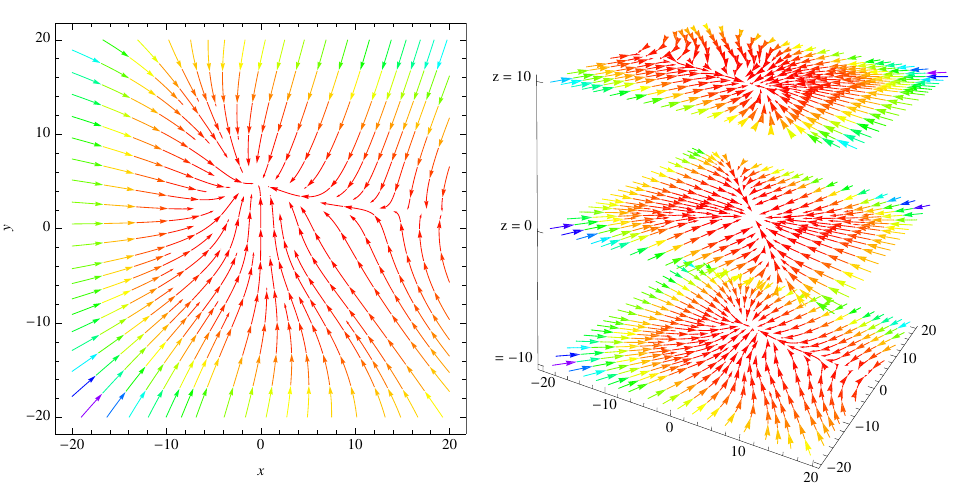}
\caption{Gradient vector field for the curvilinear Gaussian potential: (a) at $z = -10$; (b) at $z = 10$, $z = 0$ and $z = -10$.}
\label{StreamPlotStack}
\end{center}
\end{figure}

The simple formula in \eqref{BMwithDrift} is all the reader needs to know about stochastic processes in order to understand the rest of this paper, including the  examples in Sections \ref{AppMNIST} and \ref{AppCIFAR10}, \emph{infra}.  However, Equation \eqref{BMwithDrift} is part of a much broader and deeper mathematical subject, discussed in Section 2 of \citep{CCCS_AMAI}. Here is a summary, which could be skipped on a first reading:

{\bf Section 2.1} of \citep{CCCS_AMAI} discusses the connection between Equation \eqref{BMwithDrift} and the Feynman-Kac formula \citep{RPFeynman1948, MKac1949}.  

Specializing the operator $\mathcal{L}$ in \eqref{backwardEqn} to the operator $\mathcal{L}$ in \eqref{BMwithDrift}, we have:
\begin{equation}
\label{wCauchy}
\frac{\partial w}{\partial t} = \frac{1}{2}\Delta w \;+\; \nabla U({\bf x}) \boldsymbol{\cdot} \nabla w, \; {\rm with} \; w(0,\cdot) = f.
\end{equation}
Now consider the following partial differential equation:
\begin{equation}
\label{uCauchy}
\frac{\partial u}{\partial t} = \frac{1}{2}\Delta u - V({\bf x})\, u, \; {\rm with} \; u(0,\cdot) = g,
\end{equation}
where $V({\bf x}) \;=\;  \frac{1}{2} \left( \Delta U({\bf x}) + | \nabla U({\bf x}) |^{2} \right)$.
We can show, by a straightforward calculation, using the definition of $V({\bf x})$ in terms of $U({\bf x})$, that $w(t,{\bf x})$ is a solution to \eqref{wCauchy} if and only if $e^{U(\bf x)}w(t,{\bf x})$ is a solution to \eqref{uCauchy} with initial value $u(0,\cdot) = e^{U}f$.  See Lemma 1 in \citep{CCCS_AMAI}. The Feynman-Kac formula associated with \eqref{uCauchy} is:  
\begin{equation}
\label{FKformula}
u(t,{\bf x}) \;=\; \int_{\Omega}^{} g(X_{t}) \;\exp \left[ - \int_{0}^{t} V(X_{s}) \;ds \right] \;\mathcal{W}_{\bf x} (dX),
\end{equation}
where $X_{t} \equiv X(t,\omega)$ denotes a continuous path in ${\bf R}^{n}$, and $\mathcal{W}_{\bf x}$ denotes Wiener measure over all such paths beginning at $X_{0} = {\bf x}$.  Assuming mild regularity conditions, the theorem of \citet{MKac1949}, inspired by \citet{RPFeynman1948}, states that $ u(t,{\bf x}) $ as defined by \eqref{FKformula} is a solution to \eqref{uCauchy}.  Furthermore, because of the relationship between \eqref{wCauchy} and \eqref{uCauchy}, there is a similar integral, involving both $U({\bf x})$ and $V({\bf x})$, that provides a solution to \eqref{wCauchy}.  See Theorem 1 in \citep{CCCS_AMAI}.

These equations all have their origins in physics.  For example, Equation \eqref{uCauchy} is a real-valued version of the Schr\"odinger equation:
\begin{equation}
\label{schrodinger}
\frac{\hbar}{i} \, \frac{\partial \psi(t,\mathbf{x})}{\partial t} = \frac{\hbar^{2}}{2 m}\Delta  \psi(t,\mathbf{x}) - V({\bf x})\,  \psi(t,\mathbf{x}),
\end{equation}
and Equation \eqref{FKformula} with an $i$ in the exponent is Feynman's famous ``path integral'' interpretation of quantum mechanics.
(A more familiar version of Schr\"odinger's equation can be obtained by multiplying both sides of Equation \eqref{schrodinger} by $i^{2} = -1$.)
The Kolmogorov forward equation, Equation \eqref{forwardEqn}, is known to physicists as the Fokker-Planck equation.  Also well known, to physicists, is Chapter 10 of \citep{FeynmanHibbs1965}, in which Feynman and Hibbs analyze a representation of the statistical density matrix in quantum statistical mechanics by means of a real-valued path integral  in the form of Equation \eqref{FKformula}. The existence of these mathematical models in physics leads to a speculative conjecture:  Could there be a physical device, at the molecular level, perhaps, that could compute analog solutions for various quantities associated with Equation \eqref{BMwithDrift}? 
 
{\bf Section 2.2} of \citep{CCCS_AMAI} discusses the interpretation of Equations \eqref{BMwithDrift} and \eqref{backwardEqn} as \emph{stochastic differential equations}, following the theories of both \citet{ito1951} and \citet{stratonovich1966}.

It\^{o}'s theory starts with the definition of a \emph{stochastic integral} in the following form:
\begin{equation*}
X(t) \; = \; X(0) \; + \; \int_{0}^{t} \sigma (s,\omega) \,d\mathcal{B}(s,\omega) 
             \; + \; \int_{0}^{t} b(s,\omega) \,ds, 
\end{equation*}
where the first integral is an \emph{It\^{o} integral }defined with respect to the Brownian motion $ \mathcal{B}(t,\omega)  $.  In differential notation, this would be:
\begin{equation*}
   dX(t) \; = \;  \sigma (t,\omega) \,d\mathcal{B}(t,\omega)  + b(t,\omega) \,dt.
\end{equation*}
Extending the notation to $n$ dimensions, let $ \mathcal{B}_{1}(t,\omega) , \dots ,  \mathcal{B}_{d}(t,\omega) $ be $d$ independent Brownian motion processes, and define the $n$-dimensional \emph{It\^{o} process} as follows:
\begin{equation}
\label{ItoXtDef}
   dX(t) \; = \;
   \begin{pmatrix} 
      \sigma^{1}_{1} & \dots & \sigma^{1}_{d} \\
      \vdots & & \vdots \\
      \sigma^{n}_{1}  & \dots & \sigma^{n}_{d} \\
   \end{pmatrix}
      \begin{pmatrix}
         d\mathcal{B}_{1}(t)  \\
         \vdots \\
         d\mathcal{B}_{d}(t)  \\
      \end{pmatrix}
      \; + \;    
          \begin{pmatrix}              b^{1} \\
             \vdots \\
             b^{n} \\
          \end{pmatrix} 
          dt
\end{equation}
In this equation, ${\mathbf \sigma} \colon {\bf R}^{n} \to {\bf R}^{n \times d}$ is the ``square root'' of {\bf a}, that is, ${\bf a}(\mathbf{x}) = {\mathbf \sigma}(\mathbf{x}) \mathbf{\sigma}(\mathbf{x})^{T}$.  One basic result of It\^{o}'s theory is that Equation \eqref{ItoXtDef} defines the same stochastic process as the operator $\mathcal{L}$ in Equation \eqref{backwardEqn}.  See Theorem 2 in \citep{CCCS_AMAI}.

For our purposes, however, the It\^{o} process has a defect:  It is not invariant under coordinate transformations.  An alternative is to use the stochastic integral proposed by \citet{stratonovich1966}.  (Technically, in the discretization of $t$ that leads to the definition of the integral for $ d\mathcal{B}(t,\omega) $, It\^{o}'s theory evaluates the integrand at the initial point of the interval $ [ t_{j}, t_{j+1} ]$, while Stratonovich's theory evaluates it at the mid point.)  A common notation for this alternative is:
\begin{equation}
\label{StratXtDef}
   dX(t) \; = \;
   \begin{pmatrix} 
      \sigma^{1}_{1} & \dots & \sigma^{1}_{d} \\
      \vdots & & \vdots \\
      \sigma^{n}_{1}  & \dots & \sigma^{n}_{d} \\
   \end{pmatrix}
   \circ
      \begin{pmatrix}
         d\mathcal{B}_{1}(t)  \\
         \vdots \\
         d\mathcal{B}_{d}(t)  \\
      \end{pmatrix}
      \; + \;    
          \begin{pmatrix}              \tilde{b}^{1} \\
             \vdots \\
             \tilde{b}^{n} \\
          \end{pmatrix} 
          dt
\end{equation}
It turns out that the Stratonovich integral satisfies a formula for the ``chain rule'' that is consistent with the Newton-Leibniz calculus, and thus Equation \eqref{StratXtDef} behaves properly under coordinate transformations. Fortunately, we can use both theories, and translate back and forth between the two of them, because \eqref{ItoXtDef} and \eqref{StratXtDef} define the same stochastic process whenever
\begin{equation*}
\tilde{b}^{i} \;=\;  b^{i} \, - \; \frac{1}{2} \sum_{k = 1}^{d} \sum_{j = 1}^{n} 
  \frac{\partial \sigma^{i}_{k} }{ \partial {x}^{j}}  \sigma^{j}_{k}.
\end{equation*}
See Lemma 2 in \citep{CCCS_AMAI}.  This translation can therefore be used to rewrite in a nonlinear $\rho$, $\Theta$, coordinate system any stochastic process that was initially defined by Equation \eqref{BMwithDrift} in Euclidean ${\bf R}^{n}$.

In particular, Section 6 in \citep{CCCS_AMAI} shows how to convert the example in Figure \ref{CurvilinearGaussian} from an It\^{o} equation in Euclidean ${\bf R}^{3}$ into a Stratonovich equation in the coordinates ($\rho, \Theta$), and then back into an It\^{o} equation with coefficients $ \alpha^{ij} (\rho, \Theta)$ and $ \beta^{i}(\rho, \Theta)$.  One interesting consequence of these conversions is the calculation of the ``drift correction vector field'' illustrated in Figure 18 in Section 6 of \citep{CCCS_AMAI}.

{\bf Section 2.3} of \citep{CCCS_AMAI} discusses \emph{integral curves} and \emph{martingales} on manifolds, and develops another interpretation of Equation \eqref{BMwithDrift} based on Stroock's Theorem 7.3.10 in \citep{stroock1993}.  See Theorem 4 in \citep{CCCS_AMAI}. This interpretation is not actually used in \citep{CCCS_AMAI} to justify additional calculations.  However, it is likely that Stroock's work will be useful if we want to advance our theoretical understanding of how the stochastic process generated by Equation \eqref{BMwithDrift} interacts with the geometric model that we will construct in Section \ref{GeomM}.  Our coordinate system for the geometric model is based on integral curves, as we will see, and thus the papers of \citet{stroockTan1994, stroockTan1996} are highly relevant.

\section{The Geometric Model.}
\label{GeomM}

To implement the idea of prototype coding in our geometric model, we need to define a radial coordinate,  $\rho$,  and the directional coordinates, $\theta^{1}, \theta^{2}, \ldots, \theta^{n - 1}$,  where $n$ is the dimensionality of the initial Euclidean space. But what we really want is a lower dimensional subspace, a $k$-dimensional subspace, say, where $k < n$.  Somehow, we would like to choose $k-1$ out of the $n-1$ directional coordinates, and project our diffusion process onto the resulting $k-1$ dimensional space, which can then be combined with our one-dimensional radial coordinate to give us a $k$-dimensional subspace.   How should these coordinate systems be defined? 

First, we want the radial coordinate, $\rho$, to follow the drift vector, $\nabla U({\bf x})$.  To do this, we define $\rho(t)$ to be the \emph{integral curve} of the vector field $\nabla U({\bf x})$, starting at some initial point ${\bf x}_{0}$.  More specifically, we define $\rho(t)$ to be the solution to the following differential equation:
\begin{align}
\label{IntegralCurveDiffEq}
\rho'(t)  \;&=\; \frac{\nabla U( {\rho}(t) )}{ | \nabla U( {\rho}(t) ) | }  \\ 
\rho(0)  \;&=\; {\bf x}_{0} \notag
\end{align}
or, equivalently, the solution to the following integral equation:
\begin{equation}
\label{IntegralCurveIntEq}
\rho(t)  \;=\;  {\bf x}_{0} \;+\; \int_{0}^{t} \frac{\nabla U( {\rho}(s) )}{ | \nabla U( {\rho}(s) ) | }  \, ds , \;\;  0 \leq t
\end{equation}
Since the vector field in \eqref{IntegralCurveDiffEq} or \eqref{IntegralCurveIntEq} is normalized, the integral curve that solves these equations will be parameterized by Euclidean arc length.  However, the parameterization that we choose is just a matter of convenience, and what we really want is a generalization of the concept of an integral curve, known as an \emph{integral manifold}. A one-dimensional integral manifold is, roughly speaking, just the image of an integral curve without the parameterization, and it always exists, for any vector field.  
 
For the directional coordinates, $\theta^{1}, \theta^{2}, \ldots, \theta^{n - 1}$, the obvious generalization would be an integral manifold of dimension $n - 1$, orthogonal to the integral manifold for $\rho$.   But, for $k \geq 2$, a $k$-dimensional integral manifold exists if and only if certain conditions are satisfied, known as the \emph{Frobenius integrability conditions}.  Fortunately, as we will see, if we are looking for an integral manifold orthogonal to a vector field that is proportional to the gradient of a potential function, such as $\nabla U({\bf x})$, then the Theorem of Frobenius gives us the results that we want.  Our analysis here is based on the standard literature in differential geometry. See, e.g., \citep[Chapter 6]{spivakcomprehensive}; \citep[Chapter 3]{bishop1968tensor};  \citep[Chapter 8]{auslander1977introduction}; \citep[Chapter 1]{bishopcrittenden2001}; \citep[Chapter VI]{lang1995differential}. 
 
Let's consider an $n - 1$ dimensional \emph{tangent subbundle}, $E$, in ${\bf R}^{n}$ at some point ${\bf x}$ along the integral curve $\rho(t)$.  We will initially use the Cartesian coordinates from the ambient space ${\bf R}^{n}$ to define a set of basis vectors for $E$, which suggests that one axis should be used to ``center'' the coordinate system and the other $n - 1$ axes should be used to specify alternative directions in the vector space.  To simplify both the exposition and our later calculations, we will always ``center'' our coordinate system on $x^{1}$ and simply permute the coordinate axes whenever we wish to make a different choice.  It will be convenient to establish a special notation for the components of $\nabla U({\bf x})$ that reflects this convention.  Thus we define:
\begin{equation*}
\nabla U({\bf x}) \,=\, (\,  P_{0}({\bf x}) ,  P_{1}({\bf x}) , \ldots , P_{n - 1}({\bf x}) \,),
\end{equation*}
and observe that the term $P_{0}({\bf x}) = \partial U({\bf x}) / \partial x^{1}$ will play a special role because of our centering convention.  We now define the basis vectors for $E$ as follows:
 \begin{align}
\nabla U({\bf x}) \; &= \; ( \hspace{1em} P_{0}({\bf x}) ,& P_{1}({\bf x}) ,& &P_{2}({\bf x}) ,& &\ldots ,& &P_{n - 2}({\bf x}) ,& &P_{n - 1}({\bf x}) \;) \notag \\
 {\bf V}_{1}({\bf x}) \; &= \; (\; - P_{1}({\bf x}) ,& P_{0}({\bf x}) ,& &0 ,& &\ldots ,& &0 ,& &0 \;) \notag \\
 {\bf V}_{2}({\bf x}) \; &= \; (\; - P_{2}({\bf x}) ,& 0,& &P_{0}({\bf x})  ,& &\ldots ,& &0 ,& &0 \;) \notag \\
 & \ldots & & & & & & & & & \notag \\
 {\bf V}_{n - 2}({\bf x}) \; &= \; (\; - P_{n - 2}({\bf x}) ,& 0 ,& &0 ,& &\ldots ,& &P_{0}({\bf x}) ,& &0 \;) \notag \\
 {\bf V}_{n - 1}({\bf x}) \; &= \; (\; - P_{n - 1}({\bf x}) ,& 0 ,& &0 ,& &\ldots ,& &0 ,& &P_{0}({\bf x}) \;) \notag \\ \notag
 \end{align}
It is straightforward to verify that $\nabla U({\bf x})$ is orthogonal to each $  {\bf V}_{i}({\bf x}) $, but we need to analyze the tangent subbundle more carefully to verify the Frobenius integrability conditions.
 
It is standard in differential geometry to think of a vector field as a differential operator, essentially the \emph{directional derivative} with respect to a given vector ${\bf V}$.  We will write this in shorthand notation as ${\bf V} \partial$.  Let's now consider the vector fields defined by $ {\bf V}_{i}  =  {\bf V}_{i}({\bf x}) / P_{0}({\bf x}) $ and $ {\bf V}_{j}  =  {\bf V}_{j}({\bf x}) / P_{0}({\bf x}) $, with $i \neq j$, and let's compute the \emph{Lie bracket} of $ {\bf V}_{i} \partial$ and ${\bf V}_{j} \partial $.  By a straightforward (but tedious) calculation, we have:
\begin{align}
\label{ViVjLieBracket}
\left[  {\bf V}_{i} \partial , {\bf V}_{j} \partial  \right] \; &= \;  {\bf V}_{i} \partial \circ {\bf V}_{j} \partial - {\bf V}_{j} \partial \circ {\bf V}_{i} \partial  \\[1ex]
  \;&=\;  \frac{1}{P_{0}^{2}}
\left(\begin{array}{c} P_{0}  \left[ \frac{  \partial P_{i} }{ \partial x^{j+1} }  -  \frac{\partial P_{j}}{\partial x^{i+1}}  \right] \\[2ex]
 + \; P_{i}\left[ \frac{\partial P_{j}}{\partial x^{1}}- \frac{\partial P_{0}}{\partial x^{j+1}}  \right]  \\[2ex]
 + \; P_{j}\left[ \frac{\partial P_{0}}{\partial x^{i+1}}- \frac{\partial P_{i}}{\partial x^{1} }  \right]  
 \end{array}\right) \frac{\partial}{\partial x^{1}} \notag 
\end{align}
Now substitute the definitions $P_{0} = \partial U({\bf x}) / \partial x^{1}$, $P_{i} = \partial U({\bf x}) / \partial x^{i+1}$ and $P_{j} = \partial U({\bf x}) / \partial x^{j+1}$, and note that the terms in the square brackets vanish identically by virtue of the equality of mixed partial derivatives.  Thus $ \left[  {\bf V}_{i} \partial , {\bf V}_{j} \partial  \right]  = 0$, which means that the vector fields $ {\bf V}_{i} \partial$ and ${\bf V}_{j} \partial $ \emph{commute}.
 
To formulate the Theorem of Frobenius, we need several definitions.   We say that the vector field ${\bf V} \partial$ \emph{belongs to} the tangent subbundle $E$ if ${\bf V}({\bf x})$ is an element of $E$ at each point ${\bf x}$ of the domain.  Then, if $ \left[  {\bf X} \partial , {\bf Y} \partial  \right] $ belongs to $E$ whenever ${\bf X} \partial$ belongs to $E$ and ${\bf Y} \partial$ belongs to $E$, for arbitrary ${\bf X}$ and ${\bf Y}$, we say that $E$ is \emph{involutive}.  If the tangent subbundle $E$ can be extended to a full integral manifold, we say that $E$ is \emph{integrable}.   These two concepts are related by the following:
\begin{theorem}[\bf Frobenius]
\label{thmFrobenius}
A tangent subbundle, $E$, is integrable if and only if it is involutive.
\end{theorem}
\begin{proof}
See \citep{bishopcrittenden2001}, Sections 1.4 and 1.6, and Theorems 5, 6 and 7; \citep{lang1995differential}, Chapter VI, $\S$1 -- $\S$4, and Theorems 1.1 and 1.2.
\end{proof}
\noindent
We can now show that the tangent subbundle defined above by the basis vectors  ${\bf V}_{1}({\bf x})$,  ${\bf V}_{2}({\bf x})$, $\ldots$,  ${\bf V}_{n - 1}({\bf x})$, is involutive, hence integrable. 

If ${\bf V} \partial$ and ${\bf W} \partial$ are vector fields and $f$ and $g$ are differentiable real-valued functions, we have the following identity for the expansion of Lie brackets:
\begin{align}
\label{fVgWLieBracket}
\left[  f \, {\bf V} \partial , g \, {\bf W} \partial  \right] \; &= \;   f \, {\bf V} \partial \circ g \, {\bf W} \partial - g \, {\bf W} \partial \circ f \, {\bf V} \partial  \\[1ex]
 \;&=\;   f \, ({\bf V} \partial \, g)  \, {\bf W} \partial + f g \, {\bf V} \partial \circ {\bf W} \partial - g \, ({\bf W} \partial \, f)  \, {\bf V} \partial - g f \, {\bf W} \partial \circ {\bf V} \partial \notag \\[1ex]
  \;&=\;  f g \left[ {\bf V} \partial , {\bf W} \partial  \right] +  f \, ({\bf V} \partial \, g)  \, {\bf W} \partial - g \, ({\bf W} \partial \, f)  \, {\bf V} \partial \notag
\end{align}
Now let ${\bf X} \partial$ and ${\bf Y} \partial$ be two arbitrary vector fields that belong to $E$, and write each of them in terms of their basis vectors:
\begin{align}
\label{XYbasis}
{\bf X} \partial \;&=\; \sum_{i=1}^{n-1} f^{i}({\bf x}) {\bf V}_{i}({\bf x}) \partial \;=\; \sum_{i=1}^{n-1} f^{i}({\bf x}) P_{0}({\bf x}) {\bf V}_{i} \partial \notag \\
{\bf Y} \partial \;&=\; \sum_{j=1}^{n-1} g^{j}({\bf x}) {\bf V}_{j}({\bf x}) \partial \;=\; \sum_{j=1}^{n-1} g^{j}({\bf x}) P_{0}({\bf x}) {\bf V}_{j} \partial \notag
\end{align}
To compute $ \left[  {\bf X} \partial , {\bf Y} \partial  \right] $, we apply \eqref{fVgWLieBracket} and use the fact that all terms in the form $ \left[  {\bf V}_{i} \partial , {\bf V}_{j} \partial  \right] $ vanish because of \eqref{ViVjLieBracket}, to show that the remaining terms form a linear combination of the basis vectors of E.  Thus 
$ \left[  {\bf X} \partial , {\bf Y} \partial  \right] $ belongs to $E$, which means that $E$ is involutive.

\begin{figure}[tbp]
\begin{center}
\includegraphics[width=5in]{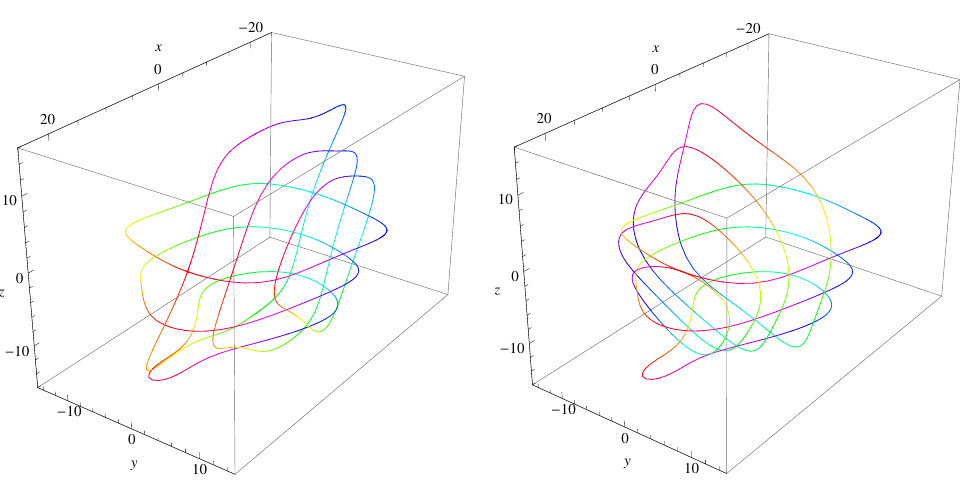}
\caption{An integral manifold with two global coordinate systems for the curvilinear Gaussian potential.}
\label{3DIntManifold}
\end{center}
\end{figure}

Figure \ref{3DIntManifold} shows two views of the integral manifold for the three-dimensional example that was depicted in Figures  \ref{CurvilinearGaussian} and \ref{StreamPlotStack} in Section \ref{ProbM}.  The view on the left is based on a coordinate system centered on the $x$ axis, and it shows several integral curves in the $xy$ plane (with $z$ constant) and the $xz$ plane (with $y$ constant).  The view on the right is based on a coordinate system centered on the $y$ axis, which was computed by a permutation of the axes resulting in the definition $P_{0}({\bf x}) = \partial U({\bf x}) / \partial x^{2}$.  It shows the same integral curves in the $xy$ plane (with $z$ constant) along with several new integral curves in the $yz$ plane (with $x$ constant).  It should be clear that both sets of integral curves are tracing out the same two-dimensional integral manifold.  Can this integral manifold be defined without reference to a specific coordinate system?
 
In \citep{CCCS_AMAI}, we reversed the procedure that we have been following  here.  Theorem 5 in \citep{CCCS_AMAI} asserts the existence of a two-dimensional integral manifold orthogonal to any vector field that is proportional to the gradient of a scalar potential function, i.e., any vector field in the form $N({\bf x}) \nabla U({\bf x})$.  The theorem, as stated, makes use of the vector cross product and the ``curl,'' which is a three-dimensional concept, but it is actually a special case of a general result in ${\bf R}^{n}$ which follows from the dual version of the Theorem of Frobenius, expressed in terms of differential forms.  Thus, if we wanted to, we could develop the theory of our integral manifold in a more ``intrinsic'' way, without reference to a special coordinate system.  We would still have to introduce a coordinate system, of course, when we wanted to do computations, as we did in \citep{CCCS_AMAI}, but this would not be our starting point.

We could use either set of integral curves in Figure \ref{3DIntManifold} to define a curvilinear coordinate system on the Frobenius integral manifold, but it would be a \emph{global} coordinate system, since it follows the \emph{global} Cartesian coordinates from the ambient space ${\bf R}^{n}$. This would not be particularly useful if we are looking for an ``optimal'' $k-1$ dimensional subspace.  But another approach is to use these global vector fields to construct a \emph{local} coordinate system.  Any linear combination of $ {\bf V}_{1}({\bf x}),  {\bf V}_{2}({\bf x}), \ldots ,  {\bf V}_{n - 1}({\bf x}) $, could be taken as one of the basis vectors for the tangent subbundle, and we can vary this linear combination as we move around the integral manifold.  To implement this idea, it is useful to define a \emph{Riemannian metric} on the integral manifold, and the most natural way to do this is to define a metric tensor on all of ${\bf R}^{n}$, using the inner products of  $\nabla U({\bf x})$, $ {\bf V}_{1}({\bf x})$,  ${\bf V}_{2}({\bf x})$, $\ldots$ ,  and ${\bf V}_{n - 1}({\bf x}) $, in that order.  We thus define:
\begin{align}
&
 \begin{pmatrix}
    \\
   {g}_{i,j}({\bf x}) \\
    \\
\end{pmatrix}
 \; = \;  \notag \\
 \notag \\
 &
\left(\begin{array}{cccccc}  | \nabla U |^{2}  & 0 & 0 & \ldots & 0 & 0 \\[1ex]
0 & P_{0}^{2} + P_{1}^{2} & P_{2}P_{1} & \ldots & P_{n-2}P_{1} & P_{n-1}P_{1} \\[1ex]
0 & P_{1}P_{2} & P_{0}^{2} + P_{2}^{2} & \ldots & P_{n-2}P_{2} & P_{n-1}P_{2} \\[1ex]
 \ldots & & & & & \ldots \\[1ex]
0 & P_{1}P_{n-2} & P_{2}P_{n-2}  & \ldots & P_{0}^{2} + P_{n-2}^{2}& P_{n-1}P_{n-2} \\[1ex] 
0 & P_{1}P_{n-1} & P_{2}P_{n-1} & \ldots & P_{n-2}P_{n-1} & P_{0}^{2} + P_{n-1}^{2} \\[1ex]
\end{array}\right)
\notag
\end{align}
If we want a uniform coordinate notation in place of $\rho$ and $\Theta$, we let $i$ and $j$ range over $0$, $1$, $2$, $\ldots$, $n-1$, and we stipulate that $u^{0} = \rho$, and $u^{i} = \theta^{i}$, for $i = 1, \ldots, n-1$.  With these conventions, we adopt the formula above as the definition of our \emph{Riemannian dissimilarity metric}. 

 
What is the purpose of this metric, and why does it make sense as a measure of dissimilarity?  To answer these questions, we need to analyze the coordinates $\rho$ and $\Theta$, separately, within the $( \rho, \Theta )$ coordinate system.
 
First, for the $\rho$ coordinate, here is Theorem 11 from \citep{CCCS_AMAI}, rewritten as:
\begin{theorem}
\label{ConstantRho}
In the ($\rho, \Theta$) coordinate system, the value of the $\rho$ coordinate on the Frobenius integral manifold is a constant.  
\end{theorem}
\begin{proof}
There is a proof of this proposition in \citep[Theorem 6.5]{spivakcomprehensive}, as part of the proof of the Frobenius Theorem.  To understand the main idea, see the proof in the three-dimensional case for Theorem 11 in \citep{CCCS_AMAI}.
\end{proof}
There are two natural measures of distance along the $\rho$ coordinate curve, the Euclidean arc length and the Riemannian arc length, the second of which is determined by the dissimilarity metric, ${g}_{i,j}({\bf x})$.  For the value of the $\rho$ coordinate in Theorem \ref{ConstantRho}, we will choose the Riemannian arc length.   We have the following result:
\begin{theorem}
\label{ConstantRiemannian}
The Riemannian distance along the $\rho$ coordinate curve $\rho(t)$ is equal to the difference in the potential function $U(\mathbf{x}(t))$ along that curve.
\end{theorem}
\begin{proof}
See the calculations in the three-dimensional case following Theorem 11 in \citep{CCCS_AMAI}, which can obviously be generalized to $n$ dimensions.
\end{proof}
We now put these two theorems together and summarize them in three points, which are paraphrased from the discussion at the end of Section 4 in \citep{CCCS_AMAI}:
\begin{itemize}

\item Theorem \ref{ConstantRho} tells us that the value of the $\rho$ coordinate on the Frobenius integral manifold is a constant.  

\item Theorem \ref{ConstantRho} and Theorem \ref{ConstantRiemannian} together imply that the Riemannian distance along the $\rho$ coordinate curve from the Frobenius integral manifold to the origin is a constant.  We can use this property to project data points onto the manifold, just by adding and subtracting Riemannian distances.

\item Since the Riemannian distance to the origin along the coordinate curve $\rho(t)$ is $U(\mathbf{x}(t))$, and since the probability density in our model is proportional to $e^{2 \, U( \mathbf{x}(t) )}$, it follows that the Frobenius integral manifold is a surface of constant probability density.

\end{itemize}
We will be using these mathematical facts frequently in the remainder of this paper. 
 
Second, for the $\Theta$ coordinates:  The main application of our dissimilarity metric is to compute \emph{geodesics} on the surface of the Frobenius integral manifold orthogonal to $\nabla U({\bf x})$.  Since any linear combination of $ {\bf V}_{1}({\bf x}),  {\bf V}_{2}({\bf x}), \ldots ,  {\bf V}_{n - 1}({\bf x}) $, yields a vector in the tangent subbundle, $E$, we can construct vector fields in $E$ in the form 
\begin{equation*}
 \sum_{i=1}^{n-1}  v^{i}(t) {\bf V}_{i}({\bf x}) 
\end{equation*}
for arbitrary functions 
\begin{equation*}
\mathbf{v}(t) \;=\; \left(\begin{array}{c} v^{1}(t) \\ v^{2}(t) \\  \cdots \\ v^{n-1}(t) \end{array}\right) 
\end{equation*}
For a geodesic, we are looking for a curve $\gamma(t)$ with values in ${\bf R}^{n}$ which minimizes the ``energy'' functional: 
\begin{equation*}
\label{EnergyFunctional}
\frac{1}{2}  \int_{0}^{T} \mathbf{v}(t)^{\top}
\left(\begin{array}{cccc} {g}_{1,1}(\gamma(t)) &  {g}_{1,2}(\gamma(t)) & \ldots & {g}_{1,n-1}(\gamma(t)) \\ 
{g}_{2,1}(\gamma(t)) & {g}_{2,2}(\gamma(t)) & \ldots & {g}_{2,n-1}(\gamma(t))\\
\ldots & & & \ldots\\
{g}_{n-1,1}(\gamma(t)) & {g}_{n-1,2}(\gamma(t)) & \ldots & {g}_{n-1,n-1}(\gamma(t))\\
\end{array}\right)
 \mathbf{v}(t)  \; dt 
\end{equation*}
subject to the constraint:
\begin{equation}
\label{GammaConstraint}
\gamma \, '(t)  \;=\;   \sum_{i=1}^{n-1}  v^{i}(t) {\bf V}_{i}(\gamma(t)) \notag
\end{equation}
This variational problem leads to a system of Euler-Lagrange equations for the curves $\gamma(t)$ and $ \mathbf{v}(t) $, plus $n$ Lagrange multipliers, and by the existence and uniqueness theorems for ordinary differential equations, the resulting system will have a solution if we specify the initial conditions $\gamma(0)$ and $ \mathbf{v}(0) $.
  
\begin{figure}[tbp]
\begin{center}
\includegraphics[width=4.5in]{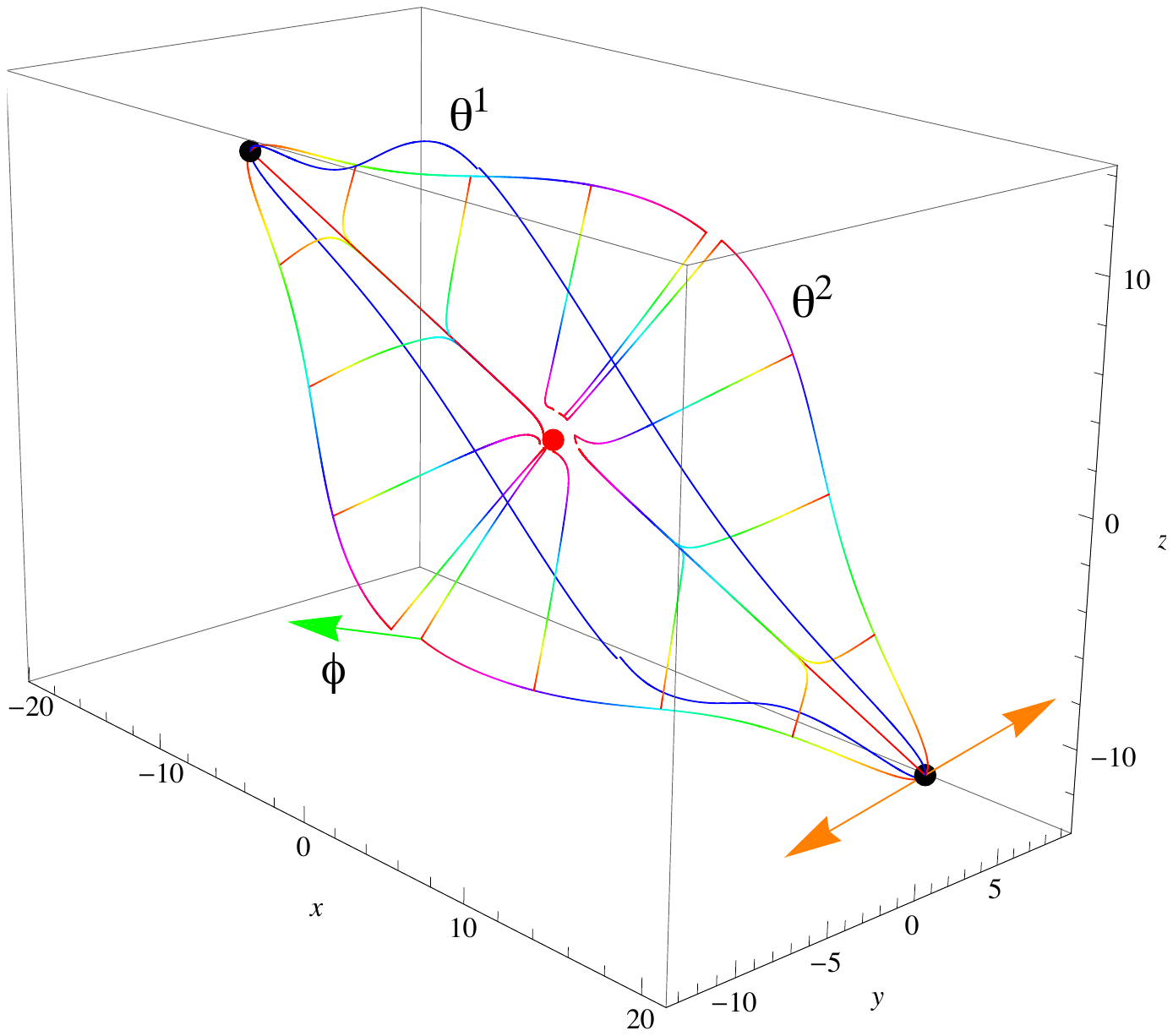}
\caption{The $\rho$ and $\Theta$ coordinate curves for the curvilinear Gaussian potential.}
\label{CGCoordsWithEig}
\end{center}
\end{figure}
 
Figure \ref{CGCoordsWithEig} shows a local coordinate system for the three-dimensional example that was depicted in Figures \ref{CurvilinearGaussian} and \ref{StreamPlotStack} in Section \ref{ProbM}, and for which we computed the global coordinate systems in Figure \ref{3DIntManifold}. We will use this illustration to explain the necessary calculations, but we are interested in generalizing from three dimensions to $n$ dimensions, and we will see that the higher dimensional case sometimes requires a slightly different treatment.  In either case, there are three steps:
\begin{itemize}
\item Step One: 
\vspace{1ex}

\noindent
To find a \emph{principal axis} for the $\rho$ coordinate, we minimize the Riemannian distance, ${g}_{i,j}({\bf x})$, along the drift vector.
\vspace{1ex}

\item  Step Two: 
\vspace{1ex}

\noindent
To choose the \emph{principal directions} for the $\theta^{1}, \theta^{2}, \ldots,$ $ \theta^{k - 1}$ coordinates, we diagonalize the Riemannian matrix, $\left( \,{g}_{i,j}({\bf x})\, \right)$, at a fixed point,   ${\bf x}_{0}$, along the principal axis.
\vspace{1ex}

\item  Step Three: 
\vspace{1ex}

\noindent
To compute the $\Theta$ \emph{coordinate curves}, we follow the geodesics of the Riemannian metric, ${g}_{i,j}({\bf x}_{0})$, in each of the $k - 1$ principal directions.
\end{itemize}
It turns out that Step Two requires the greatest modifications in the higher dimensional case, so we will defer that discussion and consider Step One and Step Three first.  

In Step One, we compute the \emph{principal axes}, which are the curves in Figure \ref{CGCoordsWithEig} that extend from the black dots to the red dot.  One way to do this is to find a point at a fixed Euclidean distance from the origin for which the integral curve, $\rho(t)$, as defined in Equation \eqref{IntegralCurveDiffEq} or \eqref{IntegralCurveIntEq}, has minimal Riemannian length to the origin.  However, if $| \nabla U({\bf x}) |$ is monotonic, we obtain approximately the same results (i.e., within the accuracy of our numerical approximations) by simply minimizing $| \nabla U({\bf x}) | ^{2}$ on a sphere at a constant distance from the origin.  For example, to find the black dot in the lower right corner of Figure \ref{CGCoordsWithEig}, we first minimize $| \nabla U(x,y,z) | ^{2}$ on the sphere $x^{2} + y^{2} + z^{2} = 500$ to locate a point and a curve, $\rho(t)$, with Riemannian length $6.30873$.  We then search within the neighborhood of this solution to find a point and a curve with a slightly smaller Riemannian length: $6.30863$.  In Figure \ref{CGCoordsWithEig}, we have chosen the latter solution, but in an $n$-dimensional space, the second solution is usually too complex, computationally, and we will therefore stick with the simple minimization of $| \nabla U({\bf x}) | ^{2}$.  

Let's now jump ahead to Step Three.  Setting $\gamma(0) = {\bf x}_{0}$, which is the black dot on the principal axis in Figure \ref{CGCoordsWithEig}, and setting $ \mathbf{v}(0) $ equal to one of the principal directions from Step Two, we solve the Euler-Lagrange equations.  The results are the \emph{coordinate curves} $\theta^{1}$ and $\theta^{2}$ in Figure \ref{CGCoordsWithEig}. These curves have interesting geometric properties.  Since they are geodesics on the surface of the Frobenius integral manifold, they are always a constant Riemannian distance from the origin.  Note that we have drawn several  $\rho$ coordinate curves in Figure \ref{CGCoordsWithEig} from the $\theta^{2}$ curve to the origin. The Riemannian length of these curves is approximately $6.30863$, even though their Euclidean length varies considerably.  (See Theorem \ref{ConstantRho} and Theorem \ref{ConstantRiemannian}, \emph{supra}.)  On the other hand, the Riemannian distance along the $\theta^{1}$ and $\theta^{2}$ curves is the same as the Euclidean distance, and we can therefore use the Riemannian/Euclidean arc length to parameterize these curves, just as we have done with the $\rho$ coordinate curves.

In practice, when we try to solve the Euler-Lagrange equations in Step Three of our procedure, we often encounter a singularity, as the value of $P_{0}({\bf x})$ approaches zero. The solution is to shift the ``center'' of our coordinate system to a different coordinate axis, and then continue the computation with a new value for $P_{0}({\bf x})$.  To validate this solution, though, we need to show that our definition of the Riemannian metric, ${g}_{i,j}({\bf x})$, is independent of our choice of a ``centered'' coordinate axis.  For the three-dimensional case, the proof appears in Section 4 of \citep{CCCS_AMAI}, and the proof for $n$ dimensions is a straightforward generalization.  Let $ \mathbf{u}( \mathbf{x} )$ denote a $\rho$, $\Theta$, coordinate system centered on the $x^{1}$ axis, and let $ \mathbf{\bar{u}}( \mathbf{x} )$ denote a $\rho$, $\Theta$, coordinate system centered on the axis $x^{p}$, for some $ p \neq 1 $. (For example, look at the second image in Figure \ref{3DIntManifold}, in which $ p = 2 $ and the centered coordinate axis is $x^{2}$.)  The Jacobian matrix of the coordinate transformation from $ \mathbf{\bar{u}}( \mathbf{x} )$ to $ \mathbf{u}( \mathbf{x} )$ can be computed as follows:
\begin{equation*}
\left(\begin{array}{c} \;\\ \partial u^{i} / \partial \bar{u}^{k} \\ \; \end{array}\right)  \; = \; \left(\begin{array}{c} \;\\ \partial x^{j} / \partial u^{i} \\ \; \end{array}\right)^{-1} \left(\begin{array}{c} \;\\ \partial x^{j} / \partial \bar{u}^{k} \\ \; \end{array}\right)
\end{equation*}
Now let ${g}_{i,j}({\bf x})$ and ${\bar{g}}_{k,l}({\bf x})$ denote the dissimilarity metric based on the $u^{i}$ and $\bar{u}^{k}$ coordinates, respectively.  We can verify by a straightforward computation (compare the three-dimensional example in \citep{CCCS_AMAI}) that
\begin{align}
{\bar{g}}_{k,l}({\bf x}) \;=\;  \sum_{i,j=1}^{n-1} \frac{\partial u^{i}}{\partial \bar{u}^{k}} \; {g}_{i,j}({\bf x}) \; \frac{\partial u^{j}}{\partial \bar{u}^{l}}  \notag
\end{align} 
But this is just an instantiation of the transformation law for a type (0,2) tensor.  Thus, on an integral manifold of dimension $n-1$, for a fixed $\rho$, the dissimilarity metric, ${g}_{i,j}({\bf x})$, is independent of the global coordinate system used to define it. 

Once we have computed the geodesic coordinate curves $\theta^{1}, \theta^{2}, \ldots,$ $ \theta^{k - 1}$, all of which emanate from ${\bf x}_{0}$, we also need to construct a system of \emph{transverse} coordinate curves, which can emanate from any point, ${\bf x}$, on the integral manifold.  For this purpose, we define the following \emph{flows}:
\begin{equation}
\label{TransverseCurves}
\vec{\theta}_{t}({\bf x}) =  \hat{\theta}_{{\bf x}}(t) = {\bf x} + \int_{0}^{t} \;  \sum_{i=1}^{n-1}  v^{i}_{\theta}(s) {\bf V}_{i}( \hat{\theta}_{{\bf x}}(s) ) \; ds 
\end{equation}
Here, $\hat{\theta}_{{\bf x}}(t)$ is an integral curve starting at ${\bf x}$, as in \eqref{IntegralCurveIntEq}, and it has an equivalent definition by a differential equation, as in \eqref{IntegralCurveDiffEq}.  We have one such flow equation for each geodesic, $\theta^{1}, \theta^{2}, \ldots,$ $ \theta^{k - 1}$, computed in Step Three, in which the coefficient vectors, $ \mathbf{v}_{\theta^{1}}(s)$,  $\mathbf{v}_{\theta^{2}}(s)$, $\ldots$, $\mathbf{v}_{\theta^{k-1}}(s) $, are the functions, $ \mathbf{v}(t) $, that were computed at the same time.  When $ {\bf x} = {\bf x}_{0} $, of course, the integral curve in Equation \eqref{TransverseCurves} coincides with the original geodesic coordinate curve.  In the $n$-dimensional case, we usually apply Equation \eqref{TransverseCurves} in a fixed order to a sequence of coordinates to specify a point on the integral manifold.  Thus, we start with the flow $\vec{\theta}^{\,1}_{t}({\bf x}_{0})$ and follow it for a distance $t_{1}$ to the point ${\bf x}_{1}$;  we then follow the flow $\vec{\theta}^{\,2}_{t}({\bf x}_{1})$ for a distance $t_{2}$ to the point ${\bf x}_{2}$; and so on.  In the three-dimensional case, as shown in Figure \ref{CGCoordsWithEig}, it is convenient to define an additional geodesic coordinate curve, $\phi$, which is orthogonal to both $\theta^{1}$ and $\theta^{2}$.  We then have three flows, $\vec{\theta}^{\,1}_{t}({\bf x})$, $\vec{\theta}^{\,2}_{t}({\bf x})$, and $\vec{\phi}_{t}({\bf x})$, and we can compose them either as $ \vec{\phi}_{s}  \circ \vec{\theta}^{\,1}_{t}({\bf x}_{0}) $ or as $ \vec{\phi}_{s}  \circ \vec{\theta}^{\,2}_{t}({\bf x}_{0}) $.  (We will see later that this choice leads to an interesting comparative analysis.)  Whatever choices we make, though, we want to make sure that our coordinate system covers the entire integral manifold.  

Let's now return to Step Two, where we encounter a surprising mathematical fact.  When we diagonalize the Riemannian matrix, $\left( \,{g}_{i,j}({\bf x})\, \right)$, in $n$ dimensions, we find only two distinct eigenvalues. The largest eigenvalue has multiplicity $2$: $\lambda_{0} = \lambda_{1} =  | \nabla U |^{2}$, with corresponding eigenvectors: 
\begin{align}
\xi_{0} =\left(\begin{array}{c} 1 \\ 0 \\ 0 \\ \cdots \\ 0 \\ 0 \end{array}\right) \;\; \mbox{\rm and} \;\;
\xi_{1} =\left(\begin{array}{c} 0 \\ P_{1} \\ P_{2} \\ \cdots \\ P_{n-2} \\ P_{n-1} \end{array}\right) \;
\notag 
\end{align} 
The smallest eigenvalue has multiplicity $n-2$: $\lambda_{2} = \lambda_{3} = \ldots = \lambda_{n-2} = \lambda_{n-1} =  P_{0}^{2} $, with corresponding eigenvectors $\xi_{2}, \;\xi_{3}, \;\ldots , \;\xi_{n-2}, \;\xi_{n-1}$, as follows:
\begin{align}
\left(\begin{array}{c} 0 \\ -P_{2} \\ P_{1} \\ 0 \\ \cdots \\ 0 \\ 0 \\ \end{array}\right) , \;
\left(\begin{array}{c} 0 \\ -P_{3} \\ 0 \\ P_{1} \\ \cdots \\ 0 \\ 0 \\ \end{array}\right) , \;
\ldots \; , \;
\left(\begin{array}{c} 0 \\ -P_{n-2} \\ 0 \\ 0 \\ \cdots \\  P_{1} \\ 0 \\ \end{array}\right) , \;
\left(\begin{array}{c} 0 \\ -P_{n-1} \\ 0 \\ 0 \\ \cdots \\  0 \\ P_{1} \\ \end{array}\right) 
\notag 
\end{align} 
This means that the three-dimensional case is special, since it only has three eigenvectors:
\begin{align}
\xi_{0} =\left(\begin{array}{c} 1 \\ 0 \\ 0 \end{array}\right)  , \;\;
\xi_{1} =\left(\begin{array}{c} 0 \\ P_{1} \\ P_{2} \end{array}\right)  , \;\;
\xi_{2} =\left(\begin{array}{c} 0 \\ -P_{2} \\ P_{1}\\ \end{array}\right) 
\notag 
\end{align} 
See Section 4 of \citep{CCCS_AMAI}. We should think of these $\{ \xi_{i} \}$ as \emph{infinitesimal} eigenvectors.  They tell us the maximal and minimal \emph{initial} directions for the integrand of the energy functional:
\begin{align}
\xi_{1}^{\top} \begin{pmatrix} \vspace{0.25ex} {g}_{i,j}  \vspace{0.5ex} \\ \end{pmatrix} \xi_{1} \;&=\;   \xi_{1}^{\top} ( \lambda_{1} \, \xi_{1} ) \;=\;   \lambda_{1} \, | \xi_{1} |^{2} \;=\;   | \nabla U |^{2} \, | \xi_{1} |^{2} 
\label{maxEig} \\
\xi_{2} ^{\top} \begin{pmatrix} \vspace{0.25ex} {g}_{i,j}  \vspace{0.5ex} \\ \end{pmatrix} \xi_{2}  \;&=\;   \xi_{2} ^{\top} ( \lambda_{2} \, \xi_{2}  ) \;=\;   \lambda_{2} \, | \xi_{2}  |^{2} \;=\;  P_{0}^{2} \, | \xi_{2}  |^{2}
\label{minEig}
\end{align} 
However, we are primarily interested in maximizing or minimizing geodesic curves over \emph{finite} distances, and there is no guarantee that maximizing or minimizing the initial directions of the geodesics in the Euler-Lagrange equations will achieve this result.  

Our solution to this problem in \citep{CCCS_AMAI} is to rotate the infinitesimal eigenvectors around ${\bf x}_{0} $, and then compute the Riemannian/Euclidean distances along each rotated geodesic, up to some specified point, for example, up to the Euclidean angle $ \pi / 2 $ from the origin, so that we can determine the \emph{global} minimum or maximum.  In Figure \ref{CGCoordsWithEig}, the orange arrows depict the minimal infinitesimal eigenvector, $\xi_{2}$, in the positive $y$-direction and the negative $y$-direction, respectively.  But the minimal geodesics over a finite distance, labelled as $\theta^{2}$, were obtained by a counter-clockwise rotation through the angle $\alpha = 0.952169$ in the positive direction, and the angle $\alpha = 1.12681$ in the negative direction.  For the maximal infinitesimal eigenvector, $\xi_{1}$, the maximal geodesic over a finite distance, labelled as $\theta^{1}$, was obtained by a counter-clockwise rotation through the angle $\alpha = 0.114166$.  See Section 5.2 of \citep{CCCS_AMAI} for the detailed calculations. 

\begin{figure}[tb]
\begin{center}
\includegraphics[width=4.0in]{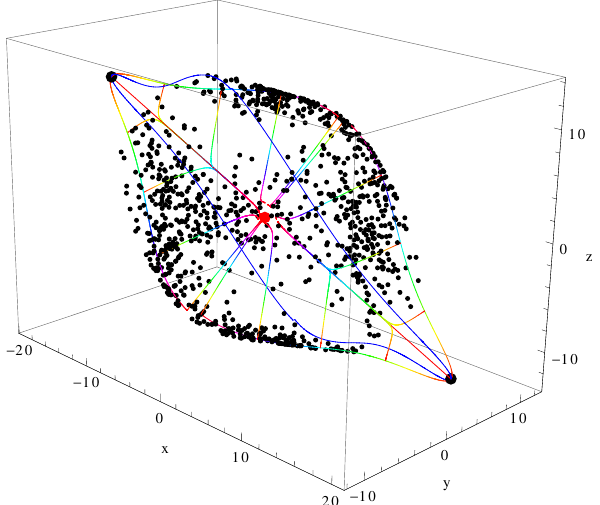}
\caption{Projecting data points from the curvilinear Gaussian potential along the $\rho$ coordinate curves to the Frobenius integral manifold.}
\label{CGCoordsWithData}
\end{center}
\end{figure}

Now look at Figure \ref{CGCoordsWithData}.  This figure includes the same view of the $\theta^{1}$ and $\theta^{2}$ coordinate curves as in Figure \ref{CGCoordsWithEig}. It also includes $1000$ data points generated according to the curvilinear Gaussian probability distribution from Figure \ref{CurvilinearGaussian}, and projected along the $\rho$ coordinate curve to the Frobenius integral manifold.  Qualitatively, the density of the data is low along the principal axis (the black dots), higher along the $\theta^{1}$ axis (the blue curve), and highest along the $\theta^{2}$ axis (the multi-colored curve). We can quantify this observation by computing a value for the ``reconstruction error'' adapted to our curvilinear coordinate system.  For this purpose, we use the geodesic coordinate curve, $\phi$, which is aligned with the green arrow in Figure \ref{CGCoordsWithEig}, and its flow, $\vec{\phi}_{s}({\bf x})$.  We want to compare two curvilinear coordinate systems for each data point: $(\rho,\theta^{1},\phi)$, which is defined by following the flow $ \vec{\phi}_{s}  \circ \vec{\theta}^{\,1}_{t}({\bf x}_{0}) $ on the integral manifold, and $(\rho,\theta^{2},\phi)$, which is defined by following the flow $ \vec{\phi}_{s}  \circ \vec{\theta}^{\,2}_{t}({\bf x}_{0}) $.  What happens when we drop the $\phi$ coordinate, in each case?  We can measure the ``reconstruction error'' by computing the Euclidean distance along the flow, $ \vec{\phi}_{s} $, and scaling this value down, proportionately, by the position of the data point along the $\rho$ coordinate curve.  We can then compute the \emph{root-mean-squared} (RMS) reconstruction error for the 1000 data points in each coordinate system.  It turns out that the RMS error for the truncation from $(\rho,\theta^{1},\phi)$ to $(\rho,\theta^{1})$ is 5.9431, and the RMS error for the truncation from $(\rho,\theta^{2},\phi)$ to $(\rho,\theta^{2})$ is 4.82787.  Thus, if our goal is to \emph{minimize} the ``reconstruction error,'' then the ``optimal'' lower dimensional encoding would be $(\rho,\theta^{2})$.  See Section 5.2 of \citep{CCCS_AMAI} for the detailed calculations. 
 
However, we show in Section 5.1 of \citep{CCCS_AMAI} that this phenomenon, in three dimensions, is actually an artifact of the projection of data points along the $\rho$ coordinate curve to the Frobenius integral manifold, and it largely disappears in higher dimensional spaces.   An alternative is to compute an analogue of the ``variance'' along the $\theta^{1}$ and $\theta^{2}$ coordinate curves in our curvilinear coordinate system.   For the truncation from $(\rho,\theta^{1},\phi)$ to $(\rho,\theta^{1})$, we simply measure the Euclidean arc length to the data point at $(\rho,\theta^{1})$ along the blue coordinate curve from a point at a Euclidean angle of $ \pi / 2 $ from the origin.  For the 1000 data points, the RMS dispersion of this quantity is 9.47889.  Similarly, for the truncation from $(\rho,\theta^{2},\phi)$ to $(\rho,\theta^{2})$, we measure the Euclidean arc length to the data point at $(\rho,\theta^{2})$ along the multi-colored coordinate curve from a point at a Euclidean angle of $ \pi / 2 $ from the origin. For the 1000 data points, the RMS dispersion of this quantity is 7.70126. Thus, if our goal is to \emph{maximize} the ``variance,'' then the ``optimal'' lower dimensional encoding would be $(\rho,\theta^{1})$. See Section 5.2 of \citep{CCCS_AMAI} again for the detailed calculations. 

Overall, from a careful analysis of the two three-dimensional examples in Sections 5.1 and 5.2 of \citep{CCCS_AMAI}, we arrive at several principles for the construction of a $( \rho, \Theta )$ coordinate system:
\begin{quotation}
\noindent
$\ldots$ (1) The geodesics computed from the infinitesimal eigenvectors of $\left( {g}_{ij}({\bf x}) \right)$ are not optimal, but we can rotate the coordinate system to find maximal and minimal geodesics up to a point on the surface at an angle of $\pi/2$ from the origin. (2) To compare geodesics, we can compare the ``variance'' of the projected data points along the geodesic curves, and choose the ones that maximize this quantity. (3) A comparison of the ``reconstruction error'' is not reliable, but might give reasonable results in higher dimensional cases. 
 
\vspace{1ex}
\noindent
See \citep{CCCS_AMAI}, at the end of Section 5.1.
\end{quotation}
Indeed, we will see that the higher dimensional case is different in two ways.
 
First, rotating the eigenvectors to find the optimal initial directions is not likely to work in $n$ dimensions.  In Figure \ref{CGCoordsWithEig}, we rotated the maximal and minimal infinitesimal eigenvectors in a two-dimensional plane, and we were able to compute and compare the geodesics under each rotation.  But to find the optimal value, we used a very naive algorithm, specifically, we applied \emph{Mathematica}'s \texttt{FindMaximum} or \texttt{FindMinimum} to a function from the angle $\alpha$ to the length of the geodesic.  Since each additional dimension in the ambient Euclidean space adds an additional dimension to the rotation matrix, taking us from $ \mathrm{SO}(2) $ to $ \mathrm{SO}(n - 1) $, the naive algorithm would be too complex.  In the present paper, we will simply skip the optimization step and use the infinitesimal eigenvector directly.  In fact, there are reasons to think that the error from this omission is small in a higher dimensional space.  For example, as we add dimensions and move from $ \mathrm{SO}(2) $, to $ \mathrm{SO}(3) $, to $ \mathrm{SO}(4) $, the angle between the maximal infinitesimal eigenvector and the maximal geodesic over a finite distance decreases. We will return to this issue in Section \ref{FutureWork} on Future Work.

Second, our measures of ``variance''  and ``reconstruction error'' are aligned in the higher dimensional case.  We will see that the maximal infinitesimal eigenvector (either with or without a rotation) generates a coordinate curve that maximizes the ``variance'' of the projected data points, and we will choose this as our first $\Theta$ coordinate.  We can then search for an optimal ordering of the remaining $\Theta$ coordinates. Note that any linear combination of the eigenvectors, $\xi_{2}, \;\xi_{3}, \; \ldots , \;\xi_{n-2}, \;\xi_{n-1}$, is also an eigenvector associated with the eigenvalue $\lambda_{2} = P_{0}^{2} $. We can therefore define new eigenvectors in the form:
\begin{equation}
\zeta = c_{2}\, \xi_{2} + c_{3}\, \xi_{3} +  \ldots +  c_{n-2}\, \xi_{n-2} +  c_{n-1}\, \xi_{n-1},
\label{linearCombEig}
\end{equation}
where the coefficients, $c_{i}$, are either $+1$ or $-1$.  (There are $2^{n-2}$ sequences of such coefficients, of course, but in practice we can generate a large subset randomly.) If we evaluate the integrand of the energy functional on these new eigenvectors,  we have:
\begin{equation}
\zeta^{\top} \begin{pmatrix} \vspace{0.25ex} {g}_{i,j}  \vspace{0.5ex} \\ \end{pmatrix} \zeta \;=\;   \zeta^{\top} ( \lambda_{2} \, \zeta ) \;=\;   \lambda_{2} \, | \zeta |^{2} \;=\;  P_{0}^{2} \, | \zeta |^{2}
\label{zetaEig}
\end{equation}
Now, order the eigenvectors, $\{ \zeta \}$, by their norms and apply the Gram-Schmidt orthogonalization procedure to a subsequence of length $n - 2$ or more.  We will then have an orthonormal basis for the $n - 2$ dimensional tangent subbundle at ${\bf x}_{0} $, in which each basis vector represents a minimal initial direction for the solution of the Euler-Lagrange equations, by \eqref{zetaEig}, and we will be able to construct a lower dimensional subspace that maximizes the ``variance''  and  minimizes the ``reconstruction error'' at the same time.  We will see how to do this, using real data, in Sections  \ref{AppMNIST} and \ref{AppCIFAR10}, \emph{infra}.
 
\begin{figure}[tbp]
\begin{center}
\includegraphics[width=4.0in]{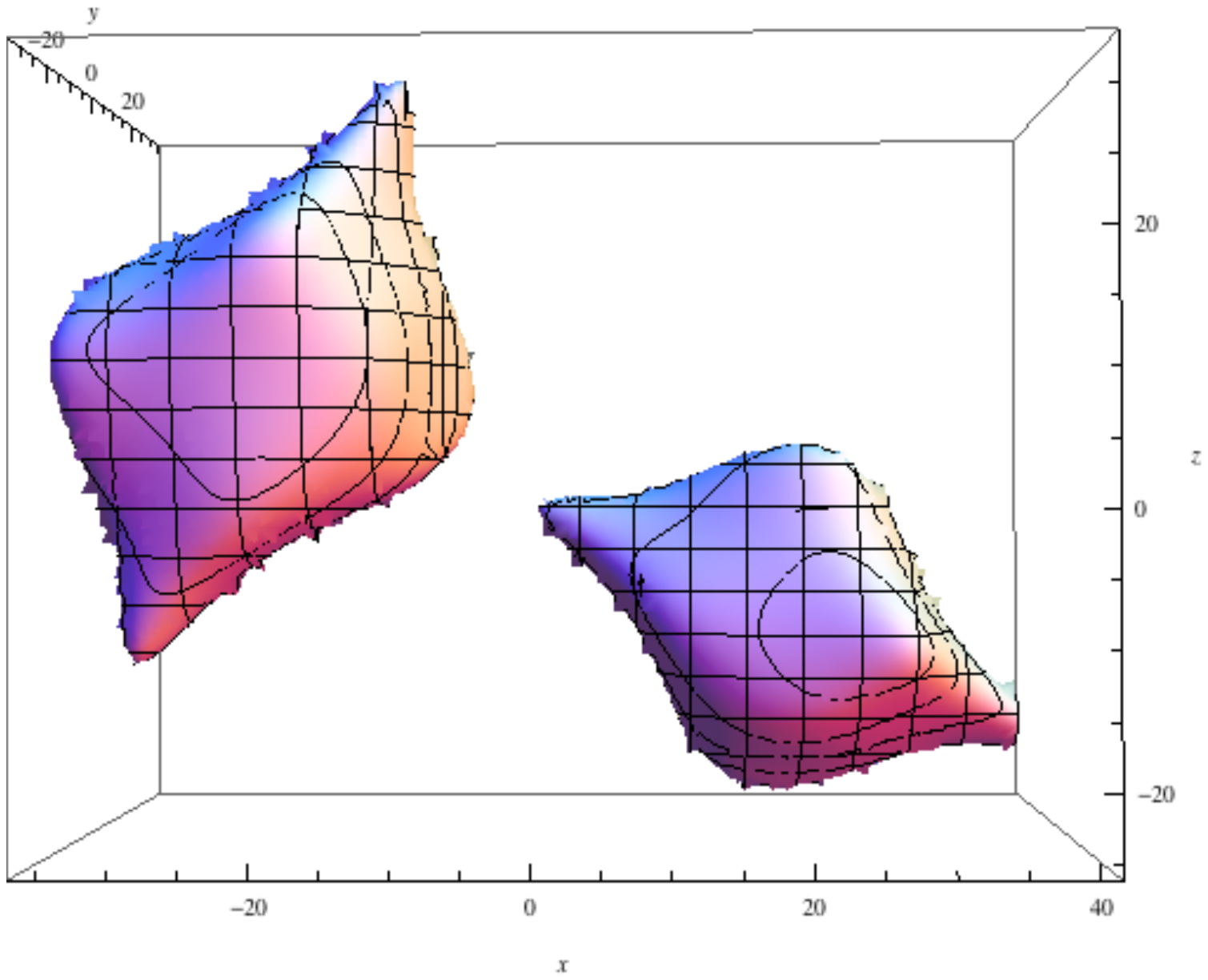}
\caption{A mixture of two curvilinear Gaussians, translated and rotated.}
\label{BiCurvilinearGaussian}
\end{center}
\end{figure}

We will also analyze in Sections \ref{AppMNIST}  and \ref{AppCIFAR10} a representation of \emph{clusters}, using real data.  The basic idea is illustrated for the three-dimensional case in Figure \ref{BiCurvilinearGaussian}, which shows two copies of the curvilinear Gaussian potential from Figure \ref{CurvilinearGaussian}.  One copy has been translated from $(0,0,0)$ to $(20,20,-10)$.  The other copy has been translated from $(0,0,0)$ to $(-20,-20,10)$ and rotated by $\pi / 2$ around a line parallel to the $y$-axis.  But the probability density is a \emph{mixture}.  If $U_{1}({\bf x})$ is the potential function for the first copy and $U_{2}({\bf x})$ is the potential function for the second copy, then the invariant probability density is given by:
\[
e^{2 U({\bf x})} \; \simeq \; p_{1} \, e^{\,2\,U_{1}({\bf x})}  +  p_{2} \, e^{\,2\,U_{2}({\bf x})},
\]
modulo an appropriate normalization factor. The advantage of this representation lies in the fact that our calculations for each copy will be almost independent of each other.  Observe that the effective potential function for the mixture will be: 
\[
U({\bf x}) \; \simeq \; \frac{1}{2}  \log ( \, p_{1} \, e^{\,2\,U_{1}({\bf x})}  +  p_{2} \, e^{\,2\,U_{2}({\bf x})} \, )
\]
Thus the gradient of $U({\bf x})$ in a neighborhood of  $(20,20,-10)$ will be almost identical to the gradient of $U_{1}({\bf x})$ computed by itself, and the gradient of $U({\bf x})$ in a neighborhood of  $(-20,-20,10)$ will be almost identical to the gradient of $U_{2}({\bf x})$ computed by itself.  Or, in terms of our dissimilarity metric, the two clusters in Figure \ref{BiCurvilinearGaussian} will be exponentially far apart.

\section{How to Estimate $\nabla U({\bf x})$ from Sample Data.}
\label{Estimator}

To apply the theory of differential similarity to real data, we need to estimate the quantities that appear in the equations for the $\rho$, $\Theta$, coordinate system.  Fortunately, this is not hard to do:  We will borrow a technique from the literature on the \emph{mean shift algorithm} \citep{Fukunaga&Hostetler:1975, Cheng:1995, Comaniciu&Meer:2002}. 

\begin{figure}[htbp]
\begin{center}
\includegraphics[width=5.0in]{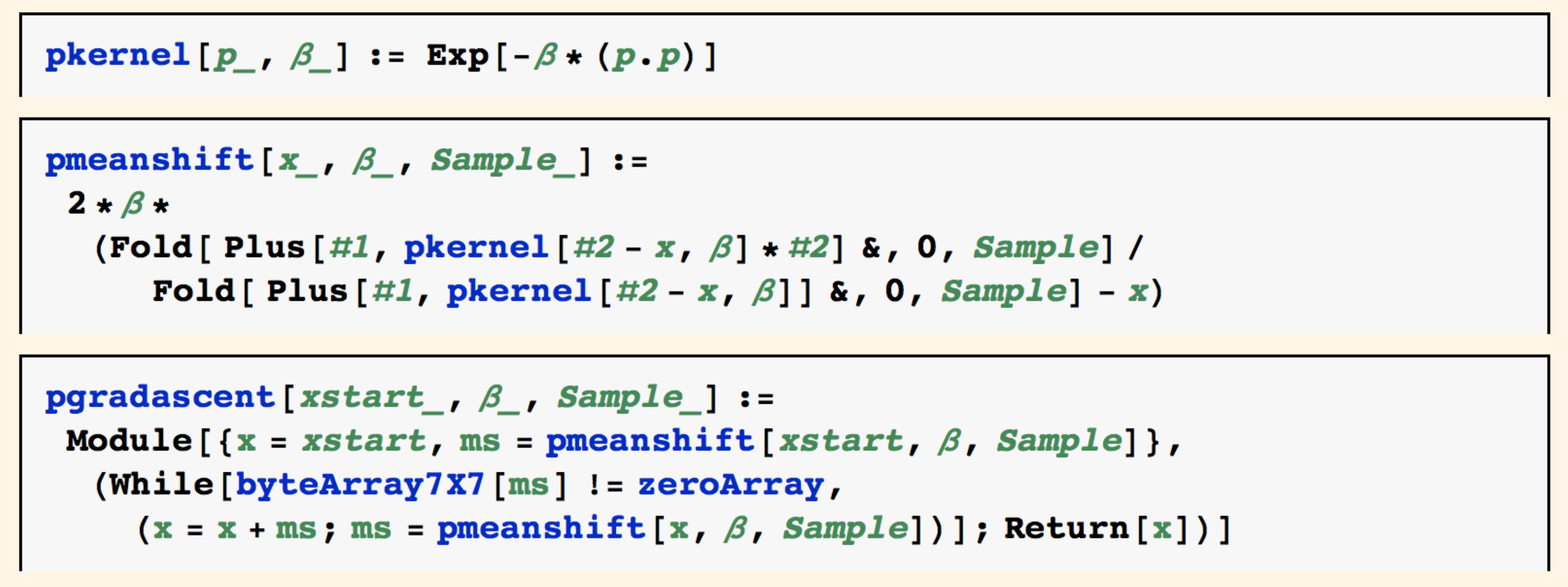}
\caption{The mean shift algorithm.}
\label{Kernel}
\end{center}
\end{figure}

Consider a \emph{kernel density estimator} with a Gaussian kernel:
\begin{equation*}
K(\mathbf{s}_{k}, \mathbf{x}) = \exp( - \beta \; \| \mathbf{s}_{k} - \mathbf{x} \|^{2})
\end{equation*}
in which $\mathbf{s}_{k}$ is a sample data point and $\beta$ is a smoothing parameter.  We can approximate a probability density by taking the average over these kernels:
\begin{equation*}
\hat{\mu}( \mathbf{x} ) = \frac{1}{n} \sum_{k=1}^{n} K(\mathbf{s}_{k}, \mathbf{x})
\end{equation*}
Now recall that $\nabla U({\bf x})$ is the \emph{gradient} of the \emph{log} of the stationary probability density in our theory.  So we can differentiate explicitly:
\begin{equation}
\frac{\partial}{\partial x^{j}} \log \hat{\mu}( \mathbf{x} ) = 2 \beta \Bigg[ \;  \frac{ \sum_{k=1}^{n} K(\mathbf{s}_{k}, \mathbf{x}) \; s_{k}^{j} }{ \sum_{k=1}^{n} K(\mathbf{s}_{k}, \mathbf{x}) }  \; - \; x^{j} \; \Bigg]
\label{EstimateGradU}
\end{equation}
to obtain an estimate for $\nabla U({\bf x})$. 

Figure \ref{Kernel} shows an example of \emph{Mathematica} code that implements Equation \eqref{EstimateGradU} and calls it \texttt{pmeanshift}. The expression ``mean shift'' refers to the fact that the first term inside the square brackets computes the weighted mean of the sample data points, $\{ \mathbf{s}_{k} \}$, around the point $\mathbf{x}$, and the second term subtracts $\mathbf{x}$ from this value to define a shift.  This function is typically used in a gradient ascent algorithm, such as \texttt{pgradascent}.  (In Figure \ref{Kernel}, \texttt{byteArray7X7} and \texttt{zeroArray} are tailored to the examples in Sections \ref{AppMNIST} and \ref{AppCIFAR10}, \emph{infra}.)  In this case, the gradient ascent algorithm will find the \emph{mode} of the probability density, $\hat{\mu}( \mathbf{x} )$, for a given \texttt{Sample}.  It is often applied iteratively:  Choose a sample around $\mathbf{x}$, ascend to the mode of the probability density to find a new $\mathbf{x}$, choose another sample, and repeat. 

We can certainly apply the code in Figure \ref{Kernel} to find the origin of an admissible coordinate system in our theory, and we will see examples of this kind of an application in Sections \ref{AppMNIST} and \ref{AppCIFAR10}.  However, Equation \eqref{EstimateGradU} has a much broader application than this, because every quantity that enters into the definition of the $\rho$ coordinates and the $\Theta$ coordinates depends on $\nabla U({\bf x})$.  We will actually modify \eqref{EstimateGradU} slightly to simplify these calculations.  The quantity $ \sum_{k=1}^{n} K(\mathbf{s}_{k}, \mathbf{x}) $ in the denominator of the first term is a normalization factor, and we can multiply the formula by this factor (and divide out the factor $ 2 \beta $) to write down an equation for the gradient without normalization:  
\begin{equation*}
\mathtt{DU} _{j} ( \mathbf{x} ) =   \sum_{k=1}^{n} K(\mathbf{s}_{k}, \mathbf{x})  \; s_{k}^{j}   \; - \; x^{j} \sum_{k=1}^{n} K(\mathbf{s}_{k}, \mathbf{x})  
\end{equation*}
The second derivatives are now much easier to compute:
\begin{equation*}
\frac{\partial}{\partial x^{i}} \mathtt{DU}_{j} ( \mathbf{x} ) =  2 \beta \sum_{k=1}^{n} K(\mathbf{s}_{k}, \mathbf{x})  (x^{i} - s_{k}^{i} ) (x^{j} - s_{k}^{j})    -  \delta_{i,j} \sum_{k=1}^{n} K(\mathbf{s}_{k}, \mathbf{x}) , 
\end{equation*}
where $ \delta_{i,j} $ is the Kronecker delta function.  These modifications will not alter the specification of our coordinate system, for two reasons: (1) the definition of the $\rho$ coordinate in Equations \eqref{IntegralCurveDiffEq} and \eqref{IntegralCurveIntEq} is already normalized, so the factor $ \sum_{k=1}^{n} K(\mathbf{s}_{k}, \mathbf{x}) $ would cancel out; and (2) the basis vectors $ \{ {\bf V}_{i}({\bf x}) \} $ can be divided by $ P_{0}({\bf x}) $ without affecting the Frobenius integral manifold, so again the factor $ \sum_{k=1}^{n} K(\mathbf{s}_{k}, \mathbf{x}) $ would cancel out.  The code for $ \mathtt{DUjF} $ and $ \mathtt{DiDUjF} $ is shown in Figure \ref{gradU}.  The variables $ {\beta}\mathtt{eta} $ and $ \mathtt{SamplePoints } $ are intended to be defined externally, and the vector variable $ \mathtt{VexpK} $ is expected to be instantiated by the kernel function $ \mathtt{VkernelF} $.

\begin{figure}[htbp]
\begin{center}
\includegraphics[width=5.0in]{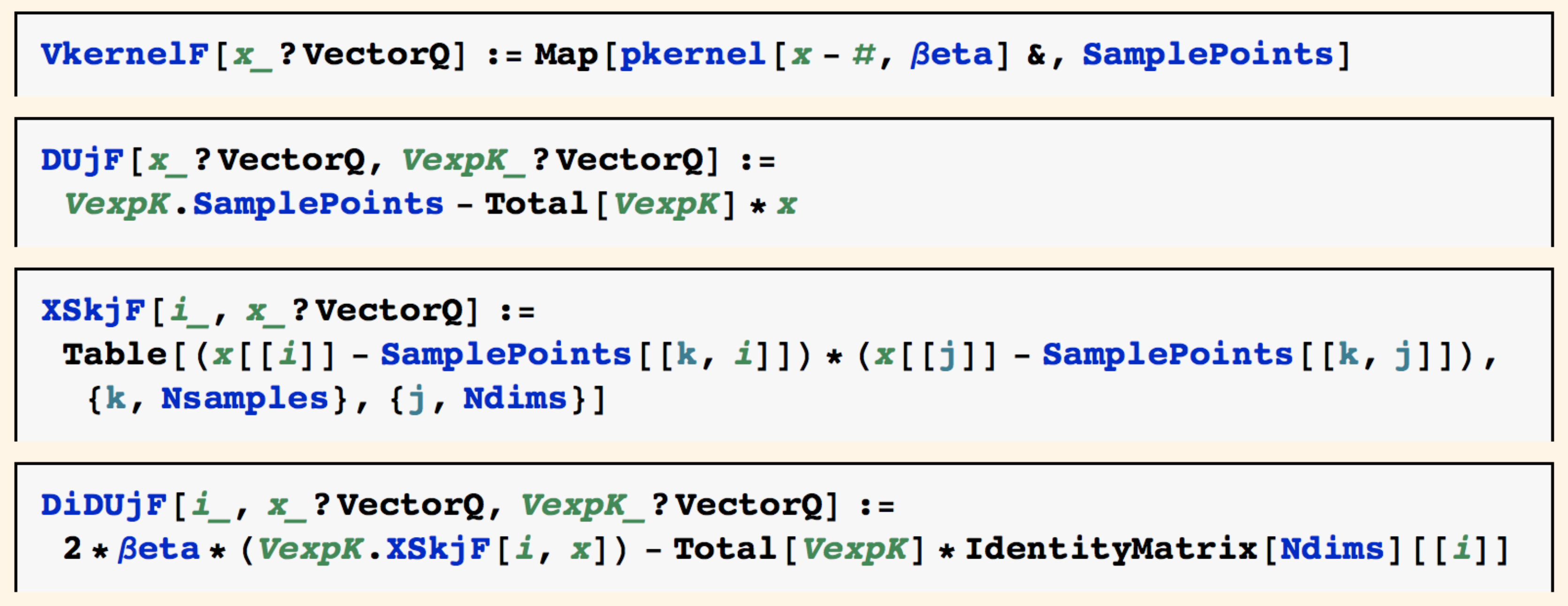}
\caption{The gradient (without normalization) and its derivatives.}
\label{gradU}
\end{center}
\end{figure}

\begin{figure}[htbp]
\begin{center}
\includegraphics[width=5.0in]{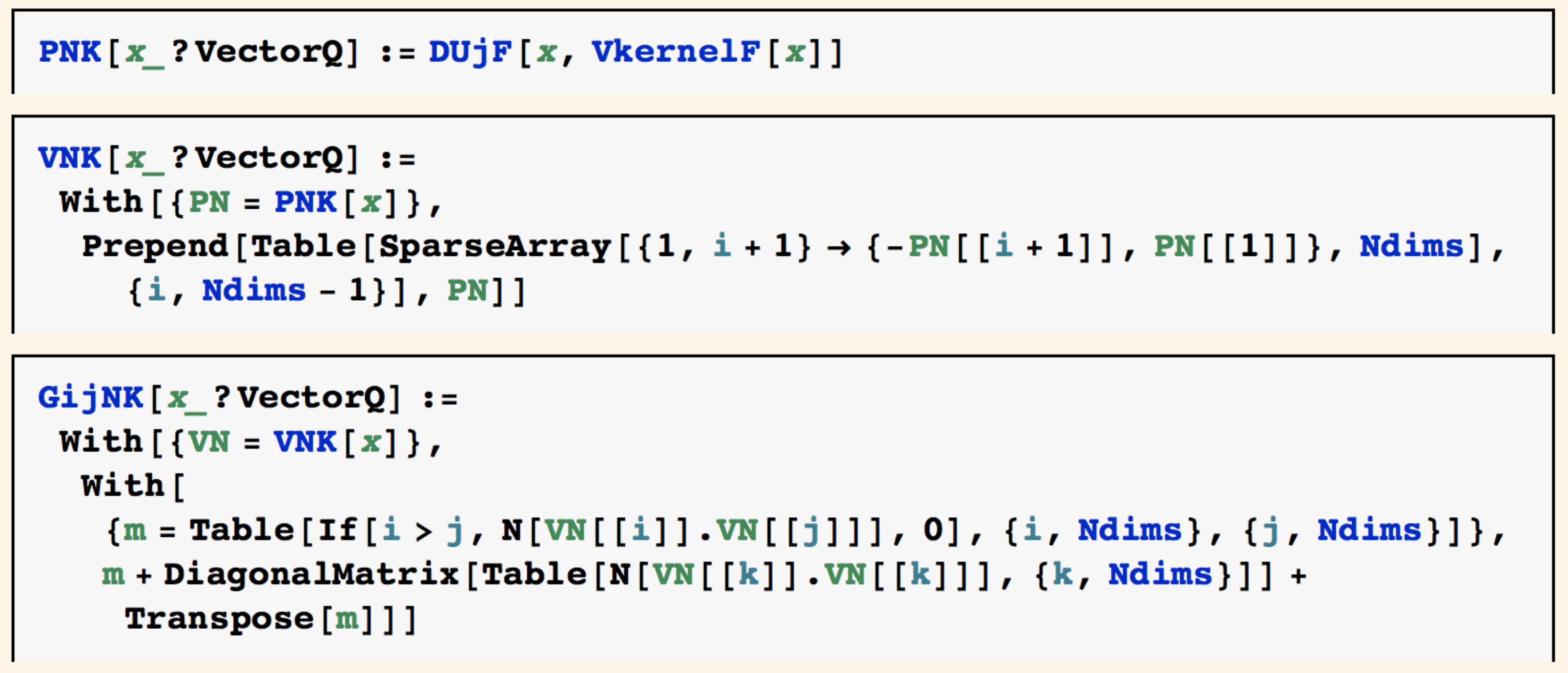}
\caption{The gradient: $(\, P_{0}({\bf x}) ,  P_{1}({\bf x}) ,  \ldots , P_{d - 1}({\bf x})\, )$, basis vectors: $ \{\, {\bf V}_{i}({\bf x})\, \} $, and Riemannian dissimilarity metric: $\left( \,{g}_{i,j}({\bf x})\, \right)$.}
\label{Gij}
\end{center}
\end{figure}
 
\section{Computing the Geodesic Coordinate Curves.}
\label{GeoCoords}
 
The main data structures in our theory are defined in Figure \ref{Gij}.  But the only data structure needed to write down the Euler-Lagrange equations is the gradient: 
$ \nabla U({\bf x}) \,=\, (\,  P_{0}({\bf x}) ,  P_{1}({\bf x}) , \ldots , P_{d - 1}({\bf x}) \,) $.  Here are the Euler-Lagrange equations: 
\begin{align} 
\lambda_{i-1}^{\prime}( t ) \;=\; & \left(  P_{0}[ \mathbf{x}( t )] \sum_{j=1}^{d-1} (v^{j})^{2}( t ) -  \sum_{j=1}^{d-1} v^{j}( t ) \, \lambda_{j}( t )  \right) \,  \frac{\partial P_{0}}{\partial x^{i}}[ \mathbf{x}( t )] \;\; + \notag \\
 & \left(  \sum_{j=1}^{d-1}  P_{j}[ \mathbf{x}( t )] \, v^{j}( t ) \; + \;   \lambda_{0}( t )  \right)  
\,  \sum_{j=1}^{d-1}   \frac{\partial P_{j}}{\partial x^{i}}[ \mathbf{x}( t )] \, v^{j}( t ), \notag \\ 
   &  \mathrm{for} \hspace{0.5em} i = 1, \ldots, d \notag \\
   \notag \\
v^{i}( t ) \;=\; &  \frac{1}{P_{0}[ \mathbf{x}( t )]} \left(  \lambda_{i}( t ) - P_{i}[ \mathbf{x}( t )] \, \frac{  \sum_{j=0}^{d-1} P_{j}[ \mathbf{x}( t )] \,  \lambda_{j}( t ) }{ \sum_{j=0}^{d-1}  P_{j}^{2}[ \mathbf{x}( t )] }  \right), \notag \\
   &  \mathrm{for} \hspace{0.5em} i = 1, \ldots, d-1 \notag \\
   \notag \\
(x^{1})^{\prime}( t ) \;=\; &  -  \sum_{j=1}^{d-1}  P_{j}[ \mathbf{x}( t )] \, v^{j}( t ) \notag \\
(x^{i})^{\prime}( t ) \;=\; & P_{0}[ \mathbf{x}( t )] \, v^{i-1}( t ), \; \mathrm{for} \hspace{0.5em} i = 2, \ldots, d  \notag
\end{align}
This is a system of ordinary differential equations for $ \mathbf{x}( t ) $ and for the Lagrange multipliers, $ \{ \lambda_{i}( t ) \} $, plus a system of algebraic equations for $ \mathbf{v}( t ) $, in which $d$ is the dimensionality of the space.  In addition, we have to include the set of kernel equations, which enter into the definition of the gradient.  So this is a very large system of equations, approximately $ 3 * \mathtt{Ndims} + \mathtt{Nsamples} $.  

\begin{figure}[htbp]
\begin{center}
\includegraphics[width=5.0in]{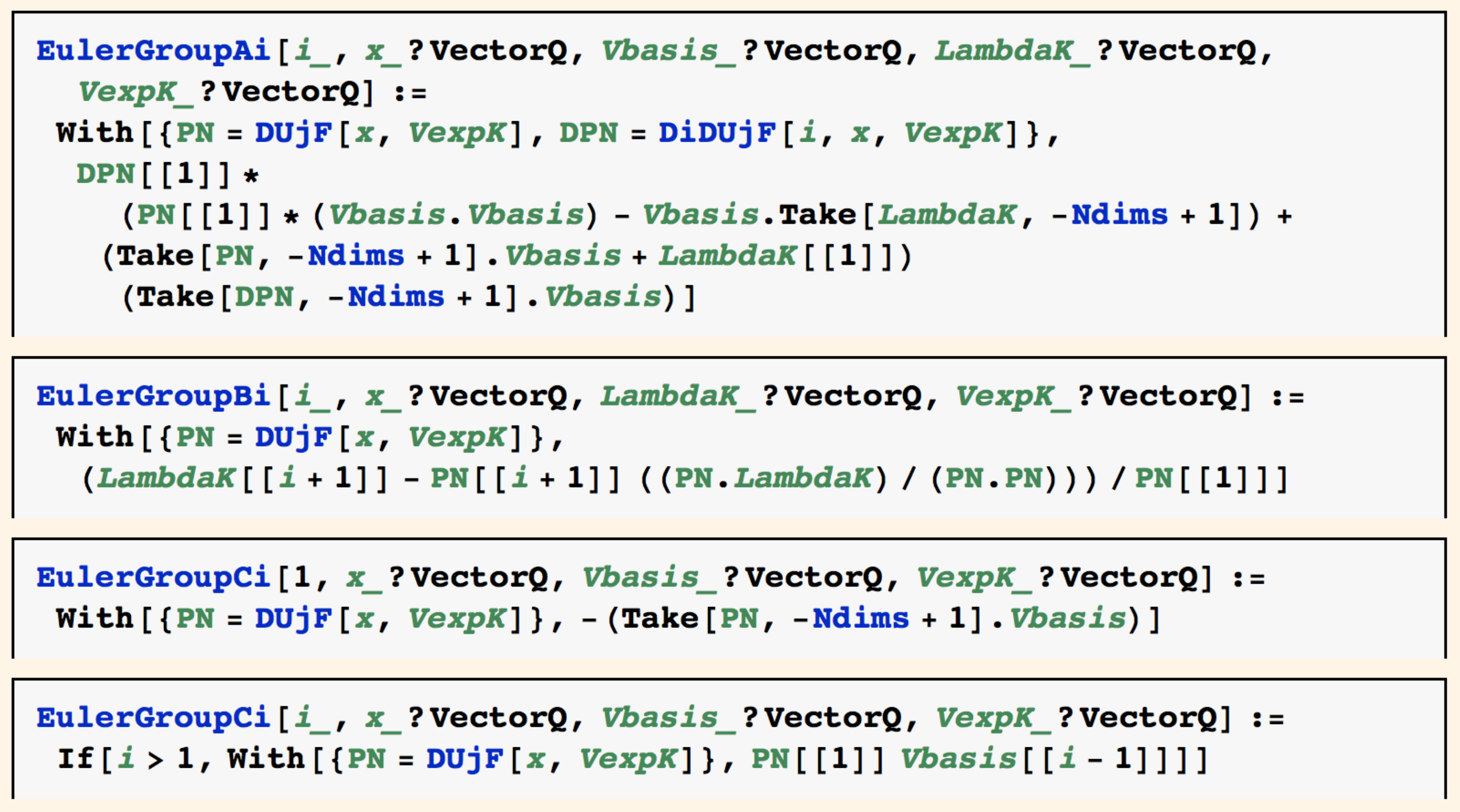}
\caption{The Euler-Lagrange equations in \emph{Mathematica}.}
\label{EulerEqns}
\end{center}
\end{figure}

Figure \ref{EulerEqns} shows the \emph{Mathematica} code for the right-hand sides of the Euler-Lagrange equations, in three groups.  Tracing this code back  from Figure \ref{EulerEqns} to Figure \ref{gradU}, it is easy to verify that everything depends on the set of $ \mathtt{SamplePoints } $ and the constant $ {\beta}\mathtt{eta} $.  Thus we should be able to apply this code to a real dataset.  We will see how this works in Section \ref{AppMNIST} on the MNIST dataset and in Section \ref{AppCIFAR10} on the CIFAR-10 dataset.


\section{Example: 7$\times$7 Patches in the MNIST Dataset.}
\label{AppMNIST}

\begin{figure}[htbp]
\begin{center}
\includegraphics[width=4.5in]{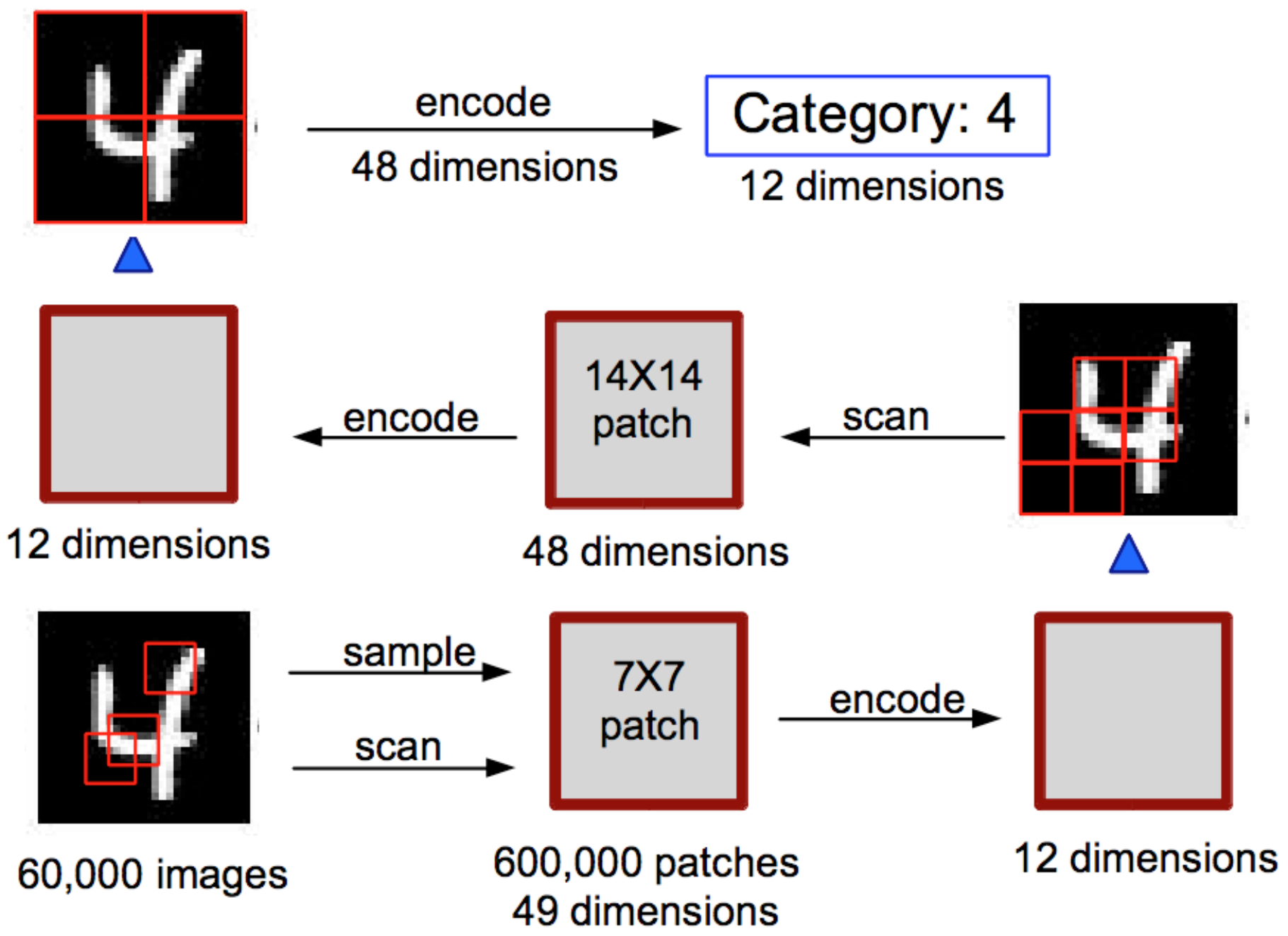}
\caption{An architecture for deep learning on the MNIST Dataset.}
\label{ArchitectureDL}
\end{center}
\end{figure}

Figure \ref{ArchitectureDL} displays an architecture for deep learning on the MNIST dataset \citep{LeCun_etal_1998}, based on several examples in the recent literature \citep{Ranzato:2009, Coates:2012}. The process starts in the lower-left corner and follows the arrows to the upper-right corner. The first step is to scan and randomly sample the 60,000  28$\times$28 images to extract a collection of 7$\times$7 ``patches'' from each one. Choosing a sampling rate of 10 scans per image, which is approximately 2\%, we end up with 600,000 patches, each one represented as a point in a 49-dimensional space. The original image intensity at each pixel is represented by an integer in the range $[0,255]$, but we have scaled these values down to a real number in the range $[0,1]$. Thus the greatest distance between any two points in our 49-dimensional hypercube is $\sqrt{49} = 7.0$. We will be using the theory of differential similarity, in this section of the paper, to reduce the dimensionality of the space to 12 dimensions, as shown in the lower-right corner of Figure \ref{ArchitectureDL}.  Note that the process continues upwards (following the blue arrow) by assembling four adjacent 7$\times$7 patches into a 2$\times$2 matrix and then resampling the image using the larger 14$\times$14 patch, but we will not pursue this analysis in the present paper.  We are thus looking only at the first step in the process, which is an example of the classical problem of \emph{unsupervised feature learning}.

There are several parameters that control the behavior of the algorithms defined in Sections \ref{Estimator} and \ref{GeoCoords}. The smoothing parameter, $\beta$, can be set to different values in different circumstances.  (In the traditional notation for a Gaussian, $\beta = 1\, / \, 2  \sigma^{2}$.) We have experimented with a range of values, but we have found that the two values, $\beta = 1/8$ (or $\sigma = 2.0$) and $\beta = 1$ (or $\sigma = 0.707107$), are sufficient for most purposes.  A typical strategy is to run a computation with $\beta = 1/8$, for a coarse approximation, and then to refine the result by running the computation again with $\beta = 1$. Another parameter is the variable $\mathtt{SamplePoints}$. \emph{Mathematica} provides a function called \texttt{Nearest} which, when applied to a large dataset, returns a \texttt{NearestFunction} which maintains an efficient data structure that can find the $n$ points that are the nearest to any specified point in the dataset.  Again, we have experimented with several values of $n$, but we have found that a \emph{Data Sphere} of 32,000 points around each prototype yields good results.  Within each Data Sphere, we draw two random 800 point samples and use these as inputs to the algorithms in Sections \ref{Estimator} and \ref{GeoCoords}. A typical strategy is to run the computations separately on each 800 point sample, and if the results are qualitatively the same and quantitatively within the range that we would expect from random sampling, then we run a final computation on the union of the two samples.  Note that a 1600 point sample is $5\%$ of a 32,000 point Data Sphere.

In addition to the definition of the Data Sphere, it is also useful to define a \emph{Coordinate Sphere} for each prototype.  Recall that the principal axis is defined by a point at a fixed Euclidean distance from the origin that has a minimal Riemannian distance to the origin, as measured along the $\rho$ coordinate curve.  The fixed Euclidean distance gives us a sphere, of course, and it is reasonable again to think of the size of this sphere in terms of the number of data points it contains.  Combined with a 32,000 point Data Sphere, we have found that an 8,000 point Coordinate Sphere yields good results.

\begin{figure}[tb]
\begin{center}
\includegraphics[width=5.0in]{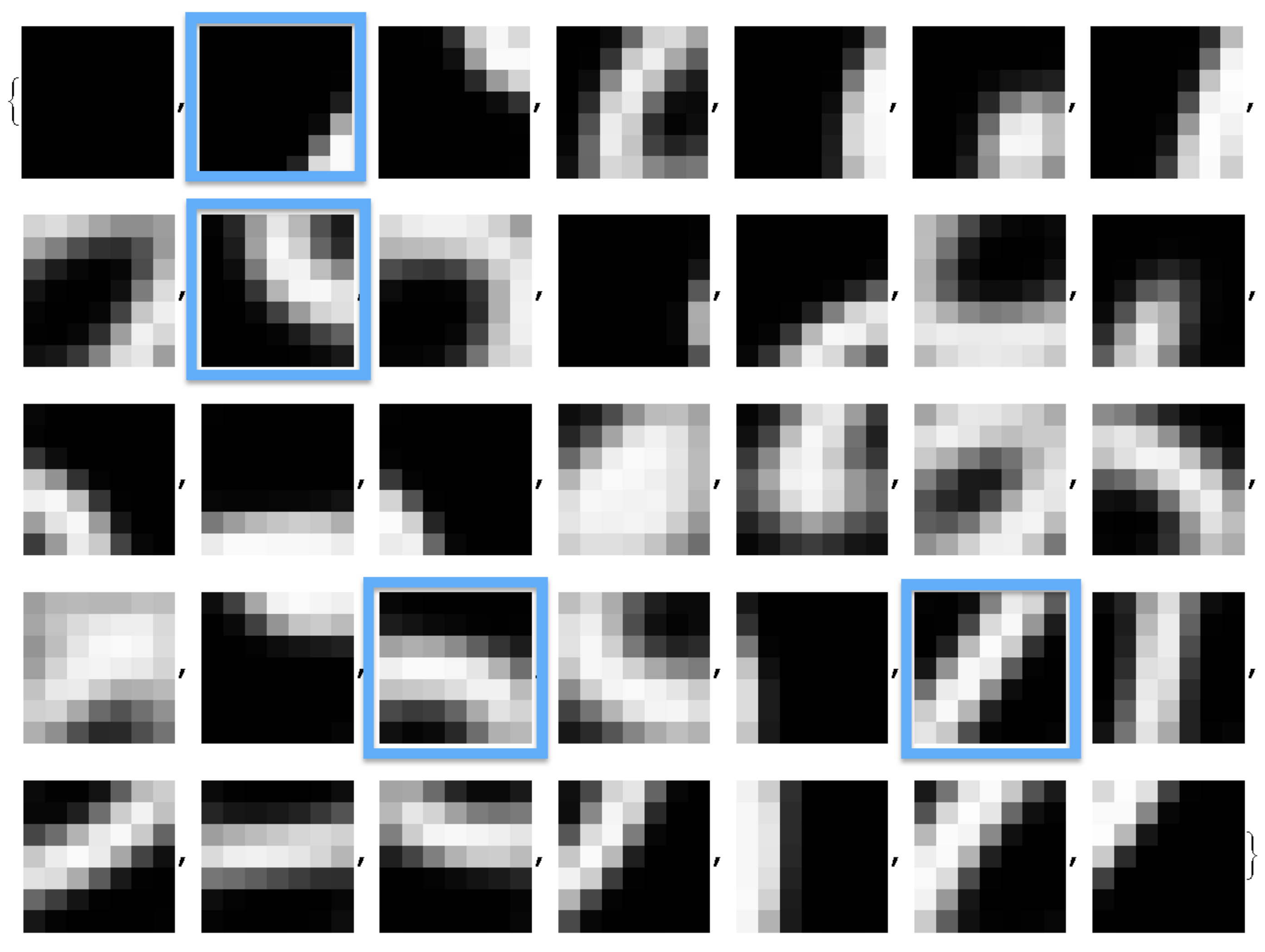}
\caption{35 prototypes for the 600,000 7$\times$7 patches in the MNIST Dataset.}
\label{Prototypes}
\end{center}
\end{figure}

Figure \ref{Prototypes} shows an initial selection of 35 prototypes for the 600,000 7$\times$7 patches. The prototypes outlined in blue were selected (subjectively) for a more detailed investigation.  (Prototype 02, in which most of the pixels are black, was only used to provide a lower dimensional test case for the development of our algorithms) We will therefore focus our attention on the three examples: Prototype 09, Prototype 24, Prototype 27.  The full set of 35 prototypes was constructed by a combination of: (i) \texttt{pgradascent} in Figure \ref{Kernel}, with \texttt{xstart} generated randomly, with $\beta = 1/8$, and with a 2000 point \texttt{Sample} retrieved by the \texttt{NearestFunction} at each iteration; plus (ii) a traditional clustering algorithm (\texttt{FindClusters} in \emph{Mathematica}) applied to the output of \texttt{pgradascent}; and finally (iii) some manual pruning at the end of the process to eliminate redundancies. We then constructed a 32,000 point Data Sphere around each prototype, and used these for all the subsequent data analyses. Note that $35 \times 32,000 = 1,120,000$, so we can potentially partition the entire dataset almost twice over.  In fact, an analysis \emph{ex post} shows that the union of the 32,000 point Data Spheres around these 35 prototypes covers $92.5\%$ of the distinct 7$\times$7 patches in our sample.  So there is room for some improvement here, but not much.
  
Although each prototype is located at the center of its Data Sphere, by definition, it will not necessarily be located at a mode of the probability distribution, since we have altered the sample in \texttt{pgradascent} to include the nearest 32,000 data points.  Thus we now compute the integral curve of $\nabla U$, starting at the original prototype, to arrive at a modified prototype, where $\nabla U({\bf x}) = 0$. The results are shown in the center column of Figure \ref{Proto&prAxis}.  We then draw a Coordinate Sphere around each modified prototype, and apply Step One of the procedure defined in Section \ref{GeomM} 
to compute the points at which the $\rho$ coordinate curves drawn \emph{inwards} to the origin have minimal Riemannian length. These points are shown in the right column of Figure \ref{Proto&prAxis}. The $\rho$ coordinate curves themselves are shown in Figure \ref{rhoCurves3Proto}.  

\begin{figure}[htbp]
\begin{center}
\includegraphics[width=4.5in]{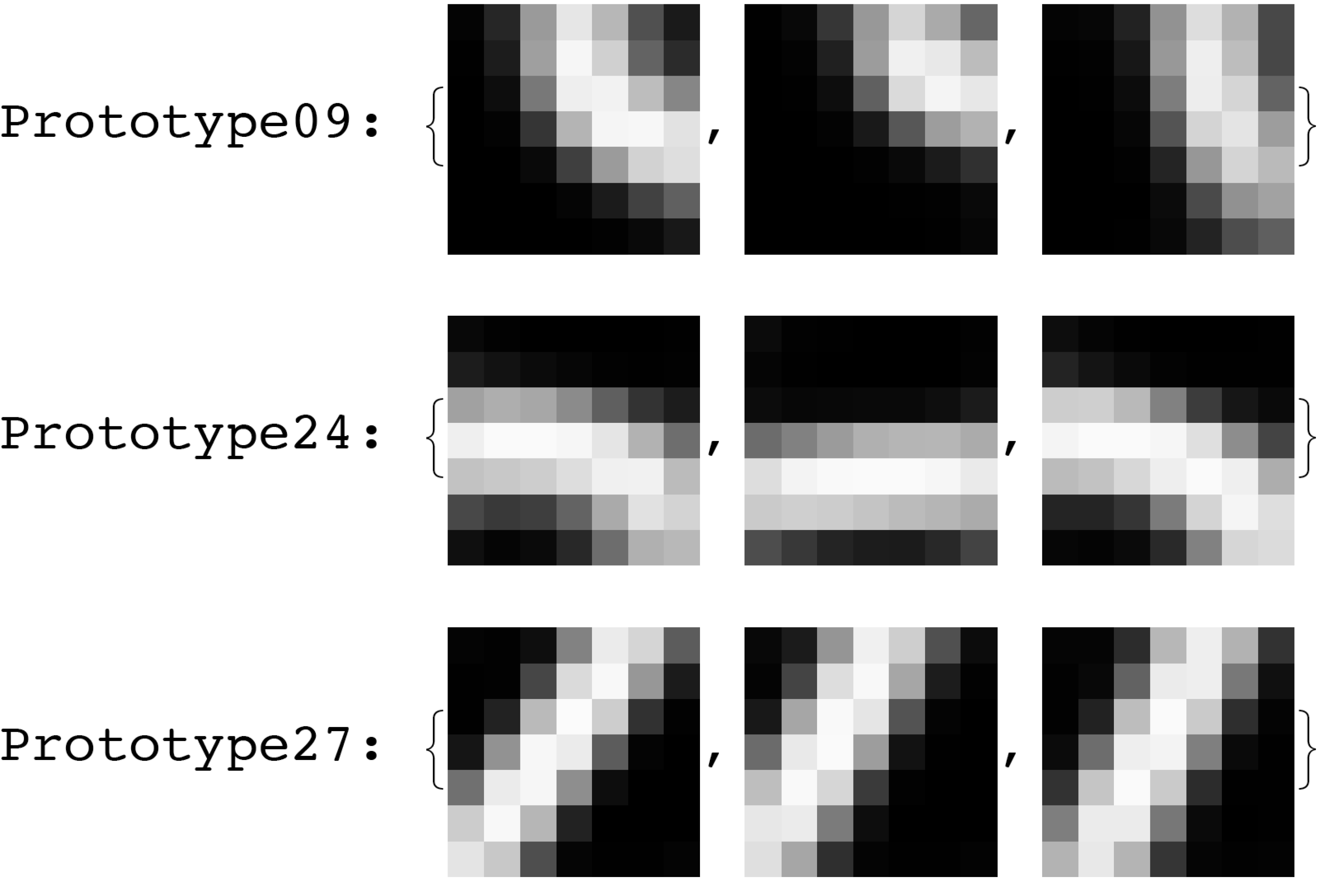}
\caption{Initial computations for three prototypes, (i) Left Column: original prototypes, from Figure \ref{Prototypes}; (ii) Center Column: modified prototypes, computed within the 32,000 point Data Spheres; (iii) Right Column: the data points computed on the principal axes, at their intersection with the 8,000 point Coordinate Spheres.}
\label{Proto&prAxis}
\end{center}
\end{figure}

\begin{figure}[htbp]
\begin{center}
\includegraphics[width=4.5in]{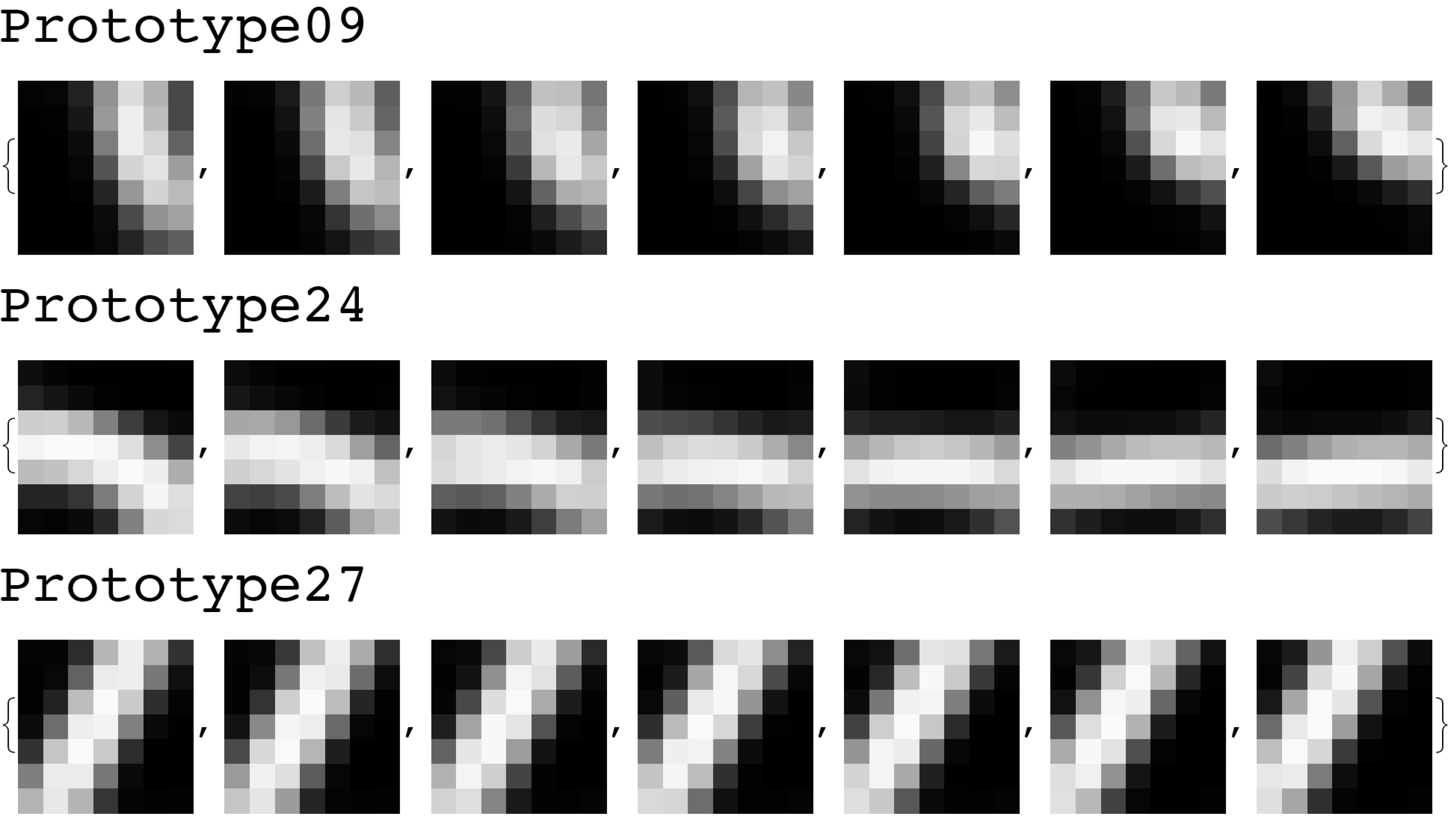}
\caption{The $\rho$ coordinate curves computed along the principal axes, for three prototypes.}
\label{rhoCurves3Proto}
\end{center}
\end{figure}

Table \ref{SphereRadius} shows the size of the Data Sphere and the Coordinate Sphere for the three \emph{prototypical clusters} that we are analyzing.  These values should be compared with the maximal distance in the hypercube ($\sqrt{49} = 7.0$), and with the standard deviation of the kernel density estimator when the smoothing parameter $\beta = 1$ ($\sigma = 0.707107$).

\begin{table}[htbp]
\caption{Radius in the original Euclidean space of the 32,000 point Data Sphere and the 8,000 point Coordinate Sphere, for three prototypical clusters.}
\begin{center}
\begin{tabular}{r|c|c|c|c|}
	& \hspace{3.2em} (a) \hspace{3.2em} & \hspace{3.2em} (b) \hspace{3.2em} \\ 
	& Radius of & Radius of	\\ 
	& Data Sphere & Coordinate Sphere \\ \hline
Prototype 09 &  2.95906  &  1.75322  \\ \hline
Prototype 24 & 3.15226   &  2.30487  \\ \hline
Prototype 27 & 2.97573   &  1.95364   \\ \hline
\end{tabular}
\end{center}
\label{SphereRadius}
\end{table}

The distances in Table \ref{rhoDistances} reveal some of the main properties of our geometric model.  The Euclidean distance along each $\rho$ coordinate curve, Column (c) in Table \ref{rhoDistances}, is greater than the Euclidean distance between its two end points, Column (b) in Table \ref{SphereRadius}. This is an indicator of the fact that the principal axis is a curve, and not a straight line. The Riemannian distance inwards along the $\rho$ coordinate curves, 
Column (d) in Table \ref{rhoDistances}, is much less than the Euclidean distance in Column (c).  This is an indicator of the fact that the principal axis is traversing a region of high probability density. But Column (e) in Table \ref{rhoDistances}, which displays the Riemannian distance computed \emph{outwards} along the $\rho$ coordinate curve for approximately the same Euclidean distance, is much larger because the probability density in this region is much smaller. 

\begin{table}[htbp]
\caption{Distances along the $\rho$ coordinate curves for the principal axes from a point on the Coordinate Sphere, both inwards and outwards. Compare Figures \ref{Proto&prAxis} and \ref{rhoCurves3Proto}.}
\begin{center}
\begin{tabular}{r|c|c|c|c|}
	& \hspace{2.2em} (c) \hspace{2.2em} & \hspace{2.2em} (d) \hspace{2.2em} & \hspace{2.2em} (e) \hspace{2.2em} \\ 
	& Euclidean & Riemannian & Riemannian  \\ 
	& Distance & Distance & Distance \\
	& Inwards & Inwards & Outwards \\ \hline
Prototype 09 &  2.40866  & 0.90783 & 4.01243 \\ \hline
Prototype 24   & 2.90017 &  0.764126 & 6.93029 \\ \hline
Prototype 27   &  2.0754  &  0.38668 & 3.41653 \\ \hline
\end{tabular}
\end{center}
\label{rhoDistances}
\end{table}
  
It is also interesting to analyze Figures  \ref{Proto&prAxis} and \ref{rhoCurves3Proto} qualitatively.  These images were constructed by solving a system of differential equations on pixels in a high-dimensional image space, a very local process, but they seem to possess certain global coherence properties.  For Prototype 09, we see a simple geometric shape modified by a shift transformation. The mapping from the left column to the center column in Figure \ref{Proto&prAxis} is basically a two-dimensional shift, one pixel up and one pixel to the right.  The mapping from the right column to the center column, which is displayed in Figure \ref{rhoCurves3Proto} as a discretely sampled continuous curve, shifts the same geometric shape up by two pixels.  For Prototype 27, we see another simple geometric shape modified by a sequence of global transformations. The mapping from the left column to the center column in Figure \ref{Proto&prAxis} is basically a counter clockwise rotation.  The lower left corner is fixed, approximately, while the upper end of the geometric shape is shifted one pixel to the left.  The mapping from the right column to the center column in Figure \ref{Proto&prAxis}, which is displayed in Figure \ref{rhoCurves3Proto} as a discretely sampled continuous curve, shifts the same geometric shape one pixel to the left.  Prototype 24 is somewhat more complex. Here the discretely sampled continuous mapping in Figure \ref{rhoCurves3Proto} shifts the left half of the geometric shape down by one pixel, while it bends the right half of the geometric shape up to construct a thick horizontal line. 
 
Keeping these coherence properties in mind, let's now look at Step Two and Step Three of the procedure defined in Section \ref{GeomM}.  In Step Two, we first diagonalize the Riemannian matrix, $\left( \,{g}_{i,j}({\bf x})\, \right)$, at the points on the principal axes shown in the right column of Figure \ref{Proto&prAxis}.  Taking Prototype 09 as an example, the two eigenvalues are 35.7443 and 13.572, and the infinitesimal eigenvectors are constructed to have unit norm.  Thus the integrand of the energy functional for the maximal infinitesimal eigenvector, $\xi_{1}$, is 35.7443. (See Equation \eqref{maxEig} in Section \ref{GeomM}.)  For the minimal infinitesimal eigenvectors, we adopt the following strategy:  We generate 10,000 random linear combinations of the eigenvectors, $\xi_{2}, \;\xi_{3}, \; \ldots , \;\xi_{48}$, using Equation \eqref{linearCombEig}, we sort these linear combinations by their norms, and we select every 200th entry in the sorted list.  We then apply the Gram-Schmidt orthogonalization procedure (using \texttt{Orthogonalize} in \emph{Mathematica}) to the selected eigenvectors to construct a set of 47 orthonormal basis vectors. For each of these basis vectors, the integrand of the energy functional is 13.572. (See Equation \eqref{zetaEig} in Section \ref{GeomM}.) Furthermore, since $\xi_{1}$ was orthogonal to all of the eigenvectors, $\xi_{2}, \;\xi_{3}, \; \ldots , \;\xi_{48}$, we now have an orthonormal basis for the 48-dimensional tangent subbundle at the point on the principal axis that will serve as the origin of our $\Theta$ coordinate curves. 
    
\begin{figure}[htbp]
\begin{center}
\includegraphics[width=4.5in]{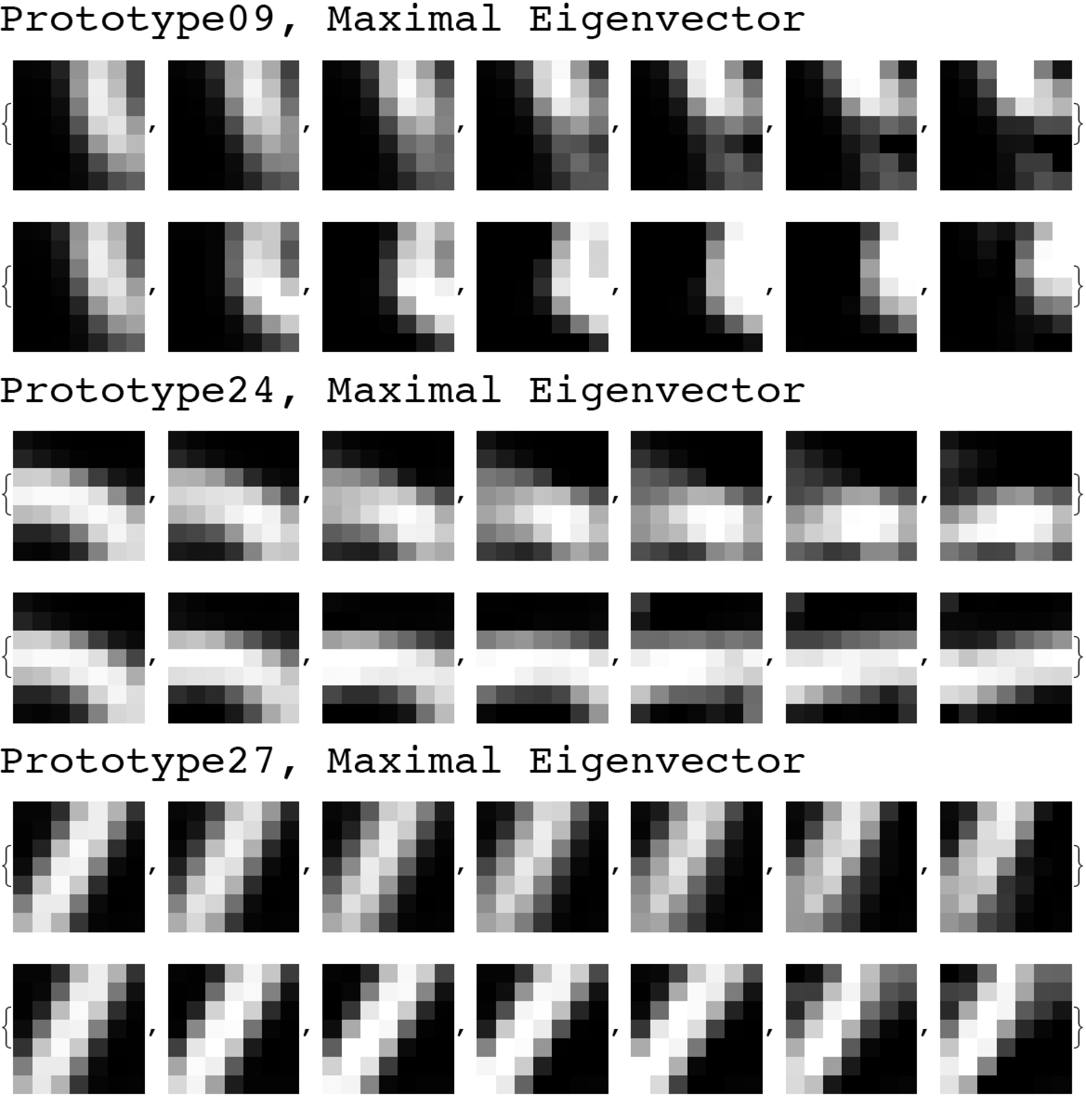}
\caption{$\Theta$ coordinate curves for the maximal infinitesimal eigenvectors. The first row in each case shows the coordinate curve in the direction of the infinitesimal eigenvector, to the Euclidean angle $ \pi / 2 $; the second row shows the coordinate curve in the opposite direction, to the Euclidean angle $ - \pi / 2 $.}
\label{ndCurvesMaxEig}
\end{center}
\end{figure}

Figure \ref{ndCurvesMaxEig} displays the $\Theta$ coordinate curves that were computed in Step Three by solving the Euler-Lagrange equations with the initial directions, $\mathbf{v}(0)$, set to the values of the \emph{maximal} infinitesimal eigenvectors, $\xi_{1}$. (The first row for each prototype follows in the positive direction; the second row follows in the negative direction; but this is an arbitrary convention.)  The curves are continued until they reach a point at a Euclidean angle of $ \pi / 2 $ from the origin.  Sometimes the computation encounters a singularity in the value of $P_{0}({\bf x})$, and it is necessary to shift the ``center'' of the coordinate system to a different coordinate axis in order to continue with the algorithm, but this does not happen more than once for each of these curves.  Table \ref{ThetaDistances} shows the distances along the coordinate curves to the Euclidean angles $ \pi / 2 $ in the positive direction and $ - \pi / 2 $ in the negative direction.  Recall that the Riemannian and Euclidean distances computed along a geodesic curve on the Frobenius integral manifold are the same, since each point on the manifold is a constant Riemannian distance from the origin. Thus there is only a single entry for the distances in Table \ref{ThetaDistances}, for each direction. 

\begin{table}[htbp]
\caption{Riemannian and Euclidean distances along the $\Theta$ coordinate curves for the maximal infinitesimal eigenvectors, from the principal axis to the Euclidean angles $ \pi / 2 $ and $ - \pi / 2 $. Compare Figure \ref{ndCurvesMaxEig}.}
\begin{center}
\begin{tabular}{r|c|c|}
	&  \hspace{2em} Positive  \hspace{2em} &  \hspace{2em} Negative  \hspace{2em}  \\ 
	& Direction & Direction   \\ \hline
Prototype 09 &  2.22261  &  4.16981  \\ \hline
Prototype 24 & 2.71304   &  3.58147   \\ \hline
Prototype 27 & 2.10106   &  2.46165   \\ \hline
\end{tabular}
\end{center}
\label{ThetaDistances}
\end{table}
  
\begin{figure}[htbp]
\begin{center}
\includegraphics[width=4.5in]{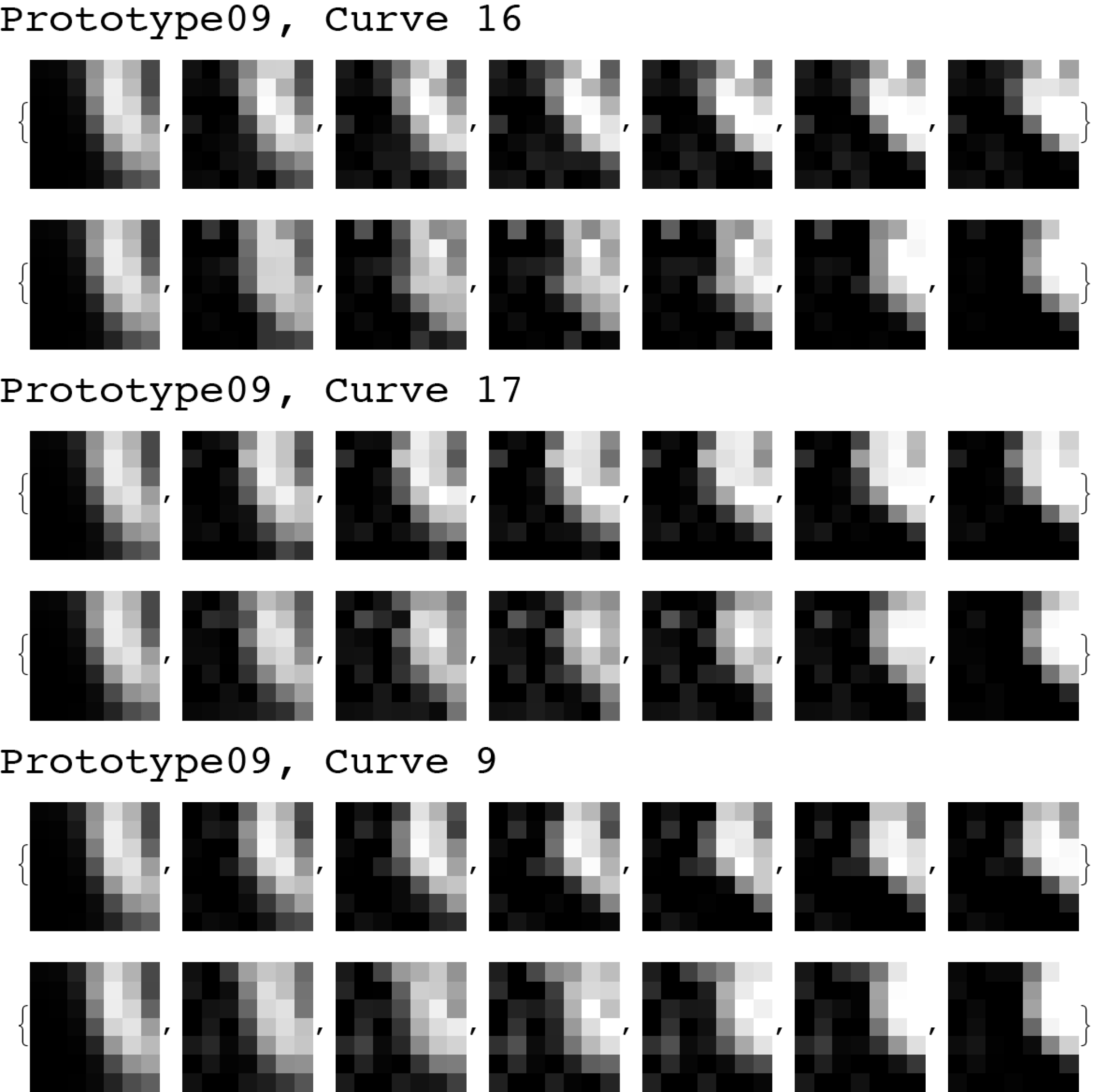}
\caption{The three minimal positive $\Theta$ coordinate curves for Prototype 09. The first row in each case shows the coordinate curve in the direction of the positive infinitesimal eigenvector, to $ \pi / 2 $; the second row shows the coordinate curve in the opposite direction, to $ - \pi / 2 $ .}
\label{ndCurvesTop3}
\end{center}
\end{figure}
  
For the \emph{minimal} infinitesimal eigenvectors, we will work exclusively with Prototype 09, as an example, but we will compute all 47 $\Theta$ coordinate curves. (Note:  All calculations up to this point have used a 1600 point sample, following a comparison of the results for two 800 point samples.  However, when we double the size of the sample from 800 to 1600, we increase the size of the data structures in Figure \ref{EulerEqns} by a factor of four, and we also increase the CPU time to find a solution to the Euler-Lagrange equations, empirically, by a factor of four.  Thus, since we are now computing 47 geodesics, in two directions, we will reduce the computational burden by restricting our analysis here to a single 800 point sample.) Figure \ref{ndCurvesTop3} shows the three $\Theta$ coordinate curves with the smallest Riemannian and Euclidean distances in the positive direction, which also happen to be the curves that have the largest Riemannian and Euclidean distances in the negative direction.  Table \ref{ndDistancesTop3} shows these distances, up to the Euclidean angles $ \pi / 2 $ and $ - \pi / 2 $, respectively.  The total distance is shown in the third column.  Note that the total distance along the $\Theta$ coordinate curve for the maximal infinitesimal eigenvector in Prototype 09, according to Table  \ref{ThetaDistances}, is considerably larger:  6.39242.

\begin{table}[htbp]
\caption{Prototype 09: Riemannian and Euclidean distances along the $\Theta$ coordinate curves to $ \pi / 2 $ and $ - \pi / 2 $ for three minimal infinitesimal eigenvectors. Compare Figure \ref{ndCurvesTop3}.}
\begin{center}
\begin{tabular}{r|c|c|c|}
	& Positive   & Negative    & Total  \\ 
	&  \hspace{1.0em} Direction \hspace{1.0em} &  \hspace{1.0em} Direction \hspace{1.0em} & \hspace{1.0em} Distance \hspace{1.0em} \\ \hline
Curve 16 &  2.61453  &  3.16092  & 5.77545 \\ \hline
Curve 17  & 2.68326  &  3.17738 & 5.86064 \\ \hline
Curve 09  & 2.68949   &  3.11284 & 5.80233 \\ \hline
\end{tabular}
\end{center}
\label{ndDistancesTop3}
\end{table}

\begin{figure}[tb]
\begin{center}
\includegraphics[width=4.5in]{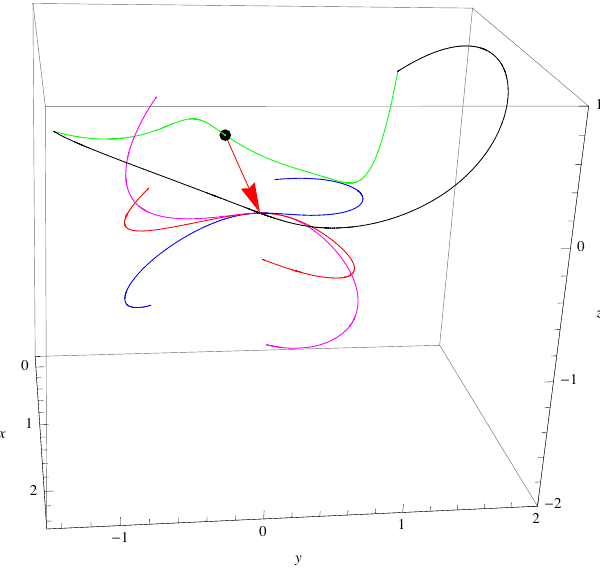}
\caption{A three-dimensional visualization of the coordinate curves for Prototype 09. The curves in Figure \ref{ndCurvesTop3} are coded by color: \emph{Red} is Curve 16, \emph{Magenta} is Curve 17, \emph{Blue} is Curve 9. The \emph{Black} curve is the $\Theta$ coordinate curve for the maximal infinitesimal eigenvector, as shown in Figure \ref{ndCurvesMaxEig}, and the \emph{Green} curves are the  $\rho$ coordinate curves at the Euclidean angles $\pi/2$ and $- \pi/2$ from the origin.}
\label{3DVisualization}
\end{center}
\end{figure}
    
Figure \ref{3DVisualization} shows another way to understand the relationship between the curves in Figures \ref{rhoCurves3Proto}, \ref{ndCurvesMaxEig} and \ref{ndCurvesTop3}:  a three-dimensional visualization.  How to map 49 dimensions down to three? The black dot at the origin is the modified Prototype 09 in the center column of Figure \ref{Proto&prAxis}.  The 49-dimensional image space is rotated around the origin to align the $x$ axis with the principal axis, so that the coordinates at the tip of the red arrow are $(1.75322,0,0)$, where $1.75322$ is the radius of the Coordinate Sphere, as shown in Table \ref{SphereRadius}. For the $y$ coordinate, we rotate the image space again around the tip of the red arrow to align the $y$ axis with the infinitesimal eigenvectors that are used to generate the $\Theta$ coordinate curves, but we do this separately for each curve.  Thus, although the initial directions of the four curves depicted in Figure \ref{3DVisualization} are mutually orthogonal in the 49-dimensional image space, in the visualization they are all drawn in the same plane.  Finally, we need to collapse the remaining 47 dimensions down to the $z$ axis, somehow.  Our solution is simply to sum all the remaining coordinates after applying the $x$ and $y$ rotations to the first two, which means that the initial value of $z$ at the tip of the red arrow will be zero, and the length of the curves in the image space will correspond roughly to the length of the curves in the visualization. In the image space, the coordinate curves are extended to the Euclidean angles $\pi/2$ and $- \pi/2$, and this is true in the visualization as well. Note that the positive direction towards $\pi/2$ corresponds to a clockwise rotation in the visualization, as can be seen by comparing the length of the two segments of the black curve in Figure \ref{3DVisualization}. 

The green curves in Figure \ref{3DVisualization} also illustrate some of the main properties of our geometric model.  These are the $\rho$ coordinate curves drawn inwards from the points on the $\Theta$ coordinate curve for the maximal infinitesimal eigenvector at the Euclidean angles $\pi/2$ and $- \pi/2$. The Euclidean distance along the curve from $\pi/2$ (on the left) is 1.50939 and the Riemannian distance is 0.947386.  The Euclidean distance along the curve from $ - \pi/2$ (on the right) is 2.1958 and the Riemannian distance is 0.922917.  These values should be compared to the values in Table \ref{rhoDistances} for the corresponding distances along the principal axis for Prototype 09.  The Euclidean distances in Figure \ref{3DVisualization} are less, but the Riemannian distances are approximately the same (within the tolerance of our numerical calculations). This is a reflection of the fact that the geodesics on the Frobenius integral manifold are a constant Riemannian distance from the origin.  See Theorem \ref{ConstantRho} and Theorem \ref{ConstantRiemannian} in Section \ref{GeomM}, \emph{supra}.

So far, we have only been looking at the \emph{geometry} of the prototypical clusters in the MNIST dataset.  But we now want to study the distribution of data points on the curves and surfaces of this geometry.  The first step is to compute the $\rho$, $\Theta$, coordinates for an arbitrary data point in the original 49-dimensional Euclidean space.  It is straightforward to compute the $\rho$ coordinate, and project the data point onto the Frobenius integral manifold.  For example, for Prototype 09, as shown in Column (d) in Table \ref{rhoDistances}, the Riemannian distance along the principal axis to the origin is $0.90783$, so for a point closer than this we project outwards along the $\rho$ coordinate to make up the difference, and for a point further away we project inwards.  Then, to locate the projected point on the surface of the manifold, we need to follow the coordinate \emph{flows}, given by Equation \eqref{TransverseCurves}, in a fixed order.  We can formulate this task as an optimization problem: Find a sequence of parameters,  $t_{1}, t_{2}, t_{3}, \ldots$, for the flows $\vec{\theta}^{\,1}_{t_{1}}({\bf x}_{0}), \vec{\theta}^{\,2}_{t_{2}}({\bf x}_{1}), \vec{\theta}^{\,3}_{t_{3}}({\bf x}_{2}), \ldots$, that takes us from the principal axis at ${\bf x}_{0}$ to a location as close as possible to the projected point on the manifold. (For an illustration of how this works in the three-dimensional case, see Figure 15 in Section 5.2 of \citep{CCCS_AMAI}.) However, if we tried to use the optimization code in \emph{Mathematica}, say, to solve this problem directly, we would never get an answer. The complexity is too high. The parameters, $t_{1}, t_{2}, t_{3}, \ldots$, are mutually dependent, and each step in the calculation requires a numerical solution of the integral equations in \eqref{TransverseCurves}.  Our approach, instead, is to first compute an \emph{ellipsoidal approximation} to the manifold, based on the simple Gaussian example in \citep{CCCS_AMAI}.  The analytical expressions for the flows $\vec{\theta}_{s}({\bf x})$ and $\vec{\phi}_{t}({\bf x})$ in Section 5.1 of \citep{CCCS_AMAI} can easily be extended to 49 dimensions, and it turns out that our optimization problem can be solved, on the 48-dimensional ellipsoidal surface, by \texttt{FindMinimum} in \emph{Mathematica}. Once we have a solution for the ellipsoidal approximation, we can use it to guide the search for a solution to the integral equations in \eqref{TransverseCurves} on the Frobenius integral manifold itself.\footnote{\label{fnEllipsoid}For a detailed explanation and justification of this ellipsoidal approximation, and a discussion of how an analytical solution on the ellipsoidal surface can be converted into a sequence of coordinate flows on the manifold, see Section 7, \textit{infra}, on the CIFAR-10 dataset.}  Even so, this last step is slow:  In a 49-dimensional example, it takes approximately 20 times more CPU time than the initial calculation on the surface of the ellipsoid. 

Let's now generate some additional MNIST data points. Recall that we used a single 800 point sample to solve the Euler-Lagrange equations. For our analysis of the distribution of the data, we will use a second independent 800 point sample, while excluding those points that lie in the opposite hemisphere from the principal axis. This leaves us with 696 sample points, which are shown in Figure \ref{EigMaxProjections}.  The black curve in this figure is the two-dimensional projection of the black $\Theta$ coordinate curve for the maximal infinitesimal eigenvector that is shown in Figure \ref{3DVisualization}.  The large blue dots delineate the ellipsoidal approximation for this coordinate.  The green dots represent the $(\rho, \theta^{1})$ coordinates of the 696 sample points.  These points are projected, first, onto the ellipsoidal approximation, and then onto the $\theta^{1}$ coordinate curve on the manifold, as illustrated by the small blue dots. Notice that the small blue dots are paired.  For an example to illustrate the pairing, the large black dot at $( 2.0066, 0.47328 )$ is paired with the large blue dot at $( 1.64069, 0.532819 )$.  

\begin{figure}[tbp]
\begin{center}
\includegraphics[width=4.5in]{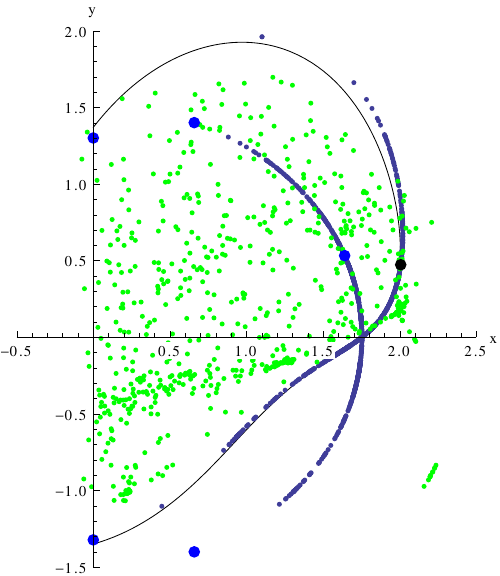}
\caption{Projecting Prototype 09 data points (in green) onto the $\Theta$ coordinate curve for the maximal infinitesimal eigenvector and its ellipsoidal approximation.}
\label{EigMaxProjections}
\end{center}
\end{figure}
 
We can now compute the ``variance'' of this data, as we did for the three-dimensional curvilinear Gaussian.  The RMS dispersion is $ 0.572678 $ along the ellipsoidal approximation and $ 0.570973 $ along the $\theta^{1}$ coordinate curve.  For a comparison, we can do the same projections and calculations for the three minimal infinitesimal eigenvectors and their curvilinear coordinate curves listed in Table \ref{ndDistancesTop3}.  The RMS dispersion along the ellipsoidal approximation is $ 0.144643 $ for Curve 16, $ 0.186668 $ for Curve 17, and $ 0.14785 $ for Curve 09.  Note that the ordering of the ``variance''  of this data for these three coordinate curves is the same as the ordering of their total Riemannian/Euclidean lengths in Table \ref{ndDistancesTop3}.  Furthermore, all three curves are in the top ten of the minimal infinitesimal eigenvectors, when ranked by their Riemannian/Euclidean distances from the principal axis to $ \pi / 2 $ and $ - \pi / 2 $.  These rankings therefore support the conclusion that we should be using the maximal infinitesimal eigenvector for the first coordinate curve, $\theta^{1}$, in our $( \rho , \Theta ) $ coordinate system. 

The next question is:  How to reorder the remaining 47 coordinates?  Let's denote these coordinate curves by $\{ \phi^{j} \mid j = 1, 47 \}$.  For specificity, let's say that we want to choose ten of the $\phi^{j}$ coordinates for a lower-dimensional approximation of the data, giving us overall a 12-dimensional subspace with coordinates $( \rho , \theta^{1}, \phi^{1}, \ldots, \phi^{10}) $.  In the three-dimensional case, we saw that the rankings of $( \rho , \theta^{1}, \phi ) $ obtained by maximizing the ``variance'' were inconsistent with the rankings obtained by minimizing the ``reconstruction error'' but both measures became consistent when we extended the model to a 12-dimensional space using artificial data.  See Section 5.1 of \citep{CCCS_AMAI}.  We can now try a similar calculation using real data.  Table \ref{tenOptimalThetaCurves} lists the ten $\Theta$ coordinate curves for which the total Riemannian/Euclidean distance from the principal axis to $ \pi / 2 $ and $ - \pi / 2 $ is minimal, the exact opposite of Table \ref{ndDistancesTop3}.  Compute the RMS dispersion of the 696 data points along the ellipsoidal approximations for these ten coordinate curves, sum them, and the total is $1.19073$.  Compute the same values for the remaining $\{ \phi^{j} \mid j = 11, 47 \}$, sum them, and the total is $6.55237$.  Thus the result for Table \ref{tenOptimalThetaCurves} is a low ``variance'' and a high ``reconstruction error.''  The ratio is $0.181724$.  (For comparison, the ratio $10/37 = 0.27027$.)  We take this to be a \emph{low baseline}. 

\begin{table}[htbp]
\caption{Prototype 09: The ten $\Theta$ coordinate curves with the least total Riemannian and Euclidean distances from the principal axis to $ \pi / 2 $ and $ - \pi / 2 $. }
\begin{center}
\begin{tabular}{r|c|c|c|}
	& Positive   & Negative    & Total  \\ 
	&  \hspace{1.0em} Direction \hspace{1.0em} &  \hspace{1.0em} Direction \hspace{1.0em} & \hspace{1.0em} Distance \hspace{1.0em} \\ \hline
Curve 42 &  2.80639  &  2.77794 & 5.58433 \\ \hline
Curve 23  & 2.77045  &  2.81585 & 5.58631 \\ \hline
Curve 30  & 2.79071   &  2.79673 & 5.58744 \\ \hline
Curve 37  & 2.77144   &  2.81653 & 5.58799 \\ \hline
Curve 29  & 2.8927     &  2.69908 & 5.59179 \\ \hline
Curve 44  & 2.83735   &  2.75788 & 5.59523 \\ \hline
Curve 11  & 2.70893   &  2.88847 & 5.59741 \\ \hline
Curve 36  & 2.80777   &  2.79319 & 5.60096 \\ \hline
Curve 26  & 2.71708   &  2.88505 & 5.60213 \\ \hline
Curve 01  & 2.69536   &  2.90715 & 5.60251 \\ \hline
\end{tabular}
\end{center}
\label{tenOptimalThetaCurves}
\end{table}
  
For a \emph{high baseline}, let's try reordering the coordinate curves separately for each data point.  Computing the RMS dispersion for this strategy does not make sense, but we can choose each step to maximize the distance from the starting point and minimize the remaining distance to the goal.  We can then compute the mean distance along the first ten coordinates, for all 696 data points, and compare this to the mean distance along the rest of the path, which is a measure of the mean ``reconstruction error.''  The results are $2.71968$ and $3.15808$, respectively, and the ratio is $0.86118$.  An alternative calculation would add the RMS dispersion of the $\theta^{1}$ coordinate curve ( $=  0.572678 $ ) to the values of the first ten $\phi^{j}$ coordinates in each baseline, on the theory that this is also a maximal step towards the goal.  The ratios would then be $ 0.269125 $ for the low baseline and $ 1.04252 $ for the high baseline. 

We see here an interesting empirical phenomenon:  The total Riemannian/Euclidean distance along the 47 $\Theta$ coordinate curves, for each sample point, tends to be concentrated on a small number of coordinates.  If the ratio for the baseline is 1.0, the distance along the $\theta^{1}$ coordinate curve plus the first ten $\phi^{j}$ coordinate curves is the same as the distance along the remaining 37 coordinate curves.  Furthermore, the top ten $\phi^{j}$ coordinates tend to fall into \emph{clusters}.  We will return to this point in Section \ref{AppCIFAR10}, but for now we will simply illustrate the phenomenon using a conventional clustering algorithm, \texttt{FindClusters} in \emph{Mathematica}.  It turns out that we can partition the top ten $\phi^{j}$ coordinates associated with the 669 sample points into 18 clusters, most of which are dominated by a single element, which is therefore the \emph{mode} of the cluster distribution.

\begin{figure}[tb]
\begin{center}
\includegraphics[width=4.5in]{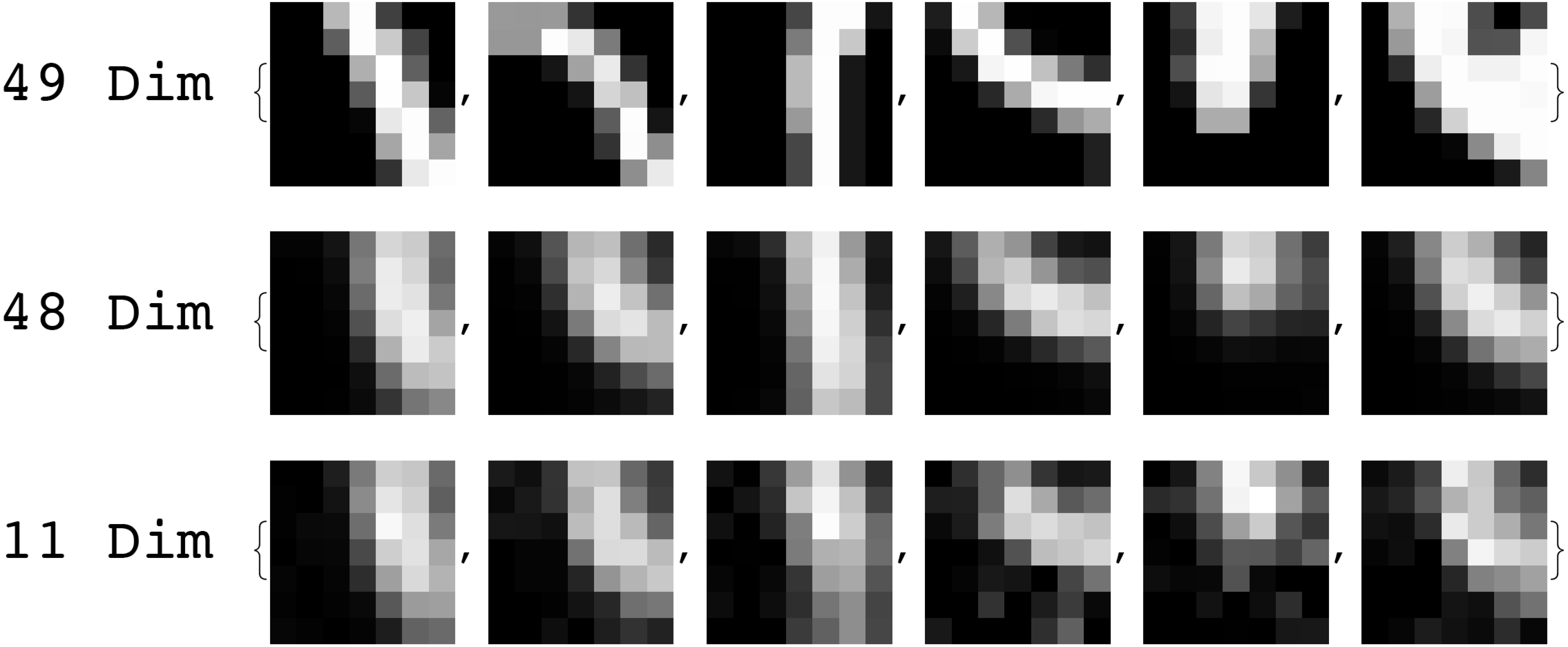}
\caption{Top Row: Sample points from the modes of six distinct clusters of top ten coordinates.  Middle Row: Projections of the sample points onto the Frobenius integral manifold.  Bottom Row: The 11-dimensional approximation that results from following the optimal coordinate curves on the surface of an ellipsoid.}
\label{ClusterModes}
\end{center}
\end{figure}

\begin{figure}[htbp]
\begin{center}
\includegraphics[width=4.5in]{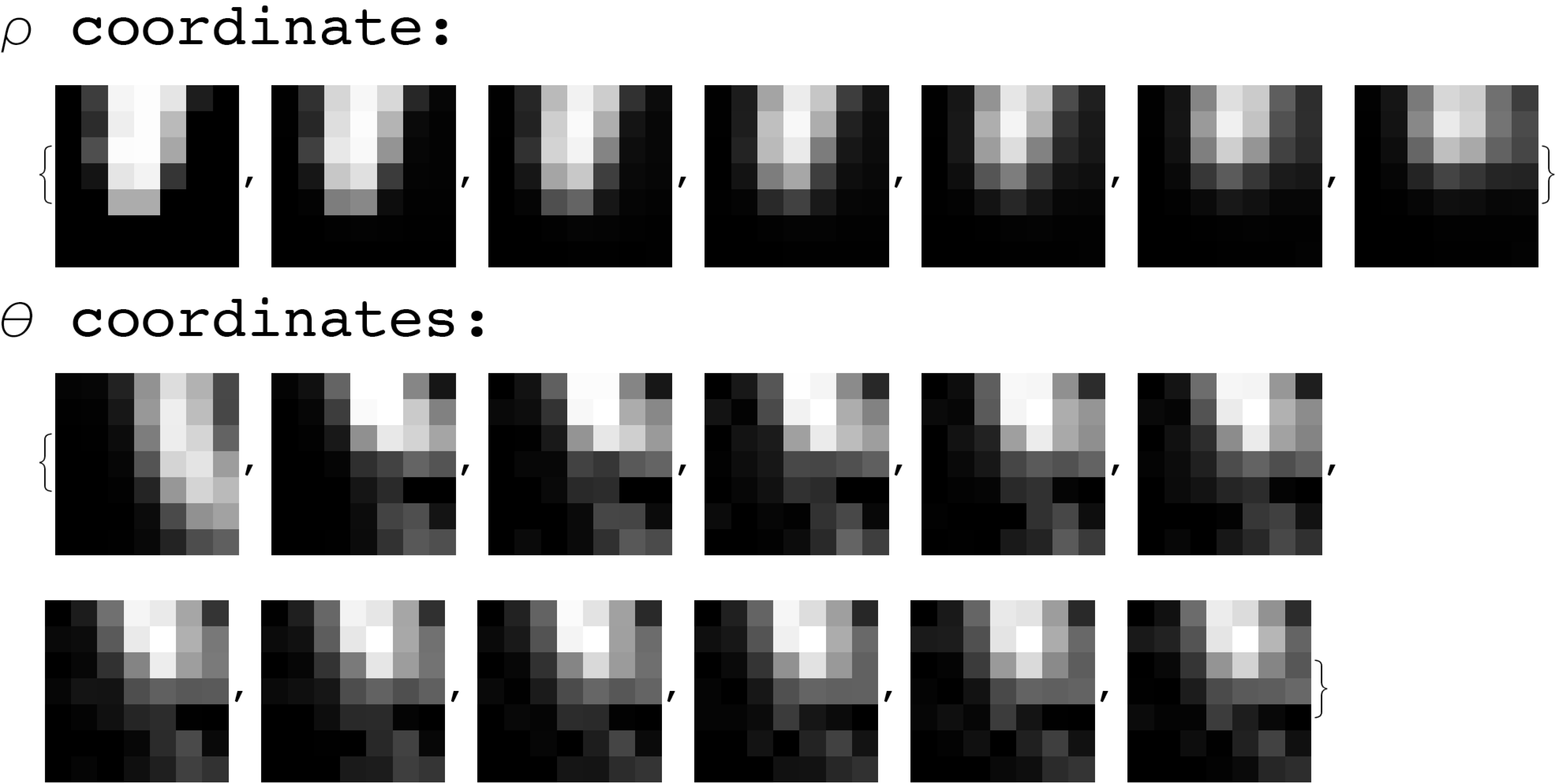}
\caption{The $\rho$ coordinate curve and the $\Theta$ coordinate curves for the mappings in the fifth column of Figure \ref{ClusterModes}.}
\label{RhoThetaCoords}
\end{center}
\end{figure}
 
Figure \ref{ClusterModes} is an illustration.  The top row displays six sample points whose top ten coordinates are included in the modes of six distinct clusters.  The second row shows the projection of each sample point onto the Frobenius integral manifold.  The first five columns represent the clusters with the largest modes, i.e., the clusters with the most number of repeated elements, in descending order.  Notice that the projections of these five sample points onto the Frobenius integral manifold tend to form coherent and distinct geometric shapes.  The last column represents the cluster to which the ten coordinates in Table \ref{tenOptimalThetaCurves} would be assigned. This is a distinct cluster in the output of \texttt{FindClusters}, but it overlaps with the cluster in the second column, and the projections of the second and the last sample points onto the Frobenius integral manifold show a clear resemblance.  The third row in Figure \ref{ClusterModes} displays the results of following the $\Theta$ coordinate curves from the principal axis for Prototype 09 towards the projected sample points, first, following the coordinate flow for the maximal infinitesimal eigenvector (see Table \ref{ThetaDistances} and Figure \ref{ndCurvesMaxEig}), and then following the top ten coordinate flows for the minimal infinitesimal eigenvectors.  (Actually, we are displaying in the third row of Figure \ref{ClusterModes} the corresponding points on the ellipsoidal approximation.)  We are thus looking at an 11-dimensional approximation of the 48-dimensional points from the Frobenius integral manifold.  

For additional insight into how this coordinate system works, Figure \ref{RhoThetaCoords} shows an expansion of the fifth column of Figure \ref{ClusterModes}. The top row traces the path of the $\rho$ coordinate curve from the original 49-dimensional sample point to a point on the 48-dimensional Frobenius integral manifold.  The bottom rows follow the $\Theta$ coordinate curves from the principal axis to a point on an 11-dimensional submanifold.  In this example, the initial step in the direction of the maximal infinitesimal eigenvector is the largest, but this is not always the case.  Note that the last step in Figure \ref{RhoThetaCoords} is slightly different from the last step in Figure \ref{ClusterModes}.  This is because the coordinate curves in Figure \ref{RhoThetaCoords} were computed on the Frobenius integral manifold itself, whereas in Figure \ref{ClusterModes} they were computed on the ellipsoidal approximation.  In either case, though, if we continued to follow the $\Theta$ coordinate curves for an additional 37 steps, we would arrive at a point that is almost identical to the target point in the second row of Figure \ref{ClusterModes} or the first row of Figure \ref{RhoThetaCoords}. In this sense, we are justified in saying that we have found an optimal encoding for the sample points in Figure \ref{ClusterModes} on a low-dimensional nonlinear subspace of the original high-dimensional Euclidean space.

\begin{figure}[tb]
\begin{center}
\includegraphics[width=4.5in]{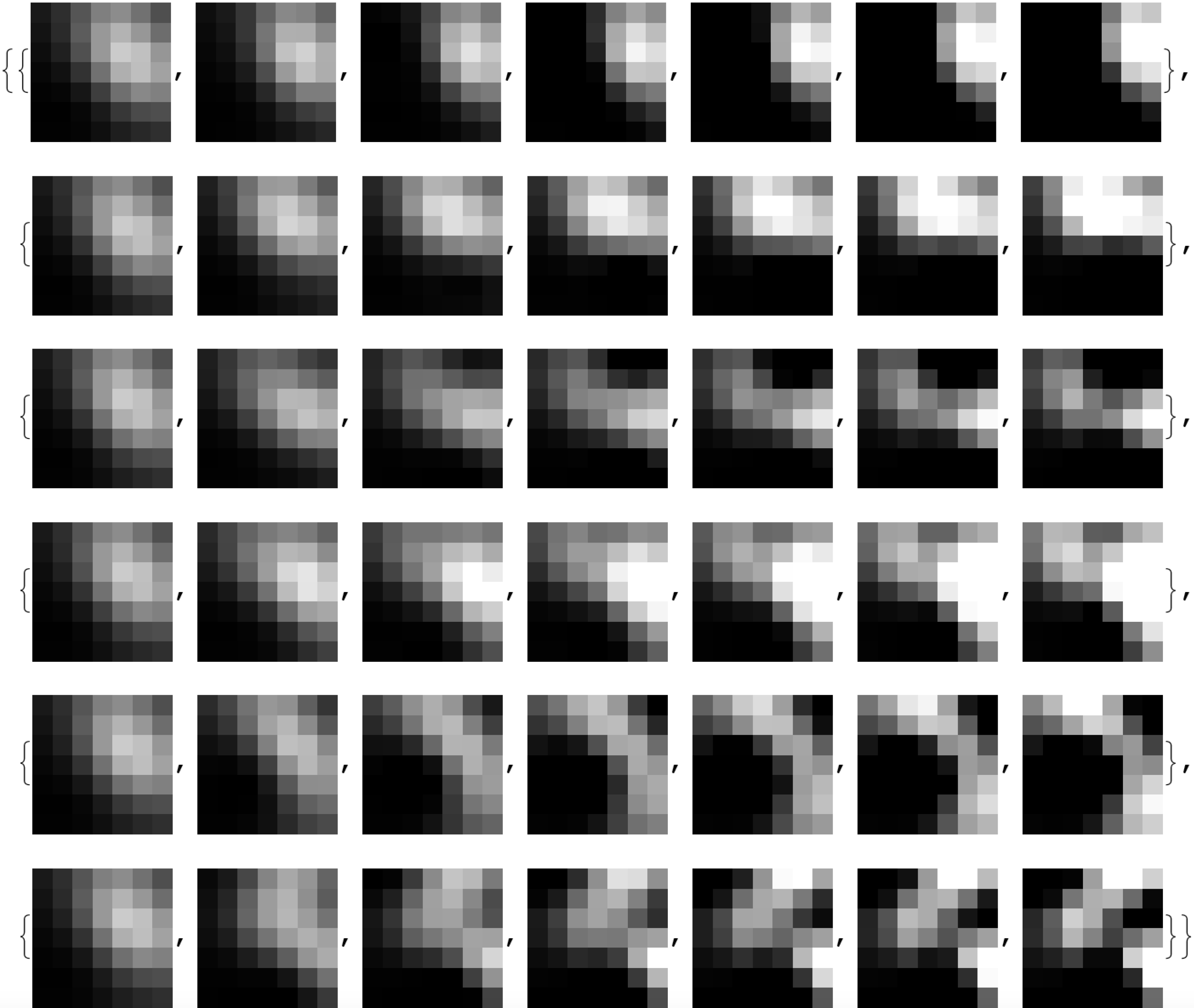}
\caption{Principal eigenvectors for the covariance matrix computed on the 32,000 point Data Sphere for Prototype 09.}
\label{PCATop6}
\end{center}
\end{figure}
  
It may be instructive to compare Figures \ref{rhoCurves3Proto}--\ref{ndCurvesTop3}  and \ref{ClusterModes}--\ref{RhoThetaCoords} to the results that would be obtained from Principal Components Analysis \citep{Pearson1901}.  Figure \ref{PCATop6} shows the first six principal eigenvectors computed on the covariance matrix for the 32,000 point Data Sphere around Prototype 09. The first two eigenvectors show some similarities to the curves for Prototype 09 in Figures \ref{rhoCurves3Proto} and \ref{ndCurvesMaxEig}.  But the subsequent eigenvectors are very different, and they do not seem to be related to any coherent global geometric shapes in the MNIST dataset.  


\section{Example: 7$\times$7 RGB Patches in the CIFAR-10 Dataset.}
\label{AppCIFAR10}

The second example that we will study is the CIFAR-10 dataset \citep{Krizhevsky:2009}, a classification task which is generally considered to be much harder than MNIST.  The dataset was constructed from small, low-resolution, full-color photographs, in ten categories: airplane, automobile, bird, cat, deer, dog, frog, horse, ship, truck.  Each image is represented as a 32$\times$32 array of pixels in three color channels: \textsf{RED}, \textsf{GREEN}, \textsf{BLUE}.  There are 50,000 images in the training set and 10,000 images in the test set.  We performed global contrast normalization on the images in the training set, and then scaled the normalized image intensity down to a real number in the range [0, 1], but we did no other preprocessing, such as whitening \citep{Krizhevsky:2009}. To facilitate a straightforward comparison to our results on the MNIST dataset, we defined a two-pixel border around each image in the training set and sampled the set of 7$\times$7 patches within the 28$\times$28 interior, drawing 12 samples per image. We thus started out with a sample that is comparable to the sample in Figure \ref{ArchitectureDL}: 600,000 7$\times$7 patches.  Our plan is to proceed along the same path as in Figure \ref{ArchitectureDL}, but only as far as the lower-right corner, as we did with MNIST, to produce a 12-dimensional encoding of the patches.  The main difference is that we are now working with three color channels instead of one monochrome channel. 

We will also introduce two new constructs in this Section that were not part of our analysis of the MNIST dataset in Section \ref{AppMNIST}:  \textit{quotient manifolds} and \textit{product manifolds}.

Quotient manifolds are used to build invariance into the geometric model, consistent with known properties of the probabilistic model. For example, in our discussion of Figures \ref{Proto&prAxis} and \ref{rhoCurves3Proto} in Section \ref{AppMNIST}, we noted that the mappings of Prototype 09 could be described by a shift transformation applied to a simple geometric shape, one or two pixels up or down, or left, or right.  This makes sense, probabilistically.  If a simple geometric shape is located at the mode of a probability distribution, then the horizontal and vertical translations of this shape will have approximately the same probability, and will also be located at a mode. But we know this, \textit{a priori}, from the properties of images.  Thus we do not need to compute the transformation from the data, as we did in Figures \ref{Proto&prAxis} and \ref{rhoCurves3Proto}.  We can build it into the geometric model from the start. 

The classical technique to impose translation invariance on an image classification task is to use a Convolutional Neural Network 
\citep{journals/neco/LeCunBDHHHJ89, LeCun_etal_1998, Zeiler14visualizingand}, 
since the convolution operation is \textit{equivariant} with respect to translation. However, there are other possible techniques, based on more general transformation groups 
\citep{ANSELMI2016112, pmlr-v48-cohenc16}.
Let $G$ be a group, and let $\{ T_{g} \;|\; g \in G \}$ be a set of transformations on the space of images, parameterized by $g$.  Suppose we hypothesize, for a particular choice of $G$, a particular choice of $T_{g}$, and a particular representation of the space of images, that the probability of $T_{g}(\textbf{x})$ is identical for all $g \in G$.  We can then form a \textit{quotient space} over the space of images, and work with a \textit{representative} of this quotient in our probabilistic calculations.  For the space of 7$\times$7 \textsf{RGB} patches in the CIFAR-10 dataset, the most useful transformations are:  vertical translations, horizontal translations, and horizontal flips.

\begin{figure}[tb]
\begin{center}
\includegraphics[width=5.0in]{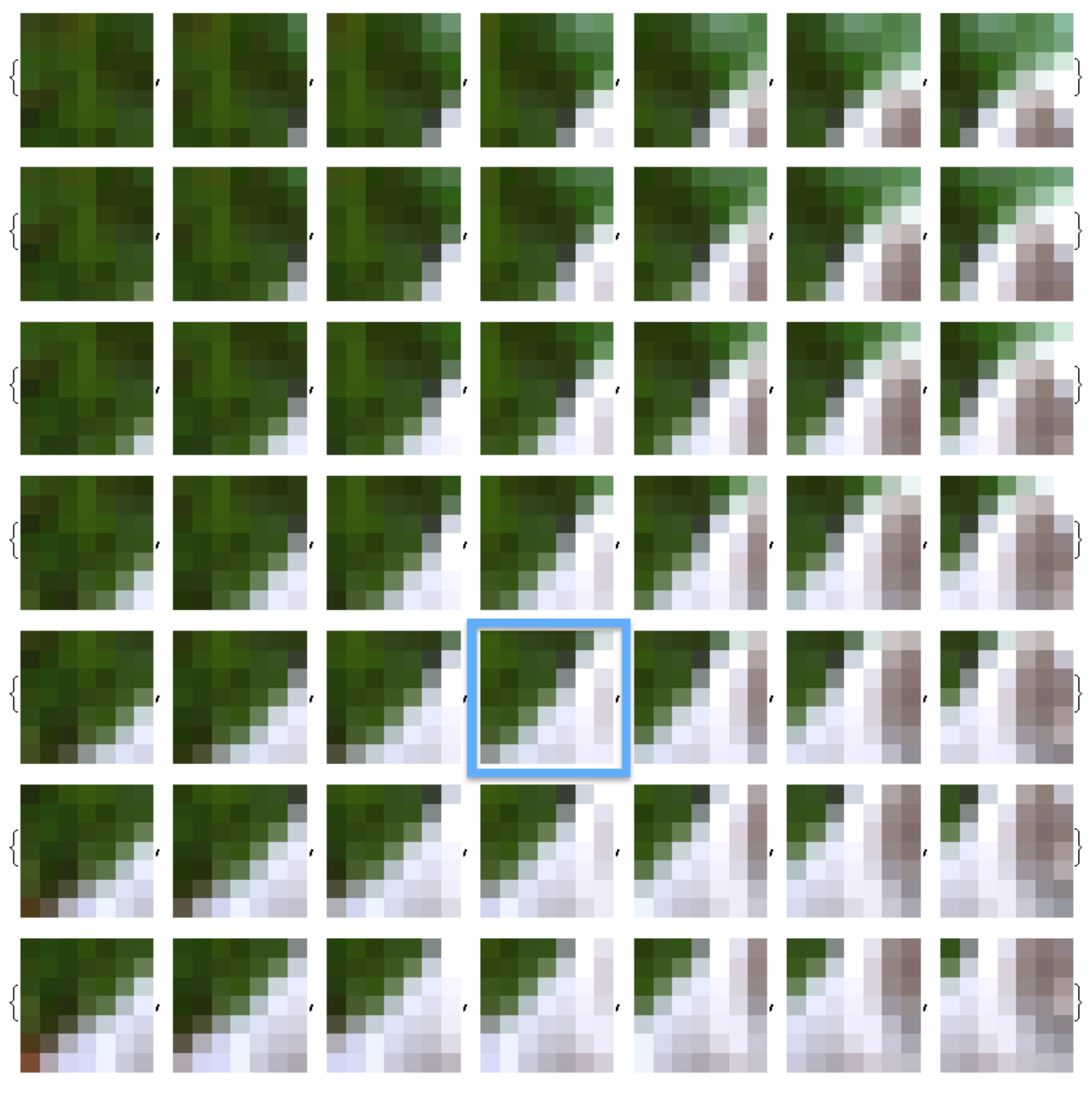}
\caption{ A 7$\times$7 \textsf{RGB} patch shifted three pixels up and down, and three pixels left and right.}
\label{ShiftedPatches}
\end{center}
 \end{figure}

Figure \ref{ShiftedPatches} illustrates how this works.  These are patches from a photograph of a white and gray cat sitting in front of a green vegetative background.  Assume that our sampling procedure has selected the patch in the center of this grid, and we are looking at all the patches shifted zero-to-three pixels up and down, and zero-to-three pixels left and right. (At the edges of the 28$\times$28 images, we are using the discarded two pixel border and an additional one pixel border set to the mean value of the image intensities in each color channel.)  We arrange the grid on a torus: The vertical shifts lie along the ``vertical'' axis of the torus, and the horizontal shifts are extended by a sequence of horizontal flips and lie along the ``horizontal'' axis of the torus. By our probabilistic hypothesis, every patch on this torus has equal probability, and thus the ``toroidal grid'' is the quotient space that we are looking for. How do we choose a representative of this quotient? What we want is an unambiguous procedure based on some salient property of the image.  For the translations, we choose the patch that has the maximum variance in the three color channels, which is the patch outlined in blue in Figure \ref{ShiftedPatches}.  For the horizontal flips, we choose the orientation that maximizes the mean intensity on the left-hand side, which means that the patch outlined in blue in Figure \ref{ShiftedPatches} needs to be flipped horizontally.  Note that our sampling procedure could initially select two distinct patches from a 28$\times$28 image and transform both of them into the same patch, but that is precisely what is required by our probabilistic model.  

Product manifolds are used to combine low-dimensional solutions into a higher dimensional problem space, so that our dimensionality reduction techniques can be applied recursively.  One example is shown in Figure \ref{ArchitectureDL} for the MNIST dataset with a single monochrome channel.  Looking at the lower-right corner, assume that we have mapped the 7$\times$7 patches into a set of 12-dimensional submanifolds.  Following the blue arrow, we assemble four adjacent 12-dimensional submanifolds into a product manifold and sample the images again using the 14$\times$14 patch.  This means that each 14$\times$14 sample is mapped into four nonlinear submanifolds to produce a point in a 48-dimensional probability space.  And, as indicated by the arrows pointing to the left in the middle of Figure \ref{ArchitectureDL}, the task now is to apply the theory of differential similarity again to reduce the dimensionality of this space down to 12. Product manifolds --- in particular, products of Grassmann manifolds --- have been used successfully in recent research on the recognition of human actions.  See, e.g., \citep{lui2012human}. They do not appear to have played a prominent role so far in work on the classification of static images. 

We will not study the 14$\times$14 sample space in the present paper (but see Section \ref{FutureWork} on ``Future Work'').  Instead, we will study a simpler example based on the \textsf{RED/GREEN/BLUE} color channels in the CIFAR-10 dataset.  The dimensionality of a 7$\times$7 \textsf{RGB} patch is $3 \times 49 = 147$, and we could either take this space as our starting point, or we could start with three separate 49-dimensional spaces.  Let's take the latter approach.  We will apply the theory of differential similarity, separately, in each color channel, to produce three mappings into a set of 12-dimensional submanifolds, and then combine these into a set of 36-dimensional product manifolds.  Several questions arise as we take the next step:  Assume that we want to construct a lower-dimensional submanifold embedded in one of these product manifolds.  How should we do this?  If there are several distinct 12-dimensional submanifolds in each color channel, is it necessary to consider all possible combinations when we define the product space?  How should we sample this space?  We will suggest answers to these questions in the present paper, with the expectation that this will give us some insights into the analogous questions for the 14$\times$14 sample space. 

\begin{figure}[tb]
\begin{center}
\includegraphics[width=5.5in]{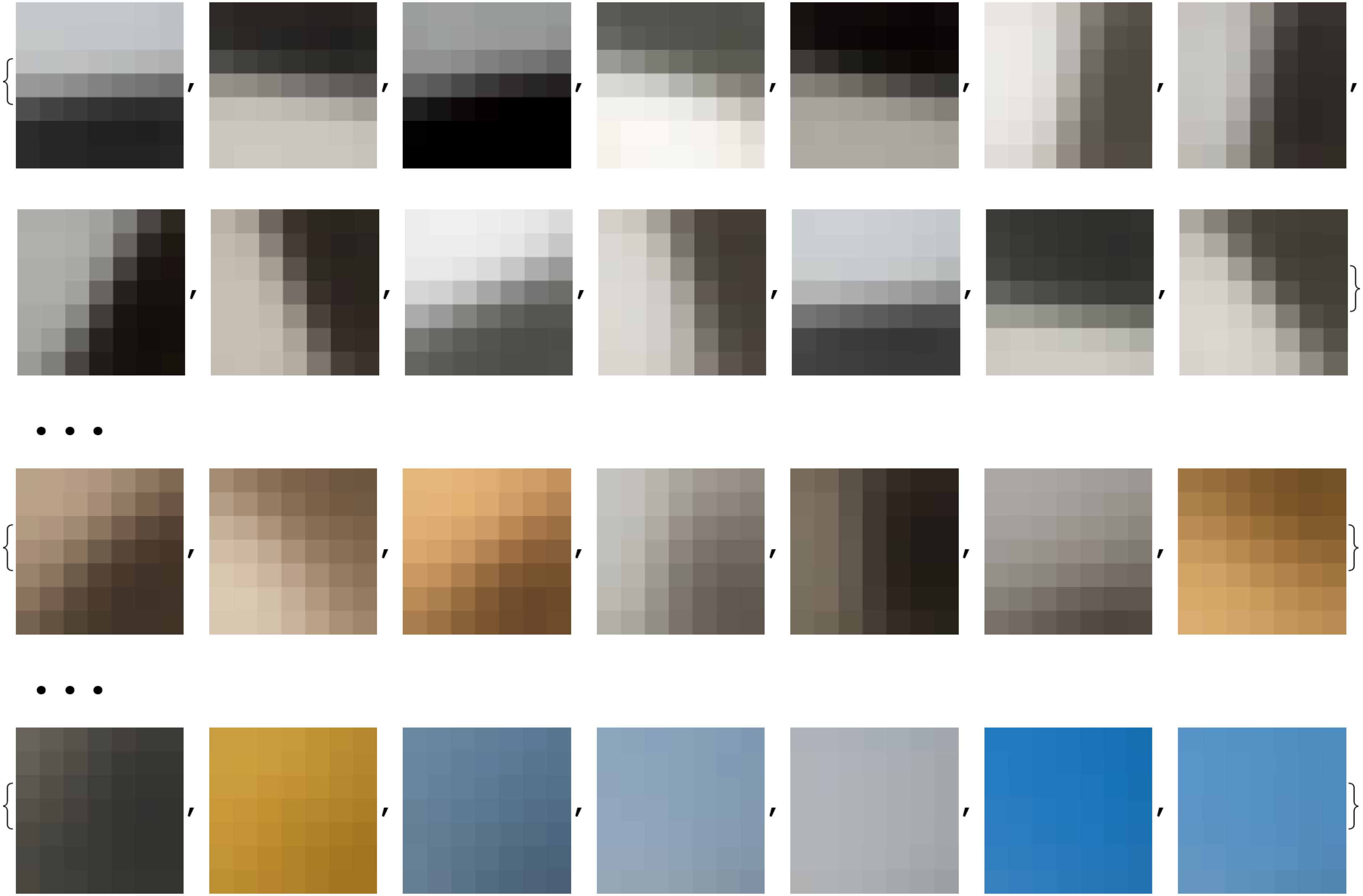}
\caption{A selection from 50 CIFAR-10 prototypes, sorted from high variance to low variance. Pictured here are Prototypes  $01$ to $14$, $32$ to $38$, and $44$ to $50$.}
\label{CIFARprototypes}
\end{center}
\end{figure}

For our initial selection of prototypes, however, we will start out with a set of data points in the full 147-dimensional space.  Recall that our selection of prototypes for the MNIST dataset included a manual post-processing step.  After some experimentation, we have developed a fully automated procedure for the CIFAR-10 dataset, with several adjustable parameters.  There are two stages:
\begin{enumerate}

\item The first stage applies \texttt{pgradascent} from Figure \ref{Kernel} to the full dataset.  We process the set of CIFAR-10 patches in five batches of 120,000 points each, generating 160 random choices of \texttt{xstart} in each batch, with $\beta = 1/128$ (or $\sigma = 8.0$) and with a 200 point \texttt{Sample} retrieved by the \texttt{NearestFunction} at each iteration.  With these parameters, \texttt{pgradascent} converges very quickly to a mode very close to \texttt{xstart}.

\item The second stage computes and follows the integral curve of $\nabla U$, starting at one of the outputs from the first stage, to arrive at a new mode, ${\bf x}$, where $\nabla U({\bf x}) = 0$.  In this computation, we set $\beta = 1$ and we set $\mathtt{SamplePoints}$ to an 800 point random sample drawn from a 4,000 point data sphere around the starting point.

\end{enumerate}
We also filter the output of each stage.  The outputs from the first stage are sorted by the average of the variance in each color channel, then processed from max variance to min variance, and a mode is filtered out if it lies within the radius of a 4,000 point data sphere around one of its sorted predecessors. This step reduces the size of the output in each batch to: 58, 64, 59, 52, 59.  The five batches are then merged, sorted, and filtered again, to produce a list of 94 modes as the input to the second stage.  Similarly, the output from the second stage is filtered to remove any mode that lies within the radius of a 1,000 point data sphere around one of its sorted predecessors. The final result is a sorted list of exactly 50 prototypes. 

\begin{figure}[tbp]
\begin{center}
\includegraphics[width=5.5in]{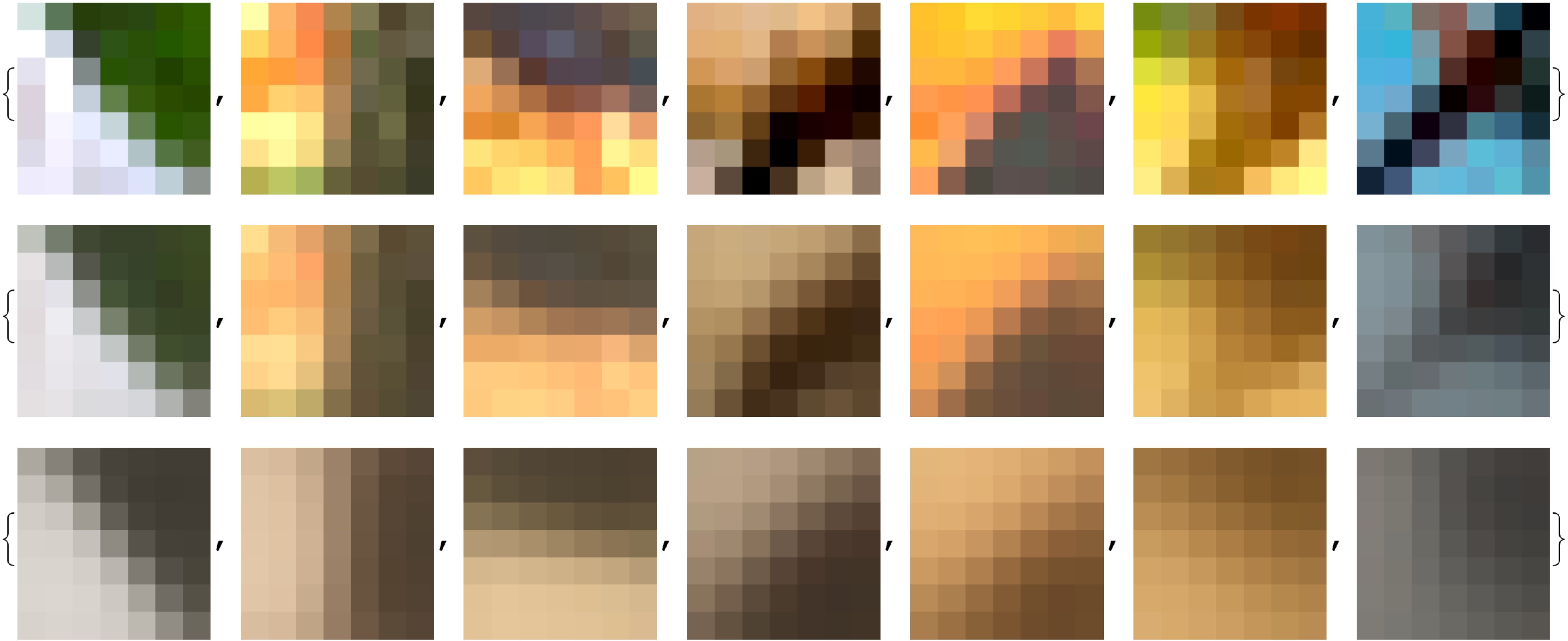}
\caption{Two stages in the computation of Prototypes 14, 20, 23, 32, 34, 38, 40.  Top Row: The initial patch selected as the value of \texttt{xstart}. Middle Row: The output of \texttt{pgradascent}.  Bottom Row: The point, ${\bf x}$, for which $\nabla U({\bf x}) = 0$.}
\label{PatchPgaProto}
\end{center}
\end{figure}

Figure \ref{CIFARprototypes} shows a selection of these 50 prototypes, in three groups: high variance, medium variance, low variance.  Seven prototypes of medium variance are shown in Figure \ref{PatchPgaProto}, along with the two stages of their computation.  Note that the value of \texttt{xstart} that led to the derivation of Prototype 14 is the patch outlined in blue in Figure \ref{ShiftedPatches}, flipped horizontally.  We will trace the analysis of these seven prototypes in the subsequent discussion, with particular attention to Prototypes 14 and 32.

\vspace{2ex}
\begin{table}[htbp]
\caption{Radius, in \texttt{3X49D} and \texttt{49D} Euclidean space, of the 8,000 point Data Sphere and the 2,000 point Coordinate Sphere, for seven prototypical clusters.}
\begin{center}
\begin{tabular}{r|c|c|c|c|}
        & \multicolumn{2}{c|}{ \hspace{2.0em} Radius of  \hspace{2.0em} } & \multicolumn{2}{c|}{\hspace{2.0em} Radius of \hspace{2.0em} }  \\ 
        & \multicolumn{2}{c|}{Data Sphere} & \multicolumn{2}{c|}{Coordinate Sphere}  \\ \cline{2-5}
                     & \hspace{1.0em}\texttt{3X49D}\hspace{1.0em} & \hspace{1.5em}\texttt{49D}\hspace{1.5em} & \hspace{1.0em}\texttt{3X49D}\hspace{1.0em} & \hspace{1.5em}\texttt{49D}\hspace{1.5em} \\ \hline
Prototype 14 & 2.11000 & 1.21821 & 1.42727 & 0.824038  \\ \hline
Prototype 20 & 1.91678 & 1.10665 & 1.31318 & 0.758166  \\ \hline
Prototype 23 & 1.97005 & 1.13741 & 1.34494 & 0.776500  \\ \hline
Prototype 32 & 1.78403 & 1.03001 & 1.31377 & 0.758503  \\ \hline
Prototype 34 & 2.06485 & 1.19214 & 1.47952 & 0.854199  \\ \hline
Prototype 38 & 1.91457 & 1.10538 & 1.42930 & 0.825207  \\ \hline
Prototype 40 & 1.76567 & 1.01941 & 1.19463 & 0.689722  \\ \hline
\end{tabular}
\end{center}
\label{CIFARSphereRadius}
\end{table}
\vspace{2ex}

Table \ref{CIFARSphereRadius} shows the size of an 8,000 point Data Sphere and a 2,000 point Coordinate Sphere for our seven prototypical clusters, and explicitly distinguishes between the 147-dimensional \textsf{RGB} space and the three 49-dimensional \textsf{R/G/B} spaces. In each case, we have computed the radius for \texttt{3X49D} and then divided by $\sqrt{3}$ to get the radius for \texttt{49D}, on the assumption that the spheres in each color channel should be the same size. Recall that the sizes of the Data Sphere and the Coordinate Sphere in the MNIST example were 32,000 and 8,000 points, respectively, and here we have chosen much smaller numbers, 8,000 and 2,000 points, respectively.  We will also generally use sample sizes that are one-half the size of the samples in the MNIST example.  This means that we will probably see clusters that are more coherent, and closer to a Gaussian, perhaps, but with more variability in the estimates. This is not to imply that these are better choices, just different choices.  We should think of this analysis as one data point in an exploration of the parameter space for applications of the theory of differential similarity to various image processing tasks.
 
\begin{figure}[tb]
\begin{center}
\includegraphics[width=4.5in]{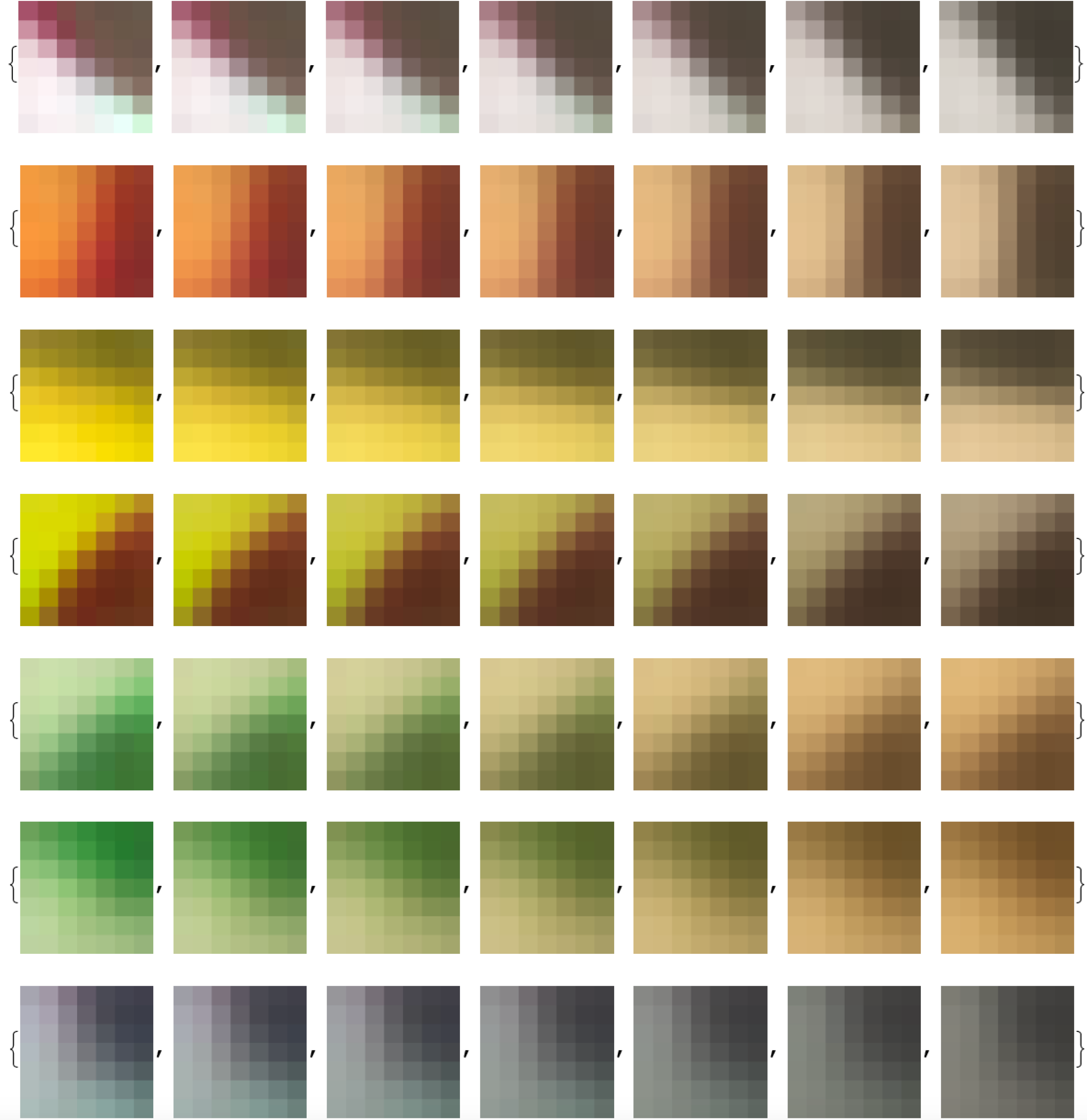}
\caption{The principal $\rho$ coordinate curves for Prototypes 14, 20, 23, 32, 34, 38, 40.  Each row shows a sequence of images along the principal axis, with the prototype displayed in the right column.}
\label{RhoCoordCurves}
\end{center}
\end{figure}

\begin{figure}[htbp]
\begin{center}
\includegraphics[width=4.5in]{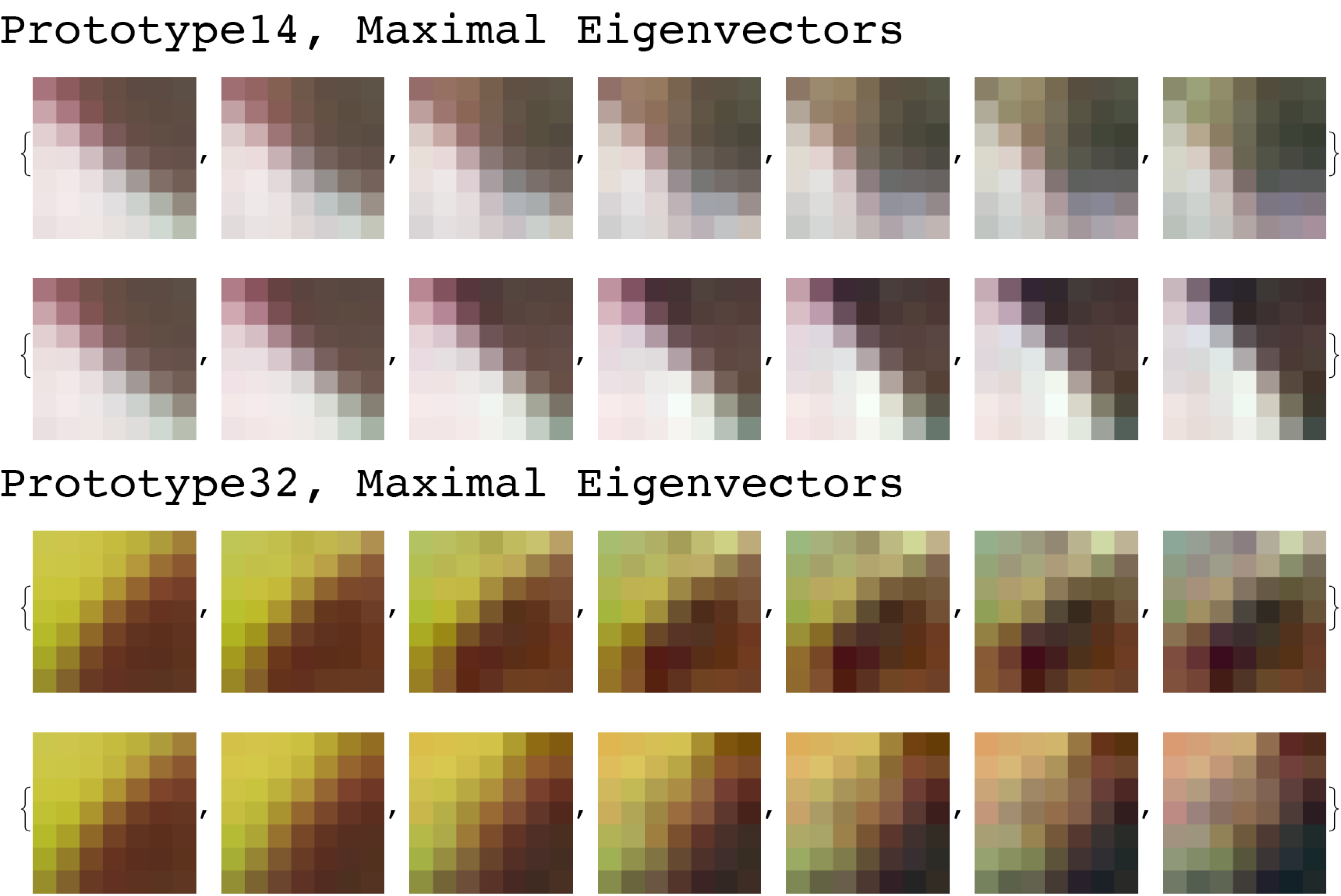}
\caption{For Prototypes 14 and 32, the $\Theta$ coordinate curves for the maximal infinitesimal eigenvectors.}
\label{ThetaCoordCurvesMaxEig}
\end{center}
\end{figure}
 
The next step is to compute the $\rho$ coordinate curves along the principal axis of each prototypical cluster.  From this point on, as explained previously, we will do all computations separately in each color channel, although we will then combine the three \textsf{R/G/B} channels to generate full-color images.  Figure \ref{RhoCoordCurves} is analogous to Figure \ref{rhoCurves3Proto} in Section \ref{AppMNIST}, showing the principal $\rho$ coordinate curves for Prototypes 14, 20, 23, 32, 34, 38, and 40.  (Note: To locate the principal axes in each color channel, we did two global searches using \texttt{NMinimize} in \emph{Mathematica}, with $\beta = 1/128$ and with a 400 point sample defining $\nabla U({\bf x})$ for each search, followed by a local search using \texttt{FindMinimum} in \emph{Mathematica}, which combined the two samples and set $\beta = 1$.) The curves here are actually extended to the left, beyond the 2,000 point Coordinate Sphere where the principal axis is computed, to the 8,000 point Data Sphere.  The color shifts along these curves are obvious, but there are also shifts in the shape of the images, as there were in the MNIST examples.  For instance, for Prototype 14, reading from right to left, we see a counter clockwise rotation of the global geometric shape. 

Figure \ref{ThetaCoordCurvesMaxEig} is analogous to Figure \ref{ndCurvesMaxEig} in Section \ref{AppMNIST}, for Prototypes 14 and 32.  In each case, the first row shows the coordinate curve in the direction of the maximal infinitesimal eigenvector, to the Euclidean angle $ \pi / 2 $, and the second row shows the coordinate curve in the opposite direction, to the Euclidean angle $ - \pi / 2 $.  (Note: To find a solution to the Euler-Lagrange equations in these examples, we set $\beta = 1$ and used a single 400 point sample defining $\nabla U({\bf x})$.) The starting points, the images on the left, are points on the $\rho$ coordinate curves from Figure \ref{RhoCoordCurves} at their intersections with the 2,000 point Coordinate Spheres.  We could draw similar curves, of course, for the minimal infinitesimal eigenvectors,\footnote{\label{fnComplex}For reasons that are not entirely clear, when we compute the eigenvectors of the matrix $\left( \,{g}_{i,j}({\bf x})\, \right)$ in \emph{Mathematica}, the function  \texttt{Eigensystem} sometimes returns a conjugate pair of complex eigenvalues.  We ignore the complex eigenvectors when we compute the coordinate curves. Thus, for Prototype 14, GREEN, we have only 45 coordinate curves for the minimal infinitesimal eigenvectors, and for Prototype 32, RED and BLUE, we have only 43.  This creates some minor problems when we try to measure the reconstruction error.  See footnote \ref{fnOptimal}, \textit{infra}.} analogous to Figure \ref{ndCurvesTop3} in Section \ref{AppMNIST}, but we first need to discuss our strategies for the computation of optimal lower-dimensional subspaces, a task to which we now turn.

\begin{figure}[tbp]
\vspace{4ex}
\begin{center}
\includegraphics[width=5.0in]{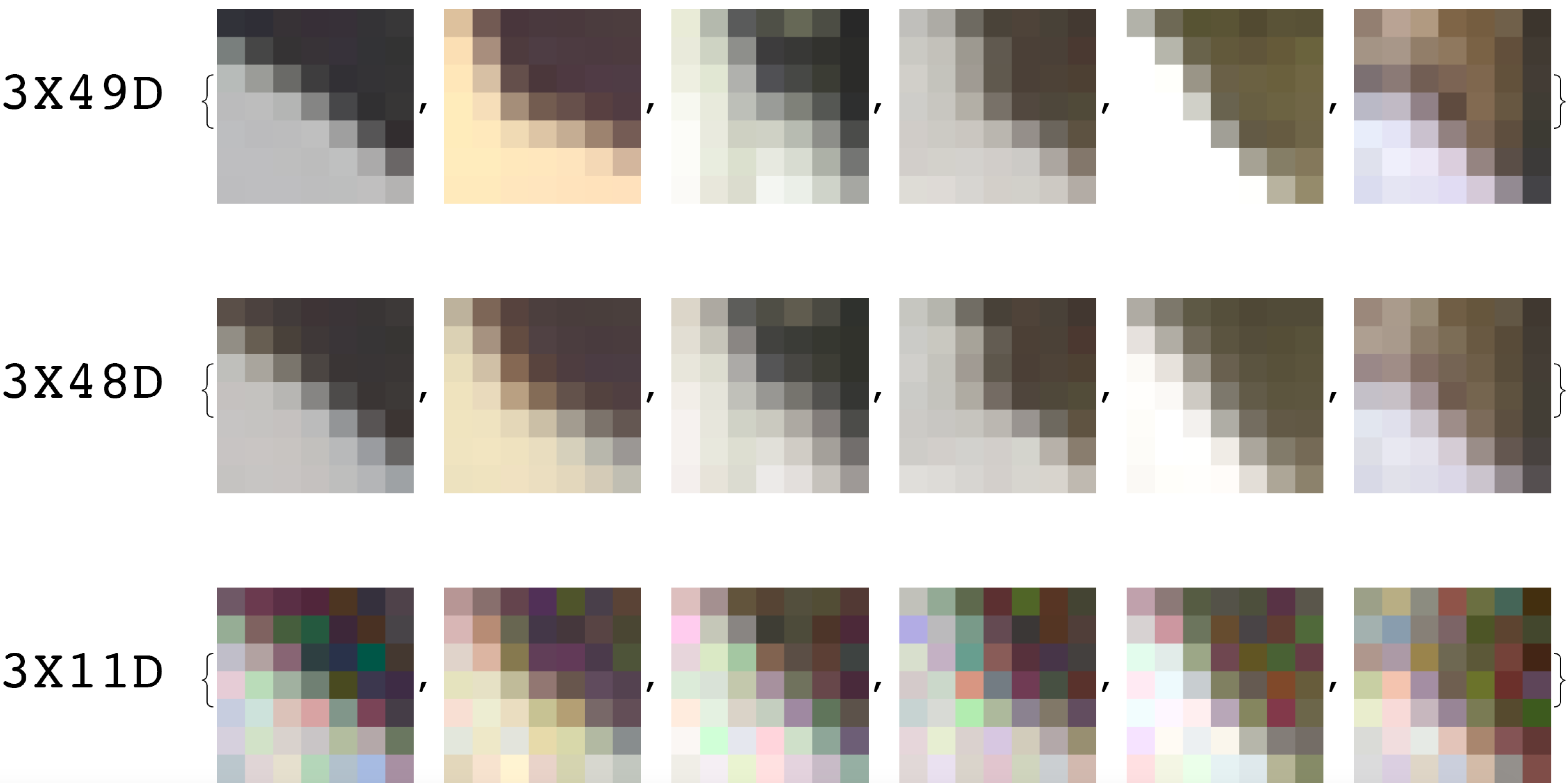}
\caption{Sample points that are the nearest to the $\Theta$-Coordinate Cluster Modes for Prototype 14.}
\label{ClusterModesP14}
\end{center}
\end{figure}

\begin{figure}[tbp]
\vspace{4ex}
\begin{center}
\includegraphics[width=5.0in]{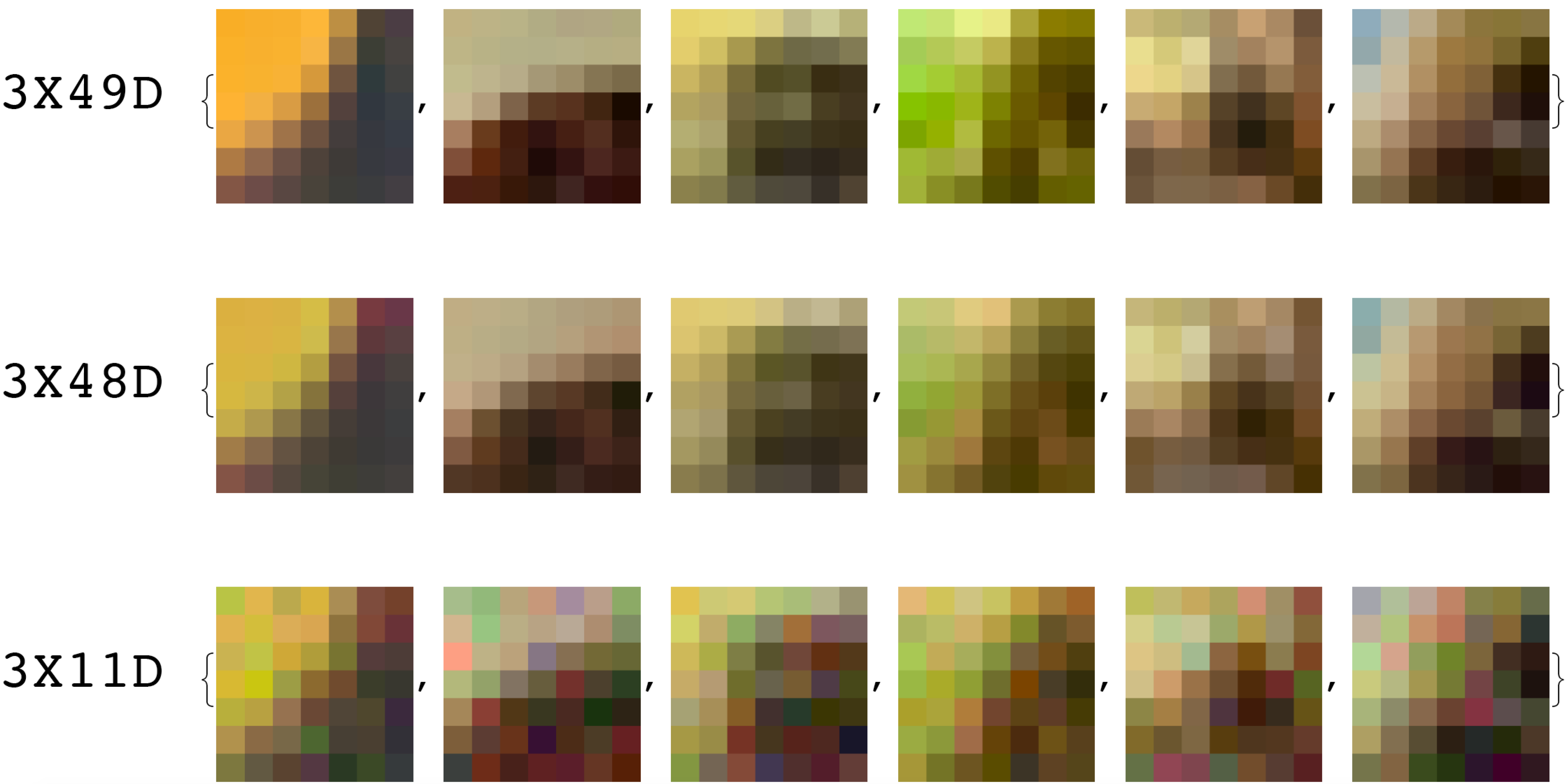}
\caption{Sample points that are the nearest to the $\Theta$-Coordinate Cluster Modes for Prototype 32.}
\label{ClusterModesP32}
\end{center}
\vspace{8ex}
\end{figure}

In our work on optimal lower-dimensional subspaces for the MNIST example, in Section \ref{AppMNIST}, we developed an \textit{ellipsoidal approximation} to the $\Theta$ coordinates on the Frobenius integral manifold, and we defined a measure of the ``variance'' and the  ``reconstruction error'' for this coordinate system in a higher dimensional space. We also generated a new sample (800 points, reduced to the 669 points in the positive hemisphere) to test these concepts.  The ten coordinate curves in Table \ref{tenOptimalThetaCurves} gave us a \emph{low baseline}, in which the ``variance'' was low and the ``reconstruction error'' was high, but we discovered that we could find an optimal set of ten coordinate curves for each individual sample point, a \emph{high baseline}.  Furthermore, these optimal coordinate curves tended to fall into distinct clusters, which were illustrated in Figure \ref{ClusterModes}.  The main challenge in extending these ideas to the CIFAR-10 dataset is to generalize from a single monochrome channel to three \textsf{R/G/B} color channels, i.e., to apply our techniques to the product manifold. 

Figure \ref{ClusterModesP14} for Prototype 14 and Figure \ref{ClusterModesP32} for Prototype 32 are analogous to Figure \ref{ClusterModes} in Section \ref{AppMNIST}, showing the end results of our calculations.  The six columns for each prototype represent six distinct $\Theta$ coordinate clusters. The first row in each figure shows six 147-dimensional sample points that are the nearest to the modes of each cluster, and the second row shows the projection of these points onto three 48-dimensional Frobenius integral manifolds.  The third row shows three 11-dimensional \textsf{R/G/B} approximations, printed as full-color images.  We will explain, step-by-step, how the sample points were selected and how the clusters were constructed, using as a running example a point from the upper left corner of Figure \ref{ClusterModesP14}:  Prototype 14, Cluster 01, Channel RED. 

First, as promised in footnote \ref{fnEllipsoid}, let's look more closely at the ellipsoidal approximation.  We start with an ellipsoid that roughly matches the Frobenius integral manifold along its semi-axes.  For the principal axis, this is just the radius of the Coordinate Sphere, as shown in Table \ref{CIFARSphereRadius}, which is 0.824038 for our running example.  Let $a$ equal the \textit{reciprocal} of the \textit{square} of the length of this axis.  For the coordinate curves aligned with the maximal infinitesimal eigenvectors, we saw in Table \ref{ThetaDistances} and Figure \ref{3DVisualization} for MNIST that the distances along the axes in the positive and the negative directions were substantially different, and the same is true for CIFAR-10.  In our running example, the matching semi-axis is 0.71319 in the positive direction and 0.589725 in the negative direction, so we split the space and use a different ellipsoid in each half. Let $b$ equal the \textit{reciprocal} of the \textit{square} of the length of the axis in each case. The equation for the coordinate flow, $\vec{\theta}^{\,1}_{t}({\bf x})$, along the surface of the ellipsoid can easily be generalized from the equations for the three-dimensional case in Section 5.1 of \citep{CCCS_AMAI}. Let ${\bf x} = (x, y, z_{1}, \ldots, z_{47})$.  Then
\vspace{1ex}
\begin{equation}
\label{intcurve.theta}
\vec{\theta}^{\,1}_{t}({\bf x}) \;=\;
\begin{pmatrix}
& x \; \cos{\sqrt {a b} \; t } \;+\; y \; \sqrt {b/a} \; \sin{\sqrt {a b} \; t} & \\
& y \; \cos{\sqrt {a b} \; t } \;-\; x \; \sqrt {a/b} \; \sin{\sqrt {a b} \; t} \\
& z_{1} \\
& \cdots \\
& z_{47}
\end{pmatrix} \notag
\vspace{1ex}
\end{equation}
For the coordinate curves aligned with the minimal infinitesimal eigenvectors, the matching semi-axes are very close together, in the range 0.610 to 0.625 for our running example, so we simply use the mean of the positive and negative values.  Let $c_{i}$ equal the \textit{reciprocal} of the \textit{square} of the mean length of axis $i$.  Then
\vspace{1ex}
\begin{equation}
\label{intcurve.phi}
\vec{\theta}^{\,1+\,i}_{t}({\bf x}) \;=\;
\begin{pmatrix}
& x \; \cos{\sqrt {a c_{i} } \; t } \;+\; z_{i} \; \sqrt {c_{i} /a} \; \sin{\sqrt {a c_{i} } \; t} & \\
& y \\
& z_{1} \\
& \cdots \\
& z_{i} \; \cos{\sqrt {a c_{i} } \; t } \;-\; x \; \sqrt {a/c_{i} } \; \sin{\sqrt {a c_{i} } \; t} \\
& \cdots \\
& z_{47}
\end{pmatrix}  \notag
\vspace{1ex}
\end{equation}
The most important property of these equations is the fact that they are closed-form \textit{analytical solutions} for the coordinate flows defined in \eqref{TransverseCurves}, and thus the optimization algorithms in \textit{Mathematica} can handle them efficiently.

\begin{figure}[tb]
\begin{center}
\includegraphics[width=4.5in]{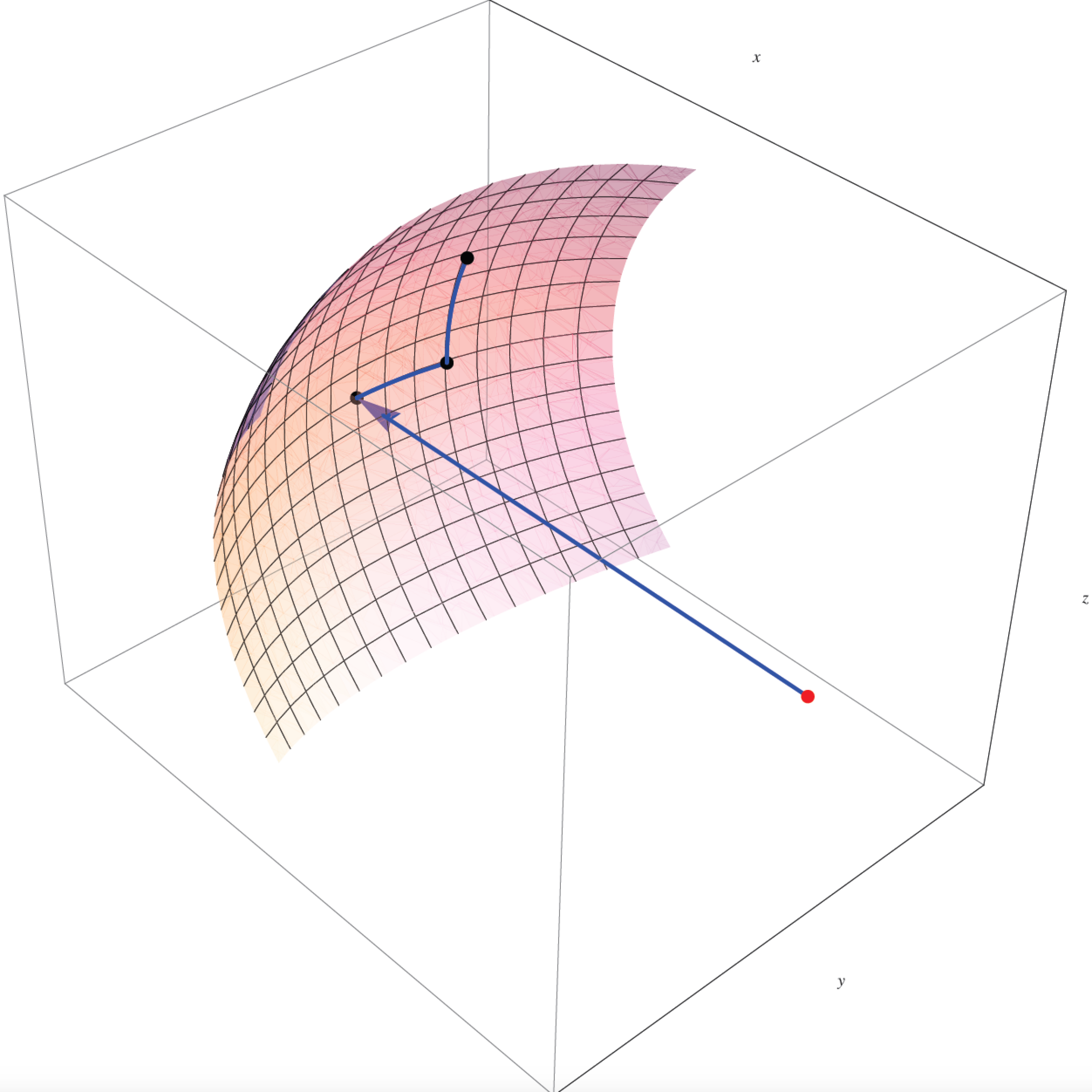}
\caption{The first three optimal coordinates in the ellipsoidal approximation to Prototype 14, Cluster 01, Channel RED, projected against the corresponding two-dimensional surface of the Frobenius Integral Manifold.}
\label{ECoordsOnFM}
\end{center}
\end{figure}

Figure \ref{ECoordsOnFM} shows the optimal coordinate flows for our running example: Prototype 14, Cluster 01, Channel RED. The goal is to approximate the coordinates of the target point on a 48-dimensional Frobenius integral manifold, which is the point in the RED channel of the image in the second row in Figure \ref{ClusterModesP14}. The first three coordinates are:  $\rho$, $\theta^{1}$, $\theta^{2}$.  It is difficult to see this in Figure \ref{ECoordsOnFM}, but the $\theta^{2}$ coordinate curve is tilted slightly inwards, and there is a small gap between the point at the end of this curve and the manifold itself.  As we proceed along the remaining 46 coordinate curves, the gap grows larger, and the coordinates get closer to the target point.  Let's analyze this approximation procedure, quantitatively.  When we first compute the coordinate flows, the best we can do is to use a random ordering of the $\theta$ coordinates for the minimal infinitesimal eigenvectors: $z_{1}, \ldots, z_{47}$.  Once we have a solution, though, we can measure the Riemannian/Euclidean distance along each coordinate curve and identify the 10 coordinates for which this distance is the greatest.  For our running example, these are:
\begin{equation*}
\left(
\begin{array}{cccccccccc}
18 & 45 & 47 & 24 & 11 & 10 & 37 & 34 & 29 & 08
\end{array}
\right)
\end{equation*}
We can then compute the coordinate flows for the minimal infinitesimal eigenvectors again, starting with these 10 coordinates.  The difference is apparent in the following table, where we are displaying the distance in the ambient Euclidean space between specific points on the coordinate curves and the target point on the manifold:
\vspace{1ex}
\begin{center}
\begin{tabular}{r|c|c|c|}
	& 1 Coordinate  & +10 Coordinates   & +37 Coordinates  \\ \hline
Random & 1.033511 &  0.943735 & 0.111027 \\ \hline
Optimal & 1.033556 &  0.608948 & 0.112818 \\ \hline
\end{tabular}
\end{center}
\vspace{1ex}
The distance at the end of coordinate curve $\theta^{1}$ does not change, and the distance at the end of coordinate curve $\theta^{48}$ is almost the same, but there is a big difference at the end of coordinate curve $\theta^{11}$.  Another way to present this result is to compute the total Riemannian/Euclidean distance along the various segments of the coordinate curves, as we did in the MNIST example:
\vspace{1ex}
\begin{center}
\begin{tabular}{r|c|c|c|}
	& 1 Coordinate  & +10 Coordinates   & +37 Coordinates  \\ \hline
Random & 0.148988 &  0.614531 & 2.887779 \\ \hline
Optimal & 0.148755 &  1.558550 & 1.876363 \\ \hline
\end{tabular}
\end{center}
\vspace{1ex}
We can then say that the ratio between the length of the first 10 and the last 37 coordinate segments is 0.212804 with a random ordering and 0.830623 with an optimal ordering of the coordinates. 
 
To continue, we need a new sample of \textsf{RGB} patches and the optimal coordinates for each one.  In the MNIST example, we restricted this new sample to the positive hemisphere, that is, to the half-space that extends from the principal axis to the Euclidean angles $\pi/2$ and $-\pi/2$.  Although we can compute $\Theta$ coordinate curves beyond these points, they become less accurate when the angles become much larger.  If we were developing a production model for an image processing task, the best approach would be to find the antipodal point in the negative hemisphere and construct additional $\Theta$ coordinate curves from there, as we did for the curvilinear Gaussian example in \citep{CCCS_AMAI}, Section 5.2.  For our present purposes, however, we can glean enough information from one hemisphere, and to do more would be duplicative.  We therefore generate 1200 new \textsf{RGB} samples apiece for Prototypes 14 and 32, and restrict these to their positive hemispheres, in two ways:  (i) we can either take the \textit{intersection} of the positive hemispheres for the three \textsf{R/G/B} channels, or (ii) we can take the \textit{union}, and it turns out to be advantageous to compute both and use each one for slightly different purposes.  For Prototype 14, the intersection contains 481 points and the union contains 899 points.   For Prototype 32, the intersection contains 185 points and the union contains 1065 points.  (Note: This comparison shows, incidentally, how the principal axes are splayed out much more in Prototype 32 than in Prototype 14.)
 
Let's now compute the $\Theta$ coordinate clusters, as we did in the MNIST example.  The first fact we discover is that \texttt{FindClusters} in \emph{Mathematica} does not work for the CIFAR-10 data.  The clusters are completely ``flat'' with modes of at most two points.  But it turns out, fortuitously, that we can use the theory of differential similarity to find clusters in the space of optimal coordinates.  Consider a point with a list of 10 optimal coordinates, such as the example above from Prototype 14, Cluster 01, Channel RED.  Write this as a 47-dimensional binary vector encoding the positions of the coordinates:
\begin{equation*}
\left(
\begin{array}{c}
00000001011000000100000100001000010010000000101
\end{array}
\right)
\end{equation*}
Do the same with the GREEN and the BLUE channels, and construct the 141-dimensional product space.  We can now define $\nabla U({\bf x})$ on this space, treating the \textsf{RGB} product as a real-valued vector space with the sample points situated at the corners of a 141-dimensional unit cube.  To find the cluster modes, we follow a variant of the strategy that we developed above to construct the CIFAR-10 prototypes.  To allocate \textsf{RGB} patches to clusters, we compute the Riemannian distance along the integral curve of $\nabla U$ towards each cluster mode, and select the mode that is the shortest distance away.

Within this framework, there are several alternatives.  Let's assume that we want the cluster modes to lie in the positive hemisphere in all three color channels. This means that we need to work with the intersection of our sample points, 481 points for Prototype 14 and 185 points for Prototype 32.  Let's assume that we want six cluster modes for each prototype.  After some experimentation, it is easy to find a set of parameters that produces these results.  For Prototype 14, we run \texttt{pgradascent} with 200 random choices out of 481 data points for \texttt{xstart}, with $\beta = 1/8$, and with a 40 point \texttt{Sample} retrieved by the \texttt{NearestFunction} at each iteration.  For Prototype 32, we run \texttt{pgradascent} with all 185 data points for \texttt{xstart}, with $\beta = 1/8$, and with a 30 point \texttt{Sample} retrieved by the \texttt{NearestFunction} at each iteration.  The filtering strategy is the same for both prototypes:  The outputs from \texttt{pgradascent} are sorted by their estimated probability values, then processed from max probability to min probability, and a mode is filtered out if it lies within one-half of the radius of a 20 point data sphere around one of its sorted predecessors.  Note that the final six modes will be sorted by probability, but they will be continuous-valued, not discrete.  If we want binary-valued outputs, we can set the top ten coordinates in each color channel to 1 and the rest to 0.  The highest probability mode for Prototype 14 is then:
\vspace{1ex}
\begin{equation*}
\left(
\begin{array}{cccccccccc}
08 & 11 & 18 & 23 & 24 & 29 & 34 & 37 & 45 & 47 \\
10 & 12 & 16 & 18 & 19 & 22 & 25 & 32 & 33 & 45 \\
03 & 04 & 05 & 12 & 13 & 23 & 29 & 30 & 40 & 47
\end{array}
\right)
\vspace{1ex}
\end{equation*}
which corresponds to Cluster 01 in Figure \ref{ClusterModesP14}.  Notice, though, that this is not quite the same as our running example, since the coordinates in the RED channel are slightly different. Instead, our running example is the sample point that is the \textit{nearest} to the mode for Cluster 01, and it has the following optimal coordinates: 
\vspace{1ex}
\begin{equation*}
\left(
\begin{array}{cccccccccc}
08 & 10 & 11 & 18 & 24 & 29 & 34 & 37 & 45 & 47 \\
09 & 10 & 16 & 17 & 18 & 19 & 22 & 32 & 33 & 45 \\
02 & 03 & 04 & 05 & 09 & 26 & 29 & 30 & 44 & 47
\end{array}
\right)
\vspace{1ex}
\end{equation*}
This is also the criterion for the selection of the other patches in the first row of Figures \ref{ClusterModesP14} and \ref{ClusterModesP32}. 

When we use the coordinates for Mode 01 to compute distances on the Frobenius integral manifold and its ellipsoidal approximation, they are slightly less than optimal.  For example, in this table:
\vspace{1ex}
\begin{center}
\begin{tabular}{r|c|c|c|}
	& 1 Coordinate  & +10 Coordinates   & +37 Coordinates  \\ \hline
Mode 01 & 1.033560 &  0.639119 & 0.113119 \\ \hline
\end{tabular}
\end{center}
\vspace{1ex}
the distance at the end of coordinate curve $\theta^{11}$ is 0.639119 instead of 0.608948, and in this table:
\vspace{1ex}
\begin{center}
\begin{tabular}{r|c|c|c|}
	& 1 Coordinate  & +10 Coordinates   & +37 Coordinates  \\ \hline
Mode 01 & 0.148732 &  1.506547 & 1.952874 \\ \hline
\end{tabular}
\end{center}
\vspace{1ex}
the ratio between the length of the first 10 and the last 37 coordinate segments is 0.771451 instead of 0.830623. But this is exactly what we want from cluster modes.  We can compute optimal coordinates, as we have seen, but that would give us a different subspace for each sample point, which is not very useful.
What we want is a small number of clusters that provide reasonable approximations to these optimal calculations. 

For the next step, allocating \textsf{RGB} patches to clusters, let's work with the union of our sample points: 899 points for Prototype 14 and 1065 points for Prototype 32.  Given any \textsf{RGB} patch, we need to compute the Riemannian distance of its optimal coordinates along the integral curve of $\nabla U$ towards each of the cluster modes.  For Prototype 14, we set $\beta = 1/16$ and we set $\mathtt{SamplePoints}$ to the 40 points that are the nearest to the cluster mode in the space of optimal coordinates. For Prototype 32, we set $\beta = 1/16$ and we set $\mathtt{SamplePoints}$ to the 30 points that are the nearest to the cluster mode in the space of optimal coordinates.  The patch is then assigned to the mode from which the Riemannian distance is the least.  Note that these calculations are using distances and probabilities \textit{in the space of optimal coordinates}, not in the original image space.  The sizes of the resulting clusters are shown in Table \ref{ClusterSizeP14P32}. 

\begin{table}[htbp]
\caption{Number of data points in the $\Theta$-Coordinate Clusters from Prototype 14 (see Figure \ref{ClusterModesP14}) and Prototype 32 (see Figure \ref{ClusterModesP32}).}
\begin{center}
\begin{tabular}{r|c|c|c|c|c|c|c}
        & \multicolumn{6}{c|}{$\Theta$-Coordinate Cluster} & \\ \cline{2-7}
	& 1 & 2 & 3  & 4 & 5 & 6  & Total \\ \hline
 Prototype 14 &  \, 149 \, & \, 180 \, & \, 127 \, & \, 108 \, & \, 121 \, & \, 214 \, & \, 899 \, \\ \hline
 Prototype 32 & 217 & 99 & 135 & 172 & 223 & 219 & 1065 \\ \hline                                                
\end{tabular}
\end{center}
\label{ClusterSizeP14P32}
\end{table}

Do these clusters, in fact, provide reasonable approximations to the optimal calculations in the original image space? Using standard statistical tests, we can show that, within each cluster, the Riemannian/Euclidean distance along the first 10 coordinates is greater (statistically significant in most cases\footnote{For Prototype 32, there are three exceptions with $0.1 > p > 0.001$ and one exception with $p > 0.1$.  Total cases: 36 $=$ 2 prototypes $\times$ 6 clusters $\times$ \textsf{R/G/B}.} at $p =0.001$) and the Riemannian/Euclidean distance along the remaining coordinates is less (statistically significant in most cases\footnote{For Prototype 14, there is one exception with $p > 0.1$. For Prototype 32, there are six exceptions with $0.1 > p > 0.01$ and four exceptions with $p > 0.1$.  Total cases: 36 $=$ 2 prototypes $\times$ 6 clusters $\times$ \textsf{R/G/B}.} at $p =0.01$) than the corresponding distances for points outside the cluster.  For an indication of the magnitude of these differences, see Table \ref{ErrorRatiosP14P32}.  The numbers in this table are the ratios between the mean values of the Riemannian/Euclidean distances along the first 10 coordinates and the mean values  of the Riemannian/Euclidean distances along the remaining coordinates, as developed in the MNIST example.  For the column labelled \textit{Optimal}, the distances are computed using the optimal coordinates for every data point in every cluster, but separately for each color channel.  This column thus sets the scale for the highest possible ratios in each case.\footnote{\label{fnOptimal}As explained in footnote \ref{fnComplex}, Prototype 14, GREEN, has only 45 coordinate curves for the minimal infinitesimal eigenvectors, and Prototype 32, RED and BLUE, has only 43. Thus the ratios in these cases are artificially high.} In the next two columns, the distances are computed using the coordinates for each cluster, in two different ways.  For the column labelled \textit{Within}, the mean values are computed for the data points within each cluster, using just the coordinates for that cluster mode.  For the column labelled \textit{Across}, the mean values are computed for all the data points using the coordinates for all the cluster modes.  Comparing the ratios in these two columns to the ratios for the optimal coordinates, there is a consistent pattern for the two prototypes and the three color channels.  The evidence therefore supports the claim that these are reasonable approximations. 

\begin{table}[htbp]
\caption{Riemannian distance ratios along the surface of the ellipsoidal approximations for the six $\Theta$-Coordinate Clusters from Prototype 14 and Prototype 32.}
\begin{center}
\begin{tabular}{r|c|c|c|c|c|}
        \multicolumn{4}{c}{ }   & \multicolumn{2}{|c|}{$\Theta$-Coordinate Clusters}  \\ \cline{5-6}

	& N &  & Optimal  & Within    & Across \\ \hline

                       &            & R & \; 0.800961 \; & \; 0.385669 \; & \; 0.303989 \; \\ 
 Prototype 14  &  899   & G & 0.863663 & 0.426305 & 0.324552 \\
                       &            & B & 0.802143 & 0.379088 & 0.294074 \\ \hline

                       &           & R & 0.954155 & 0.393014 & 0.331195 \\ 
 Prototype 32 & 1065  & G & 0.831086 & 0.339646 & 0.292296 \\
                      &           & B & 0.958359 & 0.388698 & 0.336805 \\ \hline      
                                           
\end{tabular}
\end{center}
\label{ErrorRatiosP14P32}
\end{table}

The ellipsoidal approximation played a central role in the preceding analysis, but it actually has two somewhat different roles to play in the theory of differential similarity. First, it could be a true approximation, because we are primarily interested in finding points on the Frobenius integral manifold that satisfy various conditions.  Second, it could be a target for the mapping to a lower-dimensional subspace, regardless of its relationship to the manifold. 
 
For an example of the first role, consider again our running example:  Prototype 14, Cluster 01, Channel RED.  Using the coordinates for Mode 01, we saw that the distance to the target point on the Frobenius integral manifold from the point on the ellipsoidal approximation at the end of coordinate curve $\theta^{11}$ was 0.639119.  We would like to transfer this result to the manifold itself.  The first step is to solve Equations \eqref{TransverseCurves}, the integral equations for the coordinate flows on the manifold, one segment at a time, matching the flows on the ellipsoid.  For example, if ${\bf x}_{0}$ is the initial point on the principal axis, we start by solving for $\vec{\theta}^{\,1}_{t_{1}}({\bf x}_{0})$ and then computing the value of the parameter,  $t_{1}$, that minimizes the distance from $\vec{\theta}^{\,1}_{t_{1}}({\bf x}_{0})$ to the end of coordinate curve $\theta^{1}$, calling this the \texttt{baseparameter}.  Continuing in the same manner, we compute the parameters for $\vec{\theta}^{\,2}_{t_{2}}({\bf x}_{1}), \vec{\theta}^{\,3}_{t_{3}}({\bf x}_{2}), \ldots , \vec{\theta}^{\,11}_{t_{11}}({\bf x}_{10})$, calling these the \texttt{flowparameters}.  These first steps are packaged into a function called \texttt{InitializeFlowParameters}.  We now have a sequence of coordinate points on the manifold, but the end result is actually further away from the target point, by a small amount: 0.673622.  The next step is to apply a function called \texttt{UpdateFlowParameters}, and to do so iteratively.  This function starts with solutions $\vec{\theta}^{\,1}_{t_{1}}({\bf x}_{0}), \vec{\theta}^{\,2}_{t_{2}}({\bf x}_{1}), \ldots , \vec{\theta}^{\,11}_{t_{11}}({\bf x}_{10})$, for a fixed sequence of starting points, ${\bf x}_{0}, {\bf x}_{1}, \ldots, {\bf x}_{10}$, and computes the \texttt{baseparameter} and the \texttt{flowparameters} for a modified formula that minimizes the distance to the target point, but takes the coordinate curves off the manifold.  This is not a problem, however, because we can now use the new \texttt{baseparameter} and \texttt{flowparameters} to solve the integral equations again for the coordinate flows.  After several iterations of \texttt{UpdateFlowParameters}, our algorithm converges to a point on the manifold at a distance of 0.442212 from the target point.\footnote{\label{fnConvergence}Conditions for the convergence of this algorithm are still open, and need to be investigated further.}  This is the point that produces the RED channel in Cluster 01, in the third row of Figure \ref{ClusterModesP14}.  Similar calculations produce the GREEN and the BLUE channels in Cluster 01, as well as the  other images in the third row of Figures \ref{ClusterModesP14} and \ref{ClusterModesP32}. 

The reader has probably noticed that there are high frequency variations in the color patterns in the third row of Figures \ref{ClusterModesP14} and \ref{ClusterModesP32}.  The explanation is simple:  We are looking at an 11-dimensional submanifold in each color channel, with the remaining coordinates truncated and their values set to 0.0.  Since the point of truncation is essentially random, and different for each color, we see random color variations.  It is natural to refer to this phenomenon as ``truncation noise.''  However, this is just a problem with the presentation of the submanifold when it is embedded in the original 147-dimensional space, and if we could smooth out the noise we would see that the global shape and the overall color balance of the \textsf{RGB} patches has been preserved.  There are several ways to do this.  For example, one cheap cosmetic trick is to replace the truncated coordinate values with the corresponding coordinate values from the sample point nearest to the cluster mode, that is, from the first row of Figures \ref{ClusterModesP14} and \ref{ClusterModesP32}.  


In the second role for the ellipsoidal approximation, however, we don't care about the mapping back to the Frobenius integral manifold because we are only interested in the mapping forward. At this point, we have three 12-dimensional manifolds defined by a $\rho$, $\Theta$, coordinate system, and we want to apply the theory of differential similarity again to construct a lower-dimensional submanifold of the 36-dimensional product manifold. But to do this, we need to work with sample data points in a Euclidean space, for two reasons: (i) the fundamental equation for Brownian motion with a drift term, Equation \eqref{BMwithDrift}, is only valid in a Euclidean space with a Cartesian coordinate system;  and (ii) the mean shift algorithm, Equation \eqref{EstimateGradU} and Figure \ref{Kernel}, only works in a Euclidean space with a Cartesian coordinate system.  The ellipsoidal approximation gives us exactly what we want, in the simplest possible form.  Although the coordinate flows are defined on the surface of an ellipsoid, they are expressed in Cartesian coordinates: $(x, y, z_{1}, \ldots, z_{47})$, initially, but then truncated to $(x, y, z_{1}, \ldots, z_{10})$. In addition, the product manifold is just the ordinary Cartesian product.

Let's see how this works for the data points in the $\Theta$ coordinate clusters in Table \ref{ClusterSizeP14P32}. We start with a 36-dimensional Cartesian coordinate system and our goal is to reduce it to 12 dimensions. For each of the 2 $\times$ 6 clusters, we proceed through the same steps we have seen several times before:  find a prototype, choose a coordinate sphere, locate the principal axis and define the $\rho$ coordinate curves, compute the maximal and minimal infinitesimal eigenvectors, compute the geodesics for the $\theta$ coordinate curves, define the $\theta$ coordinate flows, and compute the optimal ellipsoidal approximation for the coordinates on the Frobenius integral manifold that are determined by the minimal infinitesimal eigenvectors.  The results shown in Table \ref{Optimal10D} are analogous to the \textit{Optimal} results in Table \ref{ErrorRatiosP14P32}:  These are the ratios between the mean values of the Riemannian/Euclidean distances along the 10 optimal coordinates and the mean values of the Riemannian/Euclidean distances along the remaining coordinates.\footnote{These numbers are higher than the numbers in Table \ref{ErrorRatiosP14P32} because there are 36 coordinates in this case, in total, and we are computing the ratio between the optimal 10 coordinates for the minimal infinitesimal eigenvectors and the remaining 24 coordinates.  See also footnote \ref{fnOptimal}.  We can rescale the denominator in these ratios by $37/24$ to get an approximate comparison.  For example, in Prototype 14, Cluster 01, the ratio 1.29390 would be rescaled to 0.839284.}  

\begin{table}[htbp]
\caption{Riemannian distance ratios for the optimal 10-dimensional coordinates in the $\Theta$-Coordinate Clusters from Prototype 14 and Prototype 32.}
\begin{center}
\begin{tabular}{r|c|c|c|c|c|c|}
        & \multicolumn{6}{c|}{$\Theta$-Coordinate Cluster} \\ \cline{2-7}
	& 1 & 2 & 3  & 4 & 5 & 6  \\ \hline
 Prototype 14 &  1.29390 &  1.16251 &  1.31145  &  1.33407  &  1.29024 &  1.25100   \\ \hline
 Prototype 32 & 1.35475 & 1.16359 & 1.08401 & 1.57556 & 1.19799 & 1.31949  \\ \hline                                                
\end{tabular}
\end{center}
\label{Optimal10D}
\end{table}


We can now process any arbitrary 7$\times$7  \textsf{RGB} patches through these coordinate systems.  As an illustration, let's look at two more random points from Prototype 14 and Prototype 32.  See Figure \ref{FinalImagesP14P32}, first column. (Note: For each prototype, we generated 100 additional random patches from the 8,000 point Data Sphere, filtered out all patches with a color channel in the negative hemisphere, and selected the one located at the greatest distance from the prototype.) We have seen columns two and three before, in other examples.  For the fourth column, for each prototype, we have computed the Riemannian distance in the space of optimal coordinates to the modes of all six clusters, and selected the nearest one, which happens to be Mode 04 in each case.  Finally, for the fifth column, we computed the coordinates of the data point in the optimal 12-dimensional ellipsoidal approximation for Cluster 04, as described above.  Notice that there is a slight smoothing of the truncation noise in these final \texttt{12D} images.  

\begin{figure}[htbp]
\begin{center}
\includegraphics[width=5.0in]{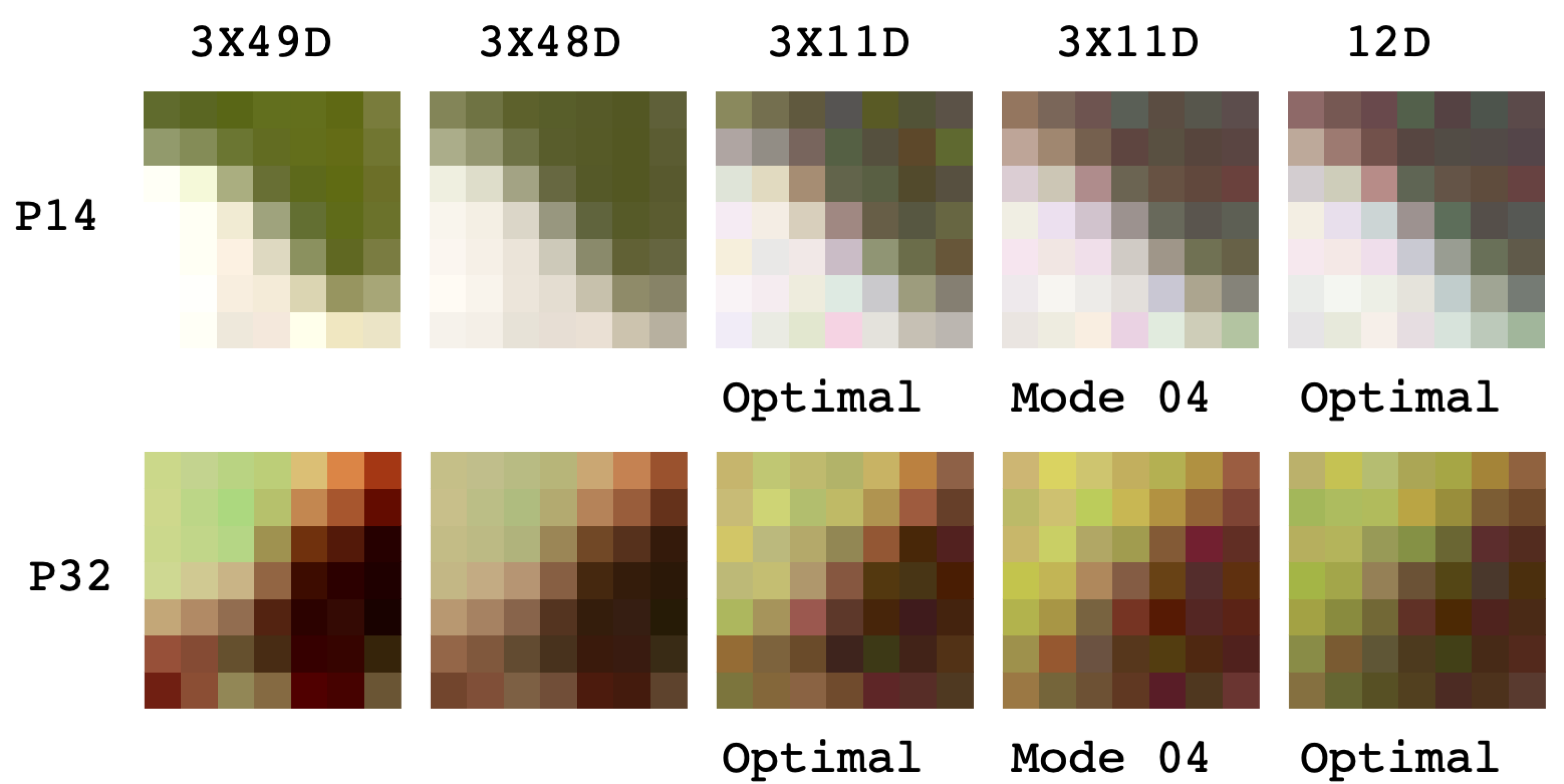}
\caption{Mapping patches within the Data Sphere of Prototypes 14 and 32 from the original 147-dimensional space to an optimal 12-dimensional submanifold.}
\label{FinalImagesP14P32}
\end{center}
\end{figure}

Figure \ref{FinalImagesP14P32} also illustrates how the theory of differential similarity can be implemented recursively.  The base step starts with a random sample of data points and constructs an optimal 12-dimensional submanifold in each \textsf{R/G/B} color channel.  The first recursive step for product manifolds starts with another random sample of data points, and sends it through the pipeline to construct an optimal 12-dimensional submanifold in each $\Theta$ coordinate cluster.  The next step would repeat this process for the 14$\times$14 sample space, and so on.  We will not pursue these subsequent steps in the present paper, but the reader should be able to anticipate roughly how they would operate.  


\section{Future Work.}
\label{FutureWork}
 
The motivating principle in the theory of differential similarity is simple:  We want to construct a geometric model that can be tuned to match the probabilistic model, in any given domain. 
 
This is easy to do for the three-dimensional examples in \citep{CCCS_AMAI}. The origin of the coordinate system is the mode of the probability density; the principal axis is a $\rho$ coordinate curve with a minimal Riemannian length, measured inwards from a fixed Euclidean distance; the Frobenius integral manifold is a surface of constant probability density, at a fixed Riemannian distance from the origin, parameterized by a pair of $\Theta$ coordinate curves. We saw in the present paper that the geometric models for the higher-dimensional MNIST and CIFAR-10 examples are exactly the same, but with many more $\Theta$ coordinate curves. Most significantly, as shown in Table \ref{rhoDistances} and Figure \ref{3DVisualization}, the data points on the Frobenius integral manifold for MNIST are all (approximately) the same Riemannian distance from the origin, even though their Euclidean distances vary substantially.  This is a consequence of Theorems \ref{ConstantRho} and \ref{ConstantRiemannian} in Section \ref{GeomM}, which hold for any $n$-dimensional Euclidean space.
   
When we proceed to the construction of a lower dimensional subspace, however, the three-dimensional cases and the higher dimensional cases begin to diverge. First, in the three-dimensional case, we rotated the maximal infinitesimal eigenvector so that it was aligned with the maximal geodesic over a finite distance, and used this geodesic as the first of our $\Theta$ coordinates.  In the higher dimensional case, this calculation was too complex, so we just used the infinitesimal eigenvector directly.  Second, trying to define an optimal lower dimensional subspace in the three-dimensional case, we observed a conflict between maximizing the ``variance'' and minimizing the ``reconstruction error,'' and we opted for the former.  In the higher dimensional case, our measures of ``variance'' and ``reconstruction error'' were consistent, and we were able to show in Figure \ref{EigMaxProjections} that the projection of data points onto the $\Theta$ coordinate curve for the maximal infinitesimal eigenvector was the optimal choice for the first $(\rho, \theta^{1})$ coordinates.  For the remaining coordinates, $\{ \phi^{j} \mid j = 1, 47 \}$, we noted that our data points fell into empirical clusters, shown in Figure \ref{ClusterModes} for MNIST and in Figures \ref{ClusterModesP14} and \ref{ClusterModesP32} for CIFAR-10.  These clusters are encoded using small subsets of the $\Theta$ coordinates, fixed arbitrarily at ten dimensions, so that each cluster is represented by a distinct direction from the prototype in the high dimensional image space. 
  
To compute a representation of the $\Theta$ coordinate clusters, it was necessary to develop an extension of the basic theory.  First, we showed that it is possible to compute the optimal coordinates for each data point.  All we need to do is to measure the Riemannian/Euclidean distance along each segment of the coordinate flows to the target point on the Frobenius integral manifold, and then choose the ten segments with the maximal length.  Second, we showed how to search for clusters in the space of optimal coordinates.  In the MNIST case, we could do this using \texttt{FindClusters} in \emph{Mathematica}, but in the CIFAR-10 case, \texttt{FindClusters} did not work, and we computed clusters using a novel adaptation of the theory of differential similarity.  The basic idea is to encode the optimal coordinates in a binary vector, and then define $\nabla U({\bf x})$ on a real-valued vector space, with the binary sample points situated at the corners of a unit cube in that space.  It is remarkable that this idea works, since we are mapping a discrete search problem into a continuous space, using $\nabla U({\bf x})$ to compute the modes of the probability density in that space, which is a continuous computation, and then truncating the solutions back to a set of binary vectors.  It would be interesting to investigate this idea further, in different but related contexts. 

There were two more additions to the basic theory in the discussion of CIFAR-10 in Section \ref{AppCIFAR10}: quotient manifolds and product manifolds. Quotient manifolds were used to build invariance into the geometric model, consistent with known properties of the probabilistic model.   The basic idea is to form a quotient space over the space of images, and then work with a representative of this quotient in our probabilistic calculations.  This approach should be studied further, and compared to similar approaches in papers such as \citep{ANSELMI2016112, pmlr-v48-cohenc16}.  Product manifolds were used to combine low-dimensional solutions into a higher dimensional problem space, so that our dimensionality reduction techniques could be applied recursively.  The main example in Section \ref{AppCIFAR10} was the 147-dimensional \textsf{RGB} manifold constructed out of three 49-dimensional \textsf{R/G/B} manifolds.  It is important to understand how our dimensionality reduction techniques, such as the search for $\Theta$ coordinate clusters, as discussed in the previous paragraph, are configured to work with product manifolds.  Recall that we computed the $\Theta$ coordinate clusters and the Riemannian distances within them in a 141-dimensional product space constructed from the optimal coordinates for the \textsf{RED}, \textsf{GREEN}, \textsf{BLUE}, channels \textit{in combination}, not separately.  If we had instead computed $6$ clusters in each channel, we would have to consider $6^{3}$ clusters in the product manifold, and perhaps estimate the probabilities of each combination in order to decide which ones to retain and which ones to discard.  The point of any dimensionality reduction technique is that we should expect to lose information, but we want (i) the simplest coding scheme that (ii) preserves as much information as possible.  For these reasons, in a coding scheme that replaces optimal coordinates with cluster modes, it is better to compute the cluster modes in the product manifold than in the separate factor manifolds. 

The same strategy can be applied to the 14$\times$14 sample space in Figure \ref{ArchitectureDL}.  If we computed, say, 40 prototypes in the 7$\times$7 sample space, we would have to consider $40^{4} = 2,560,000$ prototypical clusters in the product manifold.  Instead, we can compute 40, 50, 60, $\ldots$, prototypes directly in the 2$\times$2 product space.  We can then apply the same strategies that we used for the \textsf{R/G/B} product manifolds in the CIFAR-10 example to construct a small number of $\Theta$ coordinate clusters for each prototype.  Moreover, any prototype from the 7$\times$7 sample space that does not appear in the 2$\times$2 product space can be pruned from our construction, and this process of proliferating and pruning product manifolds can be continued, recursively, up through the hierarchy in Figure \ref{ArchitectureDL} and similar architectures.  Our initial results on the MNIST dataset show that this idea works, but we cannot yet say how well it works.  There are several hyperparameters that have to be set: the constant $\beta$, the size of $\mathtt{SamplePoints}$, the size of the Data Sphere and the Coordinate Sphere, etc., and these values can vary as we move through the architecture.  The architecture itself is a variable:  the size of the patches, the configuration of the product manifolds, the number of prototypes, the number of optimal $\Theta$ coordinates, etc., all have an impact on the outcomes.  These are empirical questions, which are currently under investigation.  When we have fully analyzed the experimental results, we will publish them in a paper with the working title: ``Deep Learning with a Riemannian Dissimilarity Metric.'' 
 
There are several open mathematical questions about the Riemannian manifolds that we have been studying in this paper.  We have seen two such questions so far:   (1) At the end of Section \ref{GeomM}, we discussed the complexity of rotating and optimizing the maximal infinitesimal eigenvectors in a high dimensional space.  In our naive algorithm, each step in the search for an optimal rotation requires the numerical solution of the Euler-Lagrange equations for the geodesic coordinate curves.  Is it possible to solve these equations just once and perturb the solution to find the optimal rotation?  If not, can we show that the rotation is unnecessary, because the maximal infinitesimal eigenvector is already a good approximation to the optimal value?  Recall that this latter alternative was the pragmatic choice that we made in the present paper.  (2) In Sections \ref{AppMNIST} and \ref{AppCIFAR10}, we introduced an ellipsoidal approximation to reduce the complexity of finding an optimal solution to the  flow equations in \eqref{TransverseCurves}.  One open question, mentioned in footnote \ref{fnConvergence}, is to find conditions for the convergence of the algorithm that we developed to map the optimal sequence of transverse coordinate curves from the ellipsoidal surface back to the Frobenius integral manifold.  But another possibility would be to incorporate the ellipsoidal approximation into the calculations on the manifold itself, so that we can solve this problem in one step rather than two. There is now a substantial literature on optimization techniques for Riemannian manifolds, see, e.g., \citep{boumal2023intromanifolds}, and a rapidly growing literature on data analysis in high dimensional spaces, see, e.g., \citep{vershynin2018}.To advance the theory of differential similarity, we need to merge these two topics and develop techniques for the optimal analysis of data on high dimensional manifolds.  

The MNIST dataset has often been used as a testbed for experiments on unsupervised learning.  Much of the recent work on this problem has been based on deep generative models implemented with neural networks, in particular, either Variational Auto-Encoders (VAEs) \citep{Kingma2014, pmlr-v32-rezende14}, Generative Adversarial Networks (GANs) \citep{GoodfellowGANs2014}, or a combination of the two \citep{MakhzaniAAE2016, MeschederAVB2017}.  In the VAE framework, for example, if $\mathbf{z}$ is a latent variable (also called a ``code'') and $\mathbf{x}$ is a data point, then the \textit{generator} (also called the ``decoder'') is the probability $p( \mathbf{x} | \mathbf{z} ; \theta )$ with prior $p( \mathbf{z} ; \theta )$, which is usually parameterized by a neural network.  The true posterior, $p( \mathbf{z} | \mathbf{x} ; \theta )$, is assumed to be intractable, and so it is approximated by the \textit{recognizer} $q( \mathbf{z} | \mathbf{x} ; \phi )$ (also called the ``encoder''), which is itself parameterized by a neural network. The parameters $\phi$ and $\theta$ are estimated by maximizing a lower bound on the log-likelihood of the data, a process called \textit{variational inference}.  If we can learn these conditional probabilities, then there are several ways to build a \textit{semi}-supervised classifier.  For an example of how this works on the MNIST dataset with only 100 labelled examples, see \citep{Kingma&RezendeNIPS2014}.  The pure GAN framework does not have a recognizer, but the generator, $p( \mathbf{x} | \mathbf{z} ; \theta )$, is trained with an auxiliary discriminative network acting as an adversary.  One way to use a GAN for unsupervised learning is to add a term to the loss function for the auxiliary discriminative network that maximizes the mutual information between the data distribution and the predicted class distribution, in an attempt to \textit{disentangle} the code for the classes. See, e.g., \citep{springenberg2016iclr, ChenInfoGAN_NIPS2016}.  Although deep generative models have achieved good results for unsupervised learning on MNIST, they do not seem to work as well on natural images.  On CIFAR-10, for example, the best current results for unsupervised learning are instead based on a variety of ad hoc methods: \citep{HanPPKCeccv2020, VanGansbekeECCV2020, Park0KKPHCcvpr2021}. For a general critique of almost all recent work on the use of deep generative models to disentangle the \textit{factors of variation} in a dataset of images, without supervision, see \citep{pmlr-v97-locatello19a}. 
  
In Section 1, \textit{supra}, we mentioned an earlier class of auto-encoders, known generically as Regularized Auto-Encoders.  In particular, we noted that Denoising Auto-Encoders (DAEs) \citep{Vincent:2011} and a specialized form of Contractive Auto-Encoders (CAEs) \citep{Alain&Bengio:2012} can be shown to compute the gradient of the log of the input probability density \citep{Bengio_etal:2013}.  In other words, these are network architectures that can be trained to compute the vector field  $\nabla U({\bf x})$.  This correspondence suggests several questions for further investigation:
\begin{enumerate}
\vspace{1ex}
\item The earliest papers on Denoising Auto-Encoders, \citep{Vincent_etal-ICML2008} and  \citep{Vincent_etal-JMLR2010}, offer an intuitive justification of the DAE algorithm in terms of the manifold hypothesis.  Consider Figure 2 in these papers, and the following informal explanation:

\begin{quotation}
\vspace{1ex}
\noindent
During denoising training, we learn a stochastic operator $p( X | \tilde{X} )$ that maps a corrupted $\tilde{X}$ back to its uncorrupted $X$ \ldots .  Corrupted examples are much more likely to be outside and farther from the manifold than the uncorrupted ones. Thus stochastic operator $p( X | \tilde{X} )$ learns a map that tends to go from lower probability points $\tilde{X}$ to nearby high probability points $X$, on or near the manifold. \citep{Vincent_etal-JMLR2010}, at 3380.
\end{quotation} 
\vspace{1ex}
Similar justifications are offered in \citep{Rifai2012AGP} to explain how Contractive Auto-Encoders tend to generate samples along a data manifold.  Can these explanations be formalized within the theory of differential similarity, using Brownian motion with the drift term $ \nabla U(\mathbf{x}) \cdot \nabla $ as a model?  See Section 2.3 in \citep{CCCS_AMAI}, and see the derivations in Section 6 of \citep{CCCS_AMAI}. 
\vspace{1ex}

\item In another theoretical perspective on Denoising Auto-Encoders, \citep{Vincent_etal-ICML2008} shows how to derive the DAE training criterion by maximizing a variational bound on a generative model that includes the corruption of the data: $ X \rightarrow \tilde{X} $.  This paper thus anticipates some of the later work on Variational Auto-Encoders \citep{Kingma2014, pmlr-v32-rezende14}. The theory of differential similarity does not use a neural network, of course, but we are computing in our geometric model a low-dimensional code $\mathbf{z} = r(\mathbf{x})$ for the data point $\mathbf{x}$, with some probability density $q( \mathbf{z} | \mathbf{x} )$.  It would be interesting to compare these architectures, layer by layer.  
\vspace{1ex}

\item \citet{Vincent:2011} reported that there was an ``unsuspected link'' between the training of a Denoising Auto-Encoder and the \textit{score matching} technique \citep{ScoreMatchingJMLR2005, ExtensionsofScoreMatchingCSDA2007} for learning the parameters of an unnormalized density model over continuous-valued data.  This result was later generalized in \citep{Alain&Bengio:2012} to a specialized form of Contractive Auto-Encoder. What Hyv\"{a}rinen refers to as the ``score'' is just the vector field $\nabla U({\bf x})$.  Using his score matching technique, Hyv\"{a}rinen avoids estimating $\nabla U({\bf x})$ directly. We have taken the opposite approach here.  We use the Frobenius integral manifold as a model of the data distribution, and this gives us an interesting compromise between traditional parametric and traditional nonparametric statistical estimation.  Note that the Frobenius integral manifold can represent almost any probability density, subject to mild regularity conditions, but its shape is strictly determined once we specify an orthonormal set of infinitesimal eigenvectors at a point along the principal axis.  This appears to be a novel statistical estimation technique, but it needs to be studied more carefully.  
\vspace{0.1ex}

\item We have used very small sample sizes in our work on the MNIST and CIFAR-10 datasets, as compared to the sample sizes needed to train neural networks.  See, for example, the size of the $\Theta$ coordinate clusters in Table \ref{ClusterSizeP14P32}.  However, empirically, these small samples seem to yield consistent and stable estimates for the $\rho$  and $\Theta$ coordinate curves, and it is important to understand, theoretically, why this is the case.  A significant step in this direction is the work of \citet{Arias-Castro_etal:2016} on the consistency and stability of the mean shift algorithm, which answers some of our questions about the $\rho$ coordinate curves.  Similar results for the $\Theta$ coordinate curves would be very useful. 

\end{enumerate}
\vspace{1ex}
Note that the $\rho$, $\Theta$, coordinate system imposes a strong inductive bias on both the model and the data, as advocated in \citep{pmlr-v97-locatello19a}, and this bias can be manipulated systematically by choosing different quotient manifolds and product manifolds in different domains. 
 
There are two popular dimensionality reduction algorithms that compute and optimize probability distributions on low-dimensional manifolds: (i) SNE or t-SNE, and (ii) UMAP.  Since these algorithms are used primarily for data visualization, the embedding spaces are usually only two- or three-dimensional.  Stochastic Neighbor Embedding (SNE) \citep{HintonRoweis2002} centers a Gaussian at each data point $\mathbf{x}_{i}$ in the original high dimensional space, and also at each data point $\mathbf{y}_{i}$ in the low dimensional space, and uses the conditional probabilities derived from these Gaussians to measure the similarities of $\mathbf{x}_{i}$ and $\mathbf{y}_{i}$ to each of their $k$ neighbors. The goal is to adjust the mapping $\{\mathbf{x}_{i}\} \mapsto \{\mathbf{y}_{i}\}$ so that the mismatch between the probability distributions in the high and low dimensional spaces is as small as possible, and this is achieved by minimizing the sums of the Kullback-Leibler divergences over all data points. An improvement of this algorithm, called t-SNE \citep{van2008visualizing}, replaces the Gaussian in the low dimensional space by a Student t-distribution with one degree of freedom, because it has much heavier tails.  This is all somewhat ad hoc, and the Uniform Manifold Approximation and Projection (UMAP) algorithm \citep{mcinnes2020umap} achieves similar and arguably better results with a rigorous mathematical foundation. The basic idea of UMAP is to rescale the metric at each point $\mathbf{x}_{i}$ in the high dimensional space, in order to approximate a uniform probability density in that space, and then to glue these incompatible metric spaces together, using a clever construction from algebraic topology and category theory. The optimization algorithm can then minimize the cross-entropy between the topological representation of the data in the high dimensional space and the corresponding topological representation in its low dimensional projection.  Here is a comment by the authors on their methodology:
\begin{quotation}
\vspace{1ex}
\noindent
\ldots [W]e feel that strong theory and mathematically justified algorithmic decisions are of particular importance in the field of unsupervised learning. This is, at least partially, due to a plethora of proposed objective functions within the area.  \ldots  UMAP's design decisions were all grounded in a solid theoretic foundation and not derived through experimentation with any particular task focused objective function. \citep{mcinnes2020umap} at 2--3.
\end{quotation}
\vspace{1ex}
Although the goals are quite different, the theory of differential similarity also has a solid mathematical foundation, drawn from the fields of stochastic processes and differential geometry.  See \citep{CCCS_AMAI}. 
   
There is one operation that appears repeatedly in the diagram in Figure \ref{ArchitectureDL}:  We construct a product manifold consisting of four prototypical clusters, and then construct a submanifold which is itself a prototypical cluster.  We can now take a big step:  We can use this construction to define the semantics of an \textit{atomic formula} in a logical language, that is, to define a \textit{predicate} with four \textit{arguments}.  The general idea is to replace the standard semantics of classical logic, based on sets and their elements, with a semantics based on manifolds and their points.  For a detailed discussion of how this works, see \citep{mccarty15ground}.  (These ideas emerge naturally from earlier work on the applications of AI to Law.)  The natural setting for these developments is a logical language based on category  theory,  or  what  is  known  as  a  \textit{categorical  logic.}  Thus, in \citep{mccarty15ground}, we analyze a categorical logic based on the category of differential manifolds ({\bf Man}), which is weaker than a logic based on the category of sets ({\bf Set}) or the category of topological spaces ({\bf Top}). Let's call this a \textit{manifold logic.}  In a manifold logic, as explained above, the prototypical clusters provide an interpretation of the atomic formulas, and the proof theory extends the differential manifold structure throughout the entire language.  A technical paper addressed to the computational logic community on these ideas is currently in preparation, with the working title: ``Manifold Logic and the Theory of Differential Similarity.''

 

\bibliography{HiDimSpaces.jmlr}
 
\end{document}